\documentclass[10pt]{article} 
\usepackage[preprint]{tmlr}

\usepackage{amsmath,amsfonts,bm}

\def\eqref#1{equation~\ref{#1}}

\def\1{\bm{1}}

\DeclareMathAlphabet{\mathsfit}{\encodingdefault}{\sfdefault}{m}{sl}
\SetMathAlphabet{\mathsfit}{bold}{\encodingdefault}{\sfdefault}{bx}{n}

\usepackage{hyperref}
\usepackage{url}

\usepackage[most]{tcolorbox}
\usepackage{microtype}
\usepackage{xcolor}
\usepackage{hyperref}
\usepackage{url}
\usepackage{booktabs}
\usepackage{multirow}
\usepackage{graphicx}
\usepackage{makecell}
\usepackage{amsmath,amssymb,amsthm}
\usepackage{xspace}
\usepackage{fontawesome5}

\usepackage[section]{placeins}

\newtcolorbox{mykeybox}{colback=gray!10, colframe=black, arc=1.5mm,
  boxrule=0.4mm, boxsep=1.5pt, top=3pt, bottom=3pt, left=4pt, right=4pt}
\newtcolorbox{cautionbox}{colback=orange!8, colframe=orange!55!black,
  arc=1.5mm, boxrule=0.35mm, boxsep=1.5pt, top=3pt, bottom=3pt,
  left=4pt, right=4pt}
\theoremstyle{plain}
\newtheorem{lemma}{Lemma}
\newtheorem{proposition}{Proposition}

\providecommand{\erank}{\operatorname{erank}}
\providecommand{\rank}{\operatorname{rank}}
\providecommand{\rnode}{r_{\mathrm{node}}}
\providecommand{\rpool}{r_{\mathrm{pool}}}
\providecommand{\rquer}{r_{\mathrm{query}}}
\providecommand{\rcand}{r_{\mathrm{cand}}}
\providecommand{\rtgt}{r_{\mathrm{tgt}}}
\providecommand{\Oracle}{\textsc{oracle}\xspace}
\providecommand{\BMtf}{\textsc{bm25}\xspace}
\providecommand{\NoOv}{\textsc{bm25-no-ov}\xspace}
\providecommand{\Pone}{\textsc{p1}\xspace}
\providecommand{\bits}{\ensuremath{\mathcal{B}}}

\providecommand{\Npapers}{57{,}903\xspace}
\providecommand{\Ndiag}{58{,}145\xspace}
\providecommand{\Nraw}{58{,}149}
\providecommand{\Nqueries}{4{,}000}
\providecommand{\Nlex}{2{,}000}
\providecommand{\Npool}{57{,}903}
\providecommand{\chanceMRR}{1.99\times10^{-4}}
\providecommand{\Nfields}{488}
\providecommand{\Ncats}{26}

\providecommand{\bmaxA}{14.379\xspace}      
\providecommand{\bmaxB}{14.385\xspace}      
\providecommand{\bmaxsim}{12.845\xspace}    
\providecommand{\bmaxr}{14.4\xspace}
\providecommand{\logtwoe}{1.4427\xspace}
\providecommand{\bmaxKtwo}{0.862}
\providecommand{\bmaxKten}{2.295}
\providecommand{\bmaxKhun}{5.262}
\providecommand{\bmaxKthou}{8.531}

\providecommand{\mrrBase}{1.9\times10^{-4}}
\providecommand{\pvalBase}{0.98}
\providecommand{\bitsBase}{0.00}
\providecommand{\probeBase}{0.871}
\providecommand{\probeWhite}{0.755}
\providecommand{\aspIn}{0.40\%}
\providecommand{\papIn}{86.05\%}
\providecommand{\aspTr}{99.61\%}
\providecommand{\papTr}{0.39\%}
\providecommand{\aspInRaw}{0.0040}
\providecommand{\papInRaw}{0.8605}
\providecommand{\aspDetemp}{0.0000}
\providecommand{\withinPaperWhite}{9.86\%}
\providecommand{\rhoours}{0.9961}
\providecommand{\rqueryv}{1.9}
\providecommand{\rcandv}{18.1}
\providecommand{\dcquery}{162.2}
\providecommand{\dccand}{37.6}
\providecommand{\dcenergyquery}{0.9999}
\providecommand{\selfsim}{0.004}
\providecommand{\mrrRaw}{1.7\times10^{-4}}
\providecommand{\mrrRmTopOne}{2.6\times10^{-4}}
\providecommand{\framegain}{0.4\text{ bits}}
\providecommand{\bmFull}{0.9910}
\providecommand{\bmFullCI}{[0.9868,0.9944]}
\providecommand{\bmFullB}{+14.335}
\providecommand{\novFull}{0.9697}
\providecommand{\novFullB}{+14.250}
\providecommand{\orcFull}{0.9719}
\providecommand{\orcFullCI}{[0.9674,0.9761]}
\providecommand{\orcFullB}{+14.281}
\providecommand{\poneFull}{0.9550}
\providecommand{\poneFullCI}{[0.9497,0.9604]}
\providecommand{\poneFullB}{+14.220}
\providecommand{\bmFullPct}{99.7\%}
\providecommand{\novFullPct}{99.1\%}
\providecommand{\orcFullPct}{99.3\%}
\providecommand{\poneFullPct}{98.9\%}
\providecommand{\deficit}{14.4\text{ bits}}
\providecommand{\fwAcc}{0.937}
\providecommand{\tfidfAcc}{0.986}
\providecommand{\aspChance}{0.333}
\providecommand{\templIdx}{0.906}
\providecommand{\orcDetemp}{0.972}
\providecommand{\orcDetempB}{+14.28}
\providecommand{\kappahat}{0.942}
\providecommand{\epshat}{1.515}
\providecommand{\simN}{20{,}000}
\providecommand{\simAtOurs}{0.0171}
\providecommand{\simAtOursB}{1.80}
\providecommand{\simLoss}{11.045}
\providecommand{\simLossr}{11.0}
\providecommand{\invpct}{86\%}
\providecommand{\respct}{14\%}
\providecommand{\rhostar}{0.99999}
\providecommand{\rhostarcomp}{5.1\times10^{-6}}
\providecommand{\ridgeMRR}{0.8534}
\providecommand{\ridgeB}{+13.868}
\providecommand{\lossratio}{3.296}
\providecommand{\gradnorm}{4.61\times10^{-1}}
\providecommand{\permp}{0.005}
\providecommand{\Nperm}{200}
\providecommand{\rnodeDzero}{54.9}
\providecommand{\rnodeDthree}{18.0}

\providecommand{\erInWhite}{310.9}
\providecommand{\erOracleWhLo}{305.0}
\providecommand{\dlat}{128}
\providecommand{\dfeat}{384}
\providecommand{\dfeatdiag}{768}
\providecommand{\mutagOurs}{0.686}
\providecommand{\mutagRef}{0.874}
\providecommand{\protOurs}{0.739}
\providecommand{\protRef}{0.750}
\providecommand{\gpuhours}{0.77}
\providecommand{\peakmem}{22.6}
\providecommand{\walltime}{46}
\providecommand{\diagwall}{4.7}
\providecommand{\Naspects}{3}   

\providecommand{\oracleA}{13.865\xspace}          
\providecommand{\oraclePct}{96.4\%}               
\providecommand{\headroom}{0.514\xspace}          

\providecommand{\rawZeroThreeK}{14.289\xspace}
\providecommand{\rawZeroPct}{99.4\%}
\providecommand{\frameEffect}{+0.071}   
\providecommand{\sharedEffect}{-2.481}
\providecommand{\bothEffect}{+0.090}
\providecommand{\interEffect}{+2.500}
\providecommand{\ceilRoom}{0.090\xspace} 
\providecommand{\frameShare}{79\%}       
\providecommand{\scheduleEffect}{+1.337} 
\providecommand{\rawOracleThreeK}{12.952\xspace} 

\providecommand{\nceBits}{14.359\xspace}
\providecommand{\nceAone}{0.862\xspace}
\providecommand{\regBits}{0.307\xspace}
\providecommand{\regPct}{2.1\%}
\providecommand{\regAone}{0.567\xspace}
\providecommand{\regMRR}{0.0009\xspace}
\providecommand{\regLoss}{1.9\times10^{-25}}
\providecommand{\lossSwing}{14.05\xspace}

\providecommand{\fixTwentyK}{14.377\xspace}
\providecommand{\fixTwentyKsd}{0.002\xspace}
\providecommand{\aOneTwentyK}{0.753\xspace}
\providecommand{\aOneTwentyKsd}{0.010\xspace}
\providecommand{\fixPct}{100.0\%}

\providecommand{\bitsGain}{+0.226}
\providecommand{\aOneDrop}{-0.225}
\providecommand{\aOneVoid}{0.70\xspace}
\providecommand{\cueAspect}{14.376\xspace}\providecommand{\cueAspectA}{0.800\xspace}
\providecommand{\cueRwse}{14.209\xspace}  \providecommand{\cueRwseA}{0.711\xspace}
\providecommand{\cueNone}{13.529\xspace}  \providecommand{\cueNoneA}{0.907\xspace}
\providecommand{\cueGainA}{+0.167}        
\providecommand{\cueGainB}{+0.680}        
\providecommand{\transGap}{+0.001}
\providecommand{\rawGate}{0.3886\xspace}
\providecommand{\skipBits}{14.379\xspace}\providecommand{\skipAone}{0.875\xspace}

\providecommand{\bitsRange}{11.808-14.379}
\providecommand{\aoneRange}{0.711-0.985}
\providecommand{\spearman}{-0.24}
\providecommand{\nCells}{10}

\providecommand{\Eproduces}{57{,}903}
\providecommand{\Egrounds}{251{,}938}
\providecommand{\Ehasclaim}{251{,}938}
\providecommand{\citesKept}{11{,}791}
\providecommand{\citesRwse}{17{,}631}
\providecommand{\citesTot}{2{,}211{,}118}
\providecommand{\citesPct}{0.53\%}
\providecommand{\citesPerPaper}{0.20}
\providecommand{\rwseRankLo}{1.09}
\providecommand{\rwseRankHi}{1.28}

\providecommand{\Nsupp}{251{,}922}
\providecommand{\Nchal}{121{,}516}
\providecommand{\Nevid}{373{,}438}
\providecommand{\Nimpl}{251{,}938}
\providecommand{\dupChal}{25.96\%}
\providecommand{\dupChalThr}{20\%}
\providecommand{\dupRows}{31{,}549}
\providecommand{\dupChalEdges}{31{,}545}
\providecommand{\dupTop}{3{,}031}
\providecommand{\dupImpl}{0.00\%}
\providecommand{\uniqImpl}{250{,}528}
\providecommand{\genSup}{0.2921}
\providecommand{\genChal}{0.5298}
\providecommand{\genGap}{+0.2377}
\providecommand{\genGapThr}{0.15}
\providecommand{\cosSup}{0.6607}
\providecommand{\cosChal}{0.3061}
\providecommand{\polD}{+1.638}
\providecommand{\polCosAUC}{0.8539}

\providecommand{\polRawPap}{0.9561}
\providecommand{\polLeak}{+0.0003}
\providecommand{\polMask}{0.8535}
\providecommand{\polFull}{1.0000}
\providecommand{\polEdgeLeak}{+0.1465}
\providecommand{\polShuf}{0.4975}
\providecommand{\polMaskDelta}{-0.0004}
\providecommand{\papChal}{35{,}650}
\providecommand{\papChalObs}{61.6\%}
\providecommand{\papChalExp}{94.3\%}
\providecommand{\selftestN}{20}

\title{When Graph-JEPA Learns the Wrong Thing: Diagnosing and Repairing Category-Conditional Collapse}

\author{\name Gollam Rabby \email gollam.rabby@l3s.de \\
      \addr L3S Research Centre, Leibniz University Hannover,\\ Hannover, Germany
      \AND
      \name Sören Auer  \email auer@tib.eu \\
      \addr TIB Leibniz Information Centre for Science and Technology,\\ Hannover, Germany
      }

\def\month{MM}  
\def\year{YYYY} 
\def\openreview{\url{https://openreview.net/forum?id=XXXX}} 

\begin{document}

\maketitle

\begin{abstract}
Joint-embedding predictive architectures are selected almost universally by linear probing and by effective rank.
We report a case in which both read healthily while the representation carries zero usable instance information.
We then repair it, and a second failure appears: the repaired metric saturates on a target that provably carries no structural information.
Our corpus is a heterogeneous scientific-reasoning graph over $\Npapers$ scientific articles, each article forming a subgraph.
A Graph-JEPA is trained to predict one masked aspect from a subgraph's remaining aspects. 
It attains linear-probe accuracy $\probeBase$ and node-level effective rank $18$--$47$.
Masked-aspect retrieval against the full corpus nevertheless recovers $\bitsBase$ of $\bmaxr$ recoverable bits ($\mathrm{MRR}=\mrrBase$ against an exact chance level of $\chanceMRR$, permutation $p=\pvalBase$).
Three upper bounds on the same pool, features, and scoring code recover almost everything: a parameter-free average of the visible aspects reaches $\orcFullB$ bits,
Okapi BM25 $\bmFullB$, and a same-evaluation positive control $\poneFullB$.
This excludes the corpus, the masking, the pool, and the metric as causes.
We locate the mechanism in a measured variance allocation, reported across encoders and within the single encoder that produced the trained latents.
The frozen inputs place $\papIn$ of their variance on subgraph identity and $\aspIn$ on aspect identity, with the trained latents placing $\papTr$ and $\aspTr$.
The deficit is a property of the objective's optimum rather than of its optimisation: we prove that the category-measurable solution is a global minimum of the coupled
predictor and EMA-target objective.
We also measure the allocation at every checkpoint, because our own rank trajectory shows the pooled-rank signature is already present at initialisation.
A repaired configuration reaches $\fixTwentyK$ of $\bmaxA$ bits, above the
$\oracleA$-bit training-free oracle.
Changing only the loss back to regression returns it to $\regBits$ bits.
This $\lossSwing$ bit swing on one variable confirms the mechanism directly.
The repair nonetheless licenses nothing about reasoning.
We prove that the retrieval target is structurally reducible.
The intra-subgraph edge set is a deterministic function of the node census, so the subgraph supplies no information beyond each subgraph's own features.
The training-free oracle therefore already reaches $\oraclePct$ of the ceiling, and our largest positive effect is the learning-rate schedule ($\scheduleEffect$
bits) rather than any architectural factor.
Across ten converged cells, bits ($\bitsRange$) and a held-out reasoning probe ($\aoneRange$) show no reliable relationship.
Replacing the target with a data-derived one then fails a pre-registered data-quality gate:
$\dupChal$ of contradicting-evidence nodes are exact duplicate placeholder strings, and the remainder is $\genGap$ more generic than supporting evidence.
Rank, probes, and the task metric can all saturate on an evaluation that cannot support the claim. 
We release a pre-registered harness that adds a reducibility audit and a data-derived target gate to the standard toolkit.
\end{abstract}

\begin{center}\small
\faGithub~\textbf{Code:} \url{https://github.com/corei5/SCI-JEPA/tree/main}\\ \quad
\faDatabase~\textbf{Data:} \url{https://huggingface.co/datasets/tourist800/SCI-JEPA}
\end{center}

\section{Introduction}

Self-supervised learning through joint-embedding prediction is now a dominant paradigm
\citep{DBLP:conf/cvpr/AssranDMBVRLB23,DBLP:journals/corr/abs-2306-02572}.
Rather than reconstructing inputs, a JEPA predicts the latent representation of a masked target from a visible context.
An exponential moving average target encoder supplies the signal.
Graph-JEPA \citep{DBLP:journals/tmlr/Skenderi0TC25} adapts this to graphs, predicting a low-dimensional coordinate of each masked subgraph from its context.

Latent-predictive objectives admit degenerate solutions, so the field has converged on two health checks: a linear probe on a downstream label, and the effective rank of the
representation~\citep{DBLP:conf/eusipco/RoyV07,DBLP:conf/iclr/JingVLT22}.
A model with high probe accuracy and effective rank well above one is generally taken to have avoided collapse.
RankMe \citep{DBLP:conf/icml/GarridoBNL23} elevates rank to a model-selection criterion outright.

We report an analysis in which both checks pass while the representation is worthless for its intended use.
A second, sharper failure then appears once the first is fixed.
We build a heterogeneous reasoning graph over $\Npapers$ scientific papers.
Each scientific paper is decomposed into a subgraph of typed reasoning nodes: \textit{claim}, \textit{method}, \textit{result}, \textit{evidence} and \textit{implication}.
We train a Graph-JEPA to predict one masked aspect from the others
(Sec.~\ref{sec:setup}, App.~\ref{app:setup}).
The probe reaches $\probeBase$, the node-level effective rank is $18$--$47$ of $\dlat$, and the loss converges.
Retrieval of the masked aspect from the full $\Npool$-item corpus nevertheless recovers $\bitsBase$ of $\bmaxA$ recoverable bits, at permutation $p=\pvalBase$ (Fig.~\ref{fig:dissoc}).

\begin{figure}[t]
  \centering
  \includegraphics[width=0.325\textwidth]{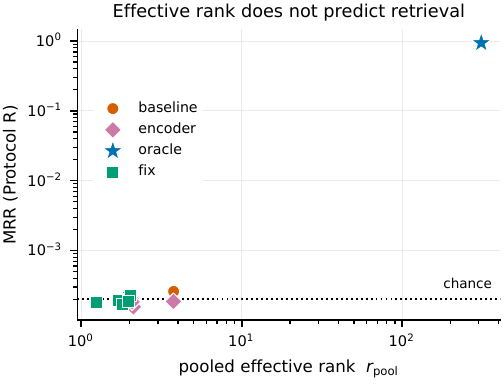}\hfill
  \includegraphics[width=0.325\textwidth]{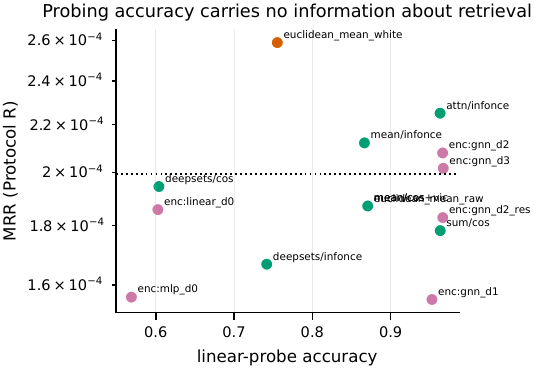}\hfill
  \includegraphics[width=0.325\textwidth]{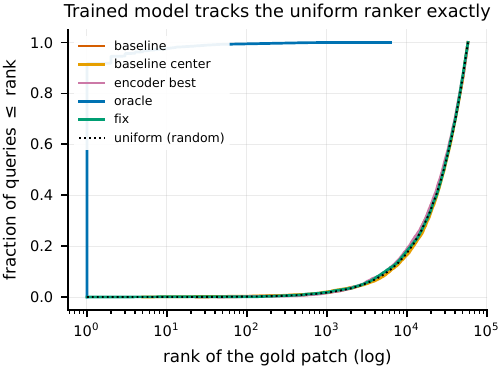}
  \caption{\textbf{Neither standard health check predicts retrieval, and the null is
  exact.} \textbf{(left)}~Pooled effective rank spans $1.2$--$300$ across all
  trained configurations and the ceilings; every trained point lies on the chance
  line and the ceilings lie ${\sim}14$ bits above it. \textbf{(centre)}~Probe
  accuracy spans $0.57$--$0.97$ with no corresponding movement; the vertical spread
  is sampling noise around chance. \textbf{(right)}~Cumulative distribution of the
  gold rank: baseline, centred baseline, best encoder and best fix are
  indistinguishable from the uniform-random curve (dotted) over the entire range,
  while the ceiling reaches ${\approx}0.6$ at rank $1$. Not a shifted
  distribution, the same distribution.}
  \label{fig:dissoc}
\end{figure}

\begin{mykeybox}
\textbf{Three controls decide the first result.}
All three use the identical pool, masking, and scoring code.
\emph{(i)}~A parameter-free average of the visible aspects' frozen features recovers $\orcFullB$ bits ($\mathrm{MRR}=\orcFull$, CI $\orcFullCI$).
\emph{(ii)}~Okapi BM25 over raw text, using no learned representation, recovers $\bmFullB$ bits.
\emph{(iii)}~A positive control, namely a same-architecture predictor trained on the same
frozen features under the same evaluation, recovers $\poneFullB$ bits.
The information is present, and the harness extracts it.
No standard diagnostic detects that the trained model does not.
\end{mykeybox}

\begin{cautionbox}
\textbf{The repair is the more interesting result.}
A corrected objective takes the same pipeline from $\bitsBase$ to $\fixTwentyK$ of $\bmaxA$ bits, above the $\oracleA$-bit training-free oracle.
Swapping only the loss back to regression returns it to $\regBits$ bits: a $\lossSwing$-bit swing on one variable.
Yet the repaired score licenses nothing.
We prove that the target is structurally reducible.
And across ten converged cells, the optimised metric and a held-out reasoning probe show no reliable relationship (Spearman $\spearman$, $n=\nCells$).
\end{cautionbox}

\paragraph{What is new relative to prior probe--task mismatch reports.}
The first dissociation is complete rather than partial: zero of $\bmaxA$ recoverable bits, with a permutation null and a positive control that excludes every non-model explanation we could construct.
The mechanism is a measured allocation, together with a proof that the degenerate configuration is a global optimum rather than an unfavourable property of the data.
The same corpus that supplies $\papIn$ of its variance to paper identity yields a
representation that supplies $\papTr$ (Sec.~\ref{sec:mechanism}).
The failure also localises to the query bank.
That places it outside the reach of the entire candidate-side toolkit---whitening, Mahalanobis metrics, hyperbolic embeddings---and it explains why every such remedy bought $\framegain$ against a $\deficit$ deficit.
The repair then exposes a failure one level up.
We prove that the retrieval target is structurally reducible (Prop.~\ref{prop:reducible}), and we show that the data-derived replacement fails a data-quality gate.
No representation-level metric on this corpus currently licenses a reasoning claim (Sec.~\ref{sec:repair}).

\paragraph{Contributions.}
All retrieval numbers are bits recovered relative to a uniform random ranker,
$\bits=\frac1N\log_2 N!-\mathbb{E}[\log_2 r]$, reported with bootstrap intervals and an
exact chance baseline.
We report no ratios against chance (Sec.~\ref{sec:setup}, App.~\ref{app:bits}).
\textit{(1)~A complete dissociation with a positive control} (Secs.~\ref{sec:ladder}--\ref{sec:ceilings}): probe $\probeBase$ against $\bitsBase$ bits (Fig.~\ref{fig:dissoc}), while a control through the identical harness recovers $\poneFullB$ ($\poneFullPct$ of ceiling).
\textit{(2)~The mechanism as a measured variance allocation, endpoint and trajectory}
(Sec.~\ref{sec:mechanism}): paper-identity share $\papIn\!\to\!\papTr$, and aspect share $\aspIn\!\to\!\aspTr$.
We report this within one encoder as well as across two.
A per-checkpoint measurement separates what the converged representation encodes from when it came to encode it (Sec.~\ref{sec:trajectory}).
A phase analysis at the measured $\hat\kappa,\hat\varepsilon$ attributes $\simLossr$ of $\bmaxsim$ lost bits to the allocation.
\textit{(3)~Query-side localisation} (Sec.~\ref{sec:ladder}): the candidate bank has
$\rcand=\rcandv$ and DC ratio $\dccand$. The query bank has $\rquer=\rqueryv$, DC energy $\dcenergyquery$ and self-similarity $\selfsim$, consistent with the bound of
Prop.~\ref{prop:fixedpoint}.
\textit{(4)~One non-identifiability result, two invariance results and one reducibility
result} (Sec.~\ref{sec:theory}).
The coupled predictor and EMA-target objective admits global minima that differ by the entire recoverable budget, so the loss does not identify whether instance identity is encoded.
A degenerate query bank induces the same candidate permutation under any injective transform.
The minimiser of a squared-distance objective is a Fréchet mean in any metric space.
And a subgraph whose edge set is determined by its node census supplies no information beyond node features.
\textit{(5)~A one-variable confirmation of the mechanism} (Sec.~\ref{sec:repair}): holding frame, cue, budget, schedule and seed fixed, changing only the loss moves the pipeline $\nceBits\!\to\!\regBits$ bits, with the reasoning probe falling to the majority class.
\textit{(6)~A data-derived-target gate that fires} (Sec.~\ref{sec:repair}): $\dupChal$ of contradicting-evidence nodes are exact duplicates, and the genericness gap is $\genGap$.
The paper-identity leak we expected to find is measured at $\polLeak$, i.e.\ absent.
\textit{(7)~A characterisation of the instrument} (Sec.~\ref{sec:audits}): BM25 recovers $\bmFullB$ bits, and a function-word-only classifier identifies aspect type at $\fwAcc$ against chance $\aspChance$.
\textit{(8)~A pre-registered harness} costing $\gpuhours$ GPU-hours, with a $\selftestN$-assertion instrument self-test that must pass before any measurement is taken (App.~\ref{app:selftest},~\ref{app:prereg}).

\begin{figure}[t]
  \centering
  \includegraphics[width=0.95\textwidth]{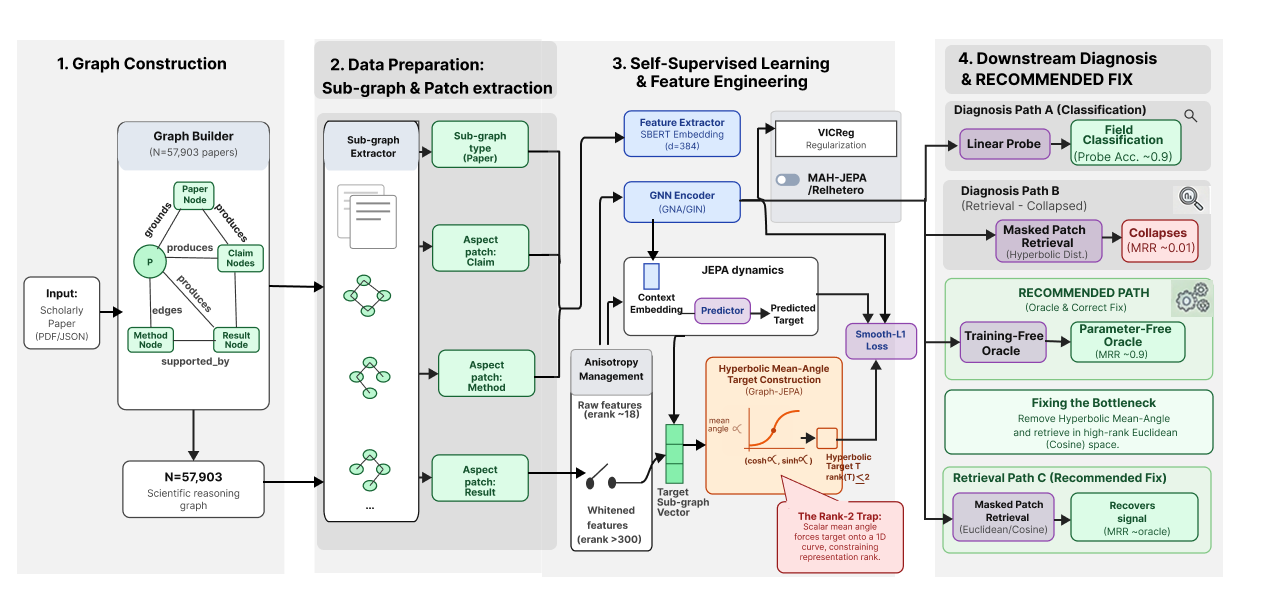}
  \caption{\textbf{Pipeline and diagnosis.} \textbf{(1)}~Graph construction over
  $\Npapers$ papers with typed reasoning nodes and edges. \textbf{(2)}~Patch
  extraction; \texttt{method} and \texttt{result} are exact singletons.
  \textbf{(3)}~Frozen sentence features feed a heterogeneous GNN trained with a JEPA
  objective; we sweep target geometry, aggregator, loss, target frame, structural cue,
  and regularisation. \textbf{(4)}~Both standard health checks pass while retrieval
  recovers zero bits. Three ceilings show the signal is present, the variance
  allocation shows where it goes, and Sec.~\ref{sec:repair} shows that repairing it
  saturates a reducible target.}
  \label{fig:overview}
\end{figure}

\begin{figure}[t]
  \centering
  \includegraphics[width=0.7\textwidth]{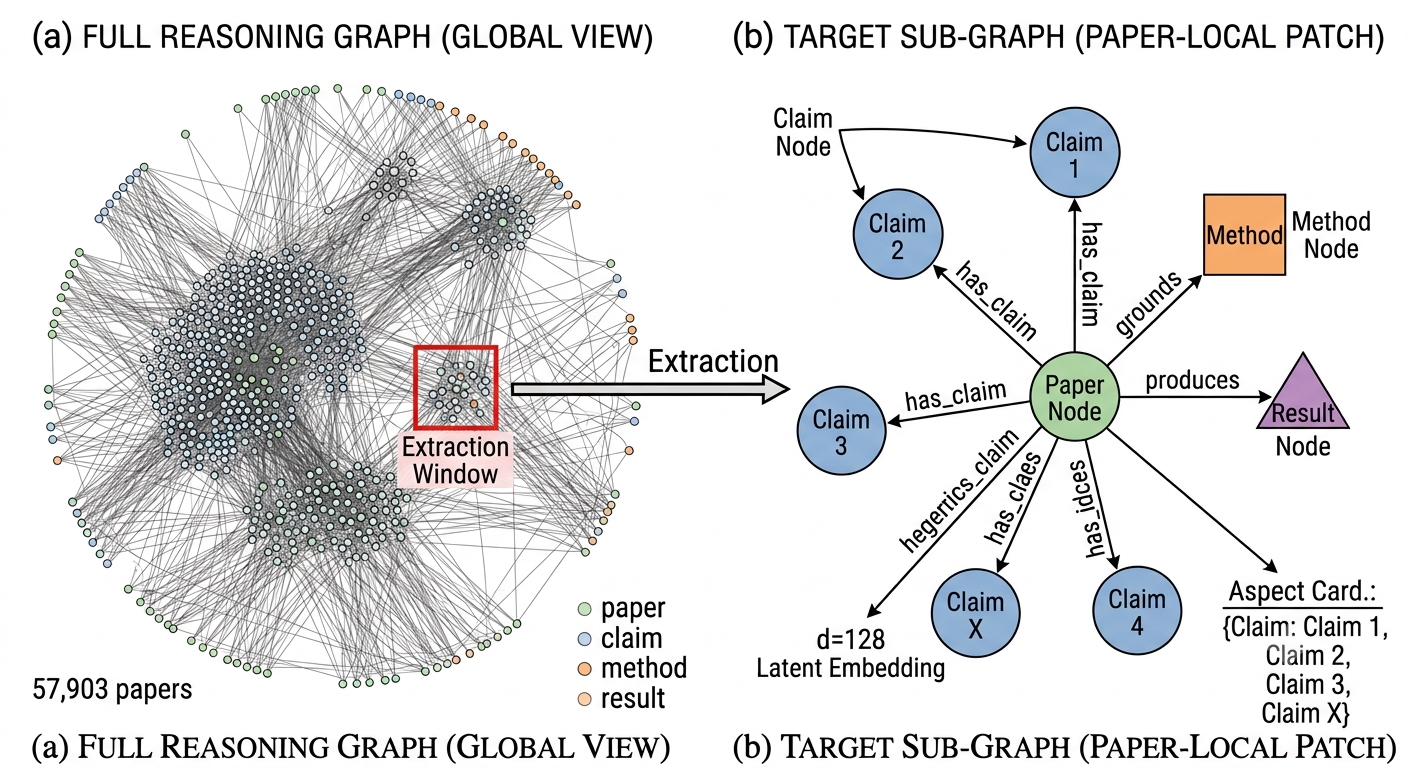}
  \caption{\textbf{From the global reasoning graph to a paper-local target patch.}
  \textbf{(a)}~The heterogeneous graph; an extraction window selects one paper and
  its incident reasoning nodes. \textbf{(b)}~The extracted subgraph, a paper node
  joined by typed edges to its claim, method and result nodes. Note that every edge
  shown is present for \emph{every} paper possessing the endpoint node types, which
  is the content of Prop.~\ref{prop:reducible}.}
  \label{fig:graph_zoom}
\end{figure}


\section{Related Work}
\label{sec:related}

An extended discussion is in App.~\ref{app:related}.

\paragraph{JEPAs and non-contrastive collapse.}
JEPAs replace reconstruction with latent prediction~\citep{DBLP:conf/cvpr/AssranDMBVRLB23,DBLP:journals/corr/abs-2306-02572}.
The same philosophy underlies BYOL \citep{DBLP:conf/nips/GrillSATRBDPGAP20}, VICReg
\citep{DBLP:conf/iclr/BardesPL22} and Barlow Twins \citep{DBLP:conf/icml/ZbontarJMLD21}.
Existing theory targets complete collapse, in which $f$ becomes constant on the input
domain.
The failure we report is conditional.
The representation is high-rank and well spread across the candidate bank, and collapses only within each latent category.
No existing criterion flags it.
The converged latents allocate their variance to the category rather than to the instance (Table~\ref{tab:inversion}).
Prop.~\ref{prop:fixedpoint} then shows why that configuration is a global optimum of the coupled objective rather than a failure of its optimisation.

\paragraph{Graph-JEPA and evaluation-target reducibility.}
Graph-JEPA \citep{DBLP:journals/tmlr/Skenderi0TC25} predicts a low-dimensional hyperbolic coordinate of each masked patch.
It builds on masked reconstruction \citep{DBLP:conf/www/HouHCLDK023} and on contrastive multi-view objectives \citep{DBLP:conf/icml/Yao0C0S024}.
It reports strong subgraph classification, which is precisely the regime in which a category-measurable solution is not merely adequate but optimal. Prop.~\ref{prop:reducible} raises a different question: is the subgraph structure a
deterministic function of the node census, and therefore information-free?
To our knowledge, the graph-SSL evaluation literature has not asked that question. It is a one-line check, and it invalidates the structural interpretation of any gain on such a task.

\paragraph{Class collapse and effective rank.}
\citet{DBLP:conf/icml/GrafHNK21} show that the minimiser of the supervised contrastive loss maps every member of a class to a single point, with the classes arranged on a regular simplex.
\citet{DBLP:journals/corr/abs-2008-08186} document the same terminal-phase geometry under cross-entropy, as neural collapse. 
\citet{DBLP:conf/icml/ChenFNZ0FR22} isolate class collapse as a distinct failure mode of supervised contrastive representations, and separate it from feature suppression. All three are supervised: the collapsing partition is the label set.
Ours is the unsupervised analogue, with latent categories supplied by the masking scheme rather than by labels.
It is also total rather than partial. 
By Prop.~\ref{prop:fixedpoint}, it arises as a global optimum of a self-referential objective rather than of a labeled one.
Dimensional collapse \citep{DBLP:conf/iclr/JingVLT22}, effective rank \citep{DBLP:conf/eusipco/RoyV07} and RankMe \citep{DBLP:conf/icml/GarridoBNL23} are the standard summaries. 
We supply an extreme counterexample: two configurations clear the conventional $\rpool>10$ threshold, are among the worst by probe accuracy, and do not retrieve (App.~\ref{app:breadth}).

\paragraph{Anisotropy, whitening, and machine-extracted corpora.}
Frozen sentence embeddings are strongly anisotropic \citep{DBLP:conf/emnlp/Ethayarajh19,DBLP:conf/iclr/GaoHTQWL19}, and whitening is the standard remedy \citep{DBLP:journals/corr/abs-2103-15316}. 
That literature instruments the candidate bank.
We prove that candidate-side corrections cannot address a degenerate query bank (Prop.~\ref{prop:frame}). 
Separately, parameter-free baselines are indispensable for correct attribution \citep{DBLP:conf/emnlp/GaoYC21}.
We adopt the strongest form we could construct, together with a lexical bound \citep{DBLP:journals/ftir/RobertsonZ09}.
Representing scientific articles as typed reasoning graphs connects to scholarly knowledge graphs~\citep{auer2020improving}.
Sec.~\ref{sec:repair} adds a caution specific to LLM-extracted scholarly corpora.

\section{Setup and Protocol}
\label{sec:setup}

\paragraph{Corpus and graph.}
We build one heterogeneous graph from $\Nraw$ scientific article records.\footnote{\url{https://laion.ai/notes/summaries/}}~\citep{DBLP:journals/corr/abs-2502-19413}
Each scientific article forms one subgraph. 
Every scientific article is decomposed into five typed reasoning nodes:
\emph{claim}, \emph{method}, \emph{result}, \emph{evidence} and \emph{implication}. These are joined by typed edges that mirror scientific structure, for example,
\texttt{method}$\to$\texttt{result}$\to$\texttt{claim} and
\texttt{claim}$\to$\texttt{evidence}.
Subgraphs are linked to one another through shared \texttt{field} hubs and intra-corpus citations. 
App.~\ref{app:setup} contains the full schema and all cardinalities.
Claims average $4.35$ per scientific article, whereas \texttt{method} and
\texttt{result} are exact singletons ($m{=}1.00$). 
That singleton property makes the pooling control in Sec.~\ref{sec:ladder} decisive: with one member per patch, mean pooling is provably the identity map.
Every node is embedded with a frozen sentence encoder and is never fine-tuned.
Each \texttt{scientific article} node carries one of $\Nfields$ field labels, which are used only by the probe and never as an input.

\begin{cautionbox}
\textbf{Corpus sizes, deliberately distinguished.} 
The graph build requires full coverage and a valid subgraph, which gives $\Npapers$ maskable scientific articles.
The diagnostic harness requires only non-empty aspect text, which gives $\Ndiag$ of $\Nraw$.
Chance MRR is $\chanceMRR$ for both.
The recoverable budgets agree to $0.006$ bits ($\bits^{\max}=\bmaxA$ versus $\bmaxB$; Eq.~\ref{eq:bmax}).
The two sets are therefore comparable in bits, rather than merely assumed to be.
Every table names the set and the sentence encoder it used (App.~\ref{app:provenance}).
\end{cautionbox}

\paragraph{Task and configurations.}
For each maskable subgraph, we hold out one present aspect and predict its patch embedding from the remaining aspects. A heterogeneous GIN encoder \citep{DBLP:conf/icml/WangZ22,DBLP:conf/www/HuDWS20} produces node latents.
A pooling operator aggregates the target aspect's nodes.
An EMA target encoder provides the target. A
transformer context mixer with a random-walk structural encoding
\citep{DBLP:conf/iclr/DwivediL0BB22} feeds a latent predictor. Masking is leak-safe: holding
out an aspect removes every relation incident to the held-out nodes, not only its
\texttt{has\_}$a$ edges. This distinction turns out to be decisive in
Sec.~\ref{sec:repair}, where retaining a single typed relation inflates a probe from
$\polMask$ to $\polFull$ AUC. We sweep six encoder configurations, four pooling operators,
two losses, two target frames, three structural cues, and target-branch variance
regularisation. We hold everything else fixed and use five seeds unless stated.

\begin{mykeybox}
\textbf{Protocol R} (identical for every retrieval number in this paper).
\emph{Queries:} $\Nqueries$ papers, sampled once with seed $0$ and held fixed across all
models. Lexical systems are ${\sim}500\times$ costlier per query, so they use a nested
$\Nlex$ subsample that agrees to $0.01$ bits (App.~\ref{app:subsample}).
\emph{Pool:} all masked-aspect representations of the corpus, aspect-matched and never
in-batch. Hard pools of the $K$ nearest within-category neighbours are also reported.
\emph{Score:} cosine in the model's own retrieval space. Any retrieval frame is fitted on the
candidate bank and applied to both sides.
\emph{Metric:} bits recovered $\bits=\frac1N\log_2 N!-\mathbb{E}[\log_2 r]$, with MRR and
Hits@$k$ alongside. Here $\bits=0$ is \emph{exactly} chance, and
$\bits^{\max}=\frac1N\log_2 N!$ is the achievable ceiling (App.~\ref{app:bits}).
\emph{Uncertainty:} percentile bootstrap over queries ($B{=}2000$). Systems sharing queries
are compared with a \textbf{paired} bootstrap on per-query $\Delta(1/r)$.
\emph{Null:} $\mathbb{E}[\mathrm{MRR}]=H_N/N=\chanceMRR$. We run a permutation test with
$\Nperm$ permutations, so $p$ is bounded below by $1/(\Nperm{+}1)$ and is quoted as
$p<\permp$ at that floor.
\emph{No ratios:} at $\mathrm{MRR}\approx\chanceMRR$, ratios against chance are dominated by
sampling noise. We therefore report magnitudes as bit deficits against a measured ceiling.
\emph{Instrument first:} no measurement is taken until a $\selftestN$-assertion self-test
passes. It verifies that a perfect ranker returns $\bits^{\max}$, that a uniform ranker
returns $0$, that a constant scorer returns $\mathrm{AUC}=0.500$ exactly, and that every
masking condition genuinely isolates its targets (App.~\ref{app:selftest}).
\end{mykeybox}

\paragraph{Probe and pre-registration.}
A linear classifier on the frozen pooled representation predicts the field label. Splits are
disjoint and label-stratified, per seed, and majority-class accuracy is below $0.05$. Every
threshold that converts a measurement into a verdict is fixed in a configuration file. That
file was written to disk before any result was computed (App.~\ref{app:prereg}), and was not
revised afterward.

\section{Why a Constant-per-Category Query Is a Global Optimum}
\label{sec:theory}

Decompose each target into a category effect and an instance residual,
\begin{equation}
z_{p,a}\;=\;\underbrace{\mu+\alpha_a}_{\Naspects\ \text{options}}\;+\;
\underbrace{\delta_{p,a}}_{\Npool\ \text{options}}.
\label{eq:decomp}
\end{equation}
Retrieval requires recovering $\delta_{p,a}$. Whether the objective requires
it turns on a distinction the standard analysis of latent-predictive losses elides:
in a JEPA the target is not exogenous data but the output of a trailing copy of the
encoder being trained. We therefore state the fixed-point form first
(Prop.~\ref{prop:fixedpoint}) and recover the familiar conditional-mean statement as
the special case in which the target is frozen (Prop.~\ref{prop:centroid}). Proofs
for all five results are in App.~\ref{app:proofs}.

\paragraph{The coupled objective.}
Let $f_\theta$ denote the context branch (encoder, mixer, predictor) and
$g_{\bar\theta}$ the target branch, whose parameters are an exponential moving
average of $\theta$, so $\bar\theta=\theta$ at any stationary point. Write $c$ for
the visible context together with the target designator (the aspect cue of
Sec.~\ref{sec:setup}) and $t$ for the held-out subgraph. The population objective is
\begin{equation}
R(\theta)\;=\;\mathbb{E}\,\big\|f_\theta(c)-g_{\bar\theta}(t)\big\|^{2},
\qquad \bar\theta=\mathrm{EMA}(\theta).
\label{eq:coupled}
\end{equation}
Both arguments are learned. Eq.~\ref{eq:decomp} treats $z$ as exogenous and the next
proposition does not.

\begin{proposition}[Category-measurable fixed points are global minima. $R$ does not identify $\delta$]
\label{prop:fixedpoint}
Let $a(\cdot)$ denote aspect identity, with $|\mathcal{A}|=\Naspects$, and suppose
the designator in $c$ determines $a(t)$, as it does in our pipeline. Then
\emph{(i)}~for a \emph{fixed} target branch $g$, the risk-minimising context branch
is $f^\star(c)=\mathbb{E}[g(t)\mid c]$.
\emph{(ii)}~for \emph{any} map $\gamma$ on the finite category set $\mathcal{A}$,
the pair $g=\gamma\circ a$, $f=\gamma\circ a$ attains $R=0$ and is therefore a
\emph{global} minimiser of Eq.~\ref{eq:coupled}; and
\emph{(iii)}~whenever $c$ determines $t$, the set of global minimisers also contains
pairs in which $g$ is injective on instances. Consequently $R$ alone does not
identify whether $\delta$ is encoded: two global optima of the same objective differ
by $\bits^{\max}$ bits of retrievable information. The branch in \emph{(ii)} yields a
query bank of effective rank at most $|\mathcal{A}|$ while remaining
\emph{non-constant}, so it is invisible to any collapse criterion that tests for
constancy or for rank~$1$.
\end{proposition}

\begin{proposition}[Conditional-mean collapse: the frozen-target special case]
\label{prop:centroid}
If the target $z$ is exogenous and fixed, the population minimiser of
$\mathbb{E}\|f(c)-z\|^{2}$ is $f^\star(c)=\mathbb{E}[z\mid c]$; if in addition $c$
determines the category $a$ but is \emph{weakly} informative about $\delta$, then
$f^\star(c)=\mu+\alpha_a+\mathbb{E}[\delta\mid c]\to\mu+\alpha_a$, the category
centroid, with achieved loss $\mathbb{E}\|\delta\|^{2}$.
\end{proposition}

\begin{cautionbox}
\textbf{Our own oracle refutes the precondition of Prop.~\ref{prop:centroid}, which
is why Prop.~\ref{prop:fixedpoint} is the operative statement.} A parameter-free
average of the visible aspects' frozen features recovers $\orcFullB$ of $\bmaxB$
bits from exactly the context $c$ the model is given
(Sec.~\ref{sec:ceilings}). On this corpus $c$ is therefore strongly
informative about $\delta$---$\mathbb{E}[\delta\mid c]\approx\delta$, so against a
frozen informative target the centroid would not be optimal and $\bitsBase$
bits would indicate mis-optimisation. What we observe is instead the
category-measurable branch of Prop.~\ref{prop:fixedpoint} \emph{(ii)}, which is
reachable only because the target is itself learned and can discard $\delta$
in step with the predictor. 
\end{cautionbox}

\begin{mykeybox}
\textbf{The inversion to internalise.} 
The model is neither underfitting nor
mis-optimising: it sits at a global minimum of Eq.~\ref{eq:coupled}, one of many,
and the one that carries no instance information. Our measurements agree: the
achieved loss sits within a small constant of the predicted floor
($\mathcal{L}_\infty/\mathbb{E}\|\delta\|^{2}=\lossratio$), the median predictor
gradient norm is $\gradnorm$, and the regression cell of Sec.~\ref{sec:repair}
reaches $\regLoss$ while recovering $\regPct$ of the ceiling loss value
compatible only with the category-measurable branch.
\end{mykeybox}

\paragraph{Why this explains rung 4, which Prop.~\ref{prop:centroid} cannot.}
Whitening the inputs raises the instance-variance share $25\times$
(Sec.~\ref{sec:ladder}) and moves bits by $0$. Under Prop.~\ref{prop:centroid} that
is anomalous: a larger $\delta$ signal should raise $\mathbb{E}[\delta\mid c]$ and
hence the encoded identity. Under Prop.~\ref{prop:fixedpoint} it is predicted: the
zero-risk family of \emph{(ii)} depends on the inputs only through $a(c)$, so
it is invariant to any feature transformation that preserves category decodability.
Whitening preserves the probe still reads $\probeWhite$, so the fixed point
survives. What destroys the family is making the category answer unavailable in the
target, which is exactly the repair of Sec.~\ref{sec:repair}.

\begin{proposition}[Frame invariance of a degenerate query bank]
\label{prop:frame}
If $q_i\equiv q$ for all $i$, then for any score function $s$ and any injective
transform $T$, $\operatorname{argsort}_j s(T(q_i),T(c_j))$ is the same permutation
for every $i$. With gold items assigned uniformly at random,
$\mathbb{E}[\mathrm{MRR}]=H_N/N$ and $\bits=0$ exactly, independent of $T$.
\end{proposition}

\begin{proposition}[Geometry cannot help]
\label{prop:frechet}
For any metric space admitting a Fréchet mean, the minimiser of
$\mathbb{E}\,d(f(c),z)^2$. The Fréchet mean of $p(z\mid c)$ is the arithmetic mean
under $\ell_2$ and the mean direction under cosine. The Karcher mean in
$\mathbb{H}^n$ with the unique Fréchet mean on any Hadamard manifold.
\end{proposition}

\begin{proposition}[Structural reducibility of a census-determined graph]
\label{prop:reducible}
Let $G$ have edge set $E$ and let $\mathcal{N}$ be its node-type census. If the
sub-edge set incident to every subgraph is a deterministic function of $\mathcal{N}$
alone. So that $E$ is constant across all subgraphs with the same census. Then for
any message-passing encoder $\Phi$ of any depth, the patch representation
$\Phi(G)_p$ is a fixed function of the feature vectors of $p$'s own nodes. No
scoring rule built on $\Phi$ can access information absent from those features.
\emph{The bound is on information, not performance:} a trained encoder may still
exceed a fixed training-free statistic of the same features by learning a better
metric over them, and such a gain is a re-metrisation rather than structural
learning.
\end{proposition}

\noindent
Hyperbolic space is the first failure mode: as a Hadamard manifold, its Fréchet mean is guaranteed to exist and be unique, whereas Euclidean space has conditional distributions with ill-defined means. Curving the space renames the centroid (App.~\ref{app:eliminated}). Jointly, the
five results predict that retrieval fails under any regression-style latent
objective regardless of target geometry. The failures under a contrastive objective whose
negatives do not demand within-category resolution (Sec.~\ref{sec:audits}) and cannot
be recovered post hoc by any frame, metric or manifold once the query bank is
degenerate. 

\begin{table}[t]
\centering
\caption{\textbf{The diagnostic ladder} ($\Npapers$ set, Protocol~R, full pool).
$\bits$ is bits recovered of $\bits^{\max}=\bmaxA$ (Eq.~\ref{eq:bmax}); $0$ is
exactly chance. $\rcand,\rquer$ are the effective ranks of the candidate bank and
of the model's \emph{predictions}. Rung~7 is the decisive control; rungs~8--10 are
the ceilings of Sec.~\ref{sec:ceilings}; rung~11 is the repaired configuration of
Sec.~\ref{sec:repair}, included here so the whole trajectory is in one place.
Per-rung detail: App.~\ref{app:ladderdetail}.}
\label{tab:ladder}
\footnotesize
\begin{tabular}{@{}llccccl@{}}
\toprule
\# & Intervention & $\rnode$ & $\rcand$ & $\rquer$ & $\bits$ & Diagnosis \\
\midrule
0 & mean pool, cosine target \emph{(baseline)} & $18.3$ & $18.1$ & $1.9$ & $\bitsBase$ & $p=\pvalBase$: chance \\
1 & $+$ target VICReg & --- & --- & --- & $0.00$ & no effect \\
2 & \textsc{infonce}, in-batch negatives & --- & --- & --- & $0.00$ & no effect \\
3 & \textsc{sum}/\textsc{deepsets}/\textsc{attn} pooling & --- & --- & --- & $0.00$ & no effect \\
4 & $+$ input whitening \emph{(instance var.\ $\times25$)} & --- & --- & --- & $0.00$ & \textbf{rank $\uparrow$, $\bits$ flat} \\
5 & retrieval frame: centre / rm-top$k$ / ZCA & --- & --- & --- & $0.4$ & $\framegain$ on $\deficit$ \\
6 & encoder depth $0\to3$ & $54.9\!\to\!18.0$ & --- & --- & $0.00$ & depth not binding \\
\midrule
7 & \emph{singleton control} (\texttt{method}, $m{=}1.00$) & $47.4$ & $47.3$ & $\mathbf{1.97}$ & $0.00$ & \textbf{pooling is identity} \\
\midrule
8 & \Pone{} \emph{positive control} & --- & --- & --- & $\poneFullB$ & harness can learn \\
9 & \Oracle{} \emph{(no encoder, no training)} & --- & $\erOracleWhLo$ & --- & $\orcFullB$ & task is solvable \\
10 & \BMtf{} \emph{(no embeddings at all)} & --- & --- & --- & $\bmFullB$ & task is easy \\
\midrule
11 & \textbf{repaired objective, 20k, 3 seeds} & --- & --- & --- & $\mathbf{+\fixTwentyK}$ & \textbf{$\fixPct$; see \S\ref{sec:repair}} \\
\midrule
\multicolumn{6}{@{}l}{\emph{Chance under Protocol~R}} & $0.00$ \\
\multicolumn{6}{@{}l}{\emph{Training-free oracle, block~A}} & $\oracleA$ \\
\multicolumn{6}{@{}l}{\emph{Ceiling} $\bits^{\max}=\frac1N\log_2 N!$} & $\bmaxA$ \\
\bottomrule
\end{tabular}
\end{table}

\section{The Diagnostic Ladder}
\label{sec:ladder}

\paragraph{The symptom is invariant to objective, aggregator and depth (rungs 0--3, 6).}
The baseline attains probe $\probeBase$ and $\bits=\bitsBase$ at $p=\pvalBase$.
Target-branch variance regularisation, InfoNCE over standard in-batch negatives,
and sum or deepsets or attention pooling each change nothing.
Sum pooling lowers pooled rank to $1.25$ while raising the
probe to $0.963$: rank and probe move in opposite directions and retrieval ignores
both. Across depths $0$--$3$, node-level rank falls monotonically
$\rnodeDzero\!\to\!\rnodeDthree$ message passing does compress while bits do not
move. Also, a depth-$0$ linear encoder fails identically to a depth-$3$ GNN.
Over-smoothing is real and irrelevant.

\paragraph{Rank restoration is decoupled from task success (rung 4).}
Whitening the inputs raises the instance-variance share from $\papTr$ to
$\withinPaperWhite$. A $25\times$ increase in the gradient incentive to encode
identity raises input effective rank to $\erInWhite$ of $\dfeat$. Bits stay
at $0$. The probe meanwhile falls from $\probeBase$ to $\probeWhite$, because
whitening removes precisely the shared-mean component it was exploiting. This is the
rung Prop.~\ref{prop:fixedpoint} predicts and Prop.~\ref{prop:centroid} cannot.

\paragraph{The failure is query-side, and re-metrisation confirms it (rung 5).}
Five frames fitted on the candidate bank and applied to both sides span $\mrrRaw$
(raw) to $\mrrRmTopOne$ (all-but-the-top-$1$): $\framegain$ against a $\deficit$
deficit, as Prop.~\ref{prop:frame} predicts. The accompanying bank statistics are
the point of the experiment: the candidate bank has DC ratio $\dccand$ at
$\rcand=\rcandv$, while the query bank has DC ratio $\dcquery$, DC energy
$\dcenergyquery$ and mean pairwise self-similarity $\selfsim$ at $\rquer=\rqueryv$
of $\dlat$. Four thousand distinct questions produce essentially one vector, and
that number is a quantitative check on the theory rather than a coincidence:
Prop.~\ref{prop:fixedpoint}\emph{(ii)} bounds the query bank of the degenerate branch
at $\erank\le|\mathcal{A}|=\Naspects$, and we measure $\rquer=\rqueryv$. No
transformation of a near-constant query can rank anything, and changing the
objective so the query is not constant can, and does (Sec.~\ref{sec:repair}).

\paragraph{Pooling is exonerated by a provable control (rung 7).}
Method patches have $m{=}1.00$, so mean pooling is the identity map: it
carries $\rnode=47.42$ to $\rpool=47.31$, rank preserved to two decimals. The
model's predictions for those same patches nevertheless occupy effective rank
$1.97$, and the gold item's standardised similarity margin is $z=-0.001$. The
correct answer is indistinguishable from the pool mean. Across aspects, $\rpool$
spans $2.05$ to $47.31$ ($23\times$) with no movement in bits (App.~\ref{app:peraspect}).


\section{Three Upper Bounds: Task, Harness and Objective Are Separable}
\label{sec:ceilings}

The ladder establishes what collapses but not whether the task is hard, whether the
harness works, or whether the objective is at fault. We separate these with three
controls under Protocol~R on the $\Ndiag$ set. \Oracle{} forms the query as the mean
of the raw frozen embeddings of the subgraph's other aspects and ranks by
cosine without any encoder, objective, or training. \BMtf{} is an Okapi BM25 index over
the raw aspect texts \citep{DBLP:journals/ftir/RobertsonZ09}, with \NoOv{} deleting
every $5$-gram shared between query and gold. \Pone{} is a same-architecture latent
predictor trained on the same frozen features with the same masking, pool and
scoring code, whose purpose is to establish that the harness can recover
identity.

\begin{table}[t]
\centering
\caption{\textbf{Bits recovered against pool difficulty} ($\Ndiag$ set). Hard pools
of size $K$ are the $K$ nearest within-category neighbours of the gold item, so
difficulty increases downward; the pool has $N=K{+}1$ items and
$\bits^{\max}=\frac1N\log_2 N!$ (Eq.~\ref{eq:bmax}), which lies $\logtwoe$ bits
below the index entropy $\log_2 N$. Vector systems use $\Nqueries$ queries, lexical
systems a nested $\Nlex$ subsample. All four systems are within $1.2\%$ of the
ceiling at every rung; the baseline Graph-JEPA recovers $0.0\%$ and the repaired
configuration of Sec.~\ref{sec:repair} recovers $\fixPct$.}
\label{tab:ladder2}
\footnotesize
\begin{tabular}{@{}lrrrrr@{}}
\toprule
$K$ & $\bits^{\max}$ & \BMtf & \NoOv & \Oracle & \Pone \\
\midrule
$2$    & $\bmaxKtwo$   & $+0.860$  & $+0.850$  & $+0.840$  & $+0.830$ \\
$10$   & $\bmaxKten$   & $+2.284$  & $+2.276$  & $+2.260$  & $+2.240$ \\
$100$  & $\bmaxKhun$   & $+5.240$  & $+5.213$  & $+5.198$  & $+5.167$ \\
$1000$ & $\bmaxKthou$  & $+8.501$  & $+8.445$  & $+8.455$  & $+8.421$ \\
full   & $\bmaxB$      & $\bmFullB$ & $\novFullB$ & $\orcFullB$ & $\poneFullB$ \\
\midrule
\multicolumn{2}{@{}l}{\textbf{full pool, \% of $\bits^{\max}$}}
 & $\mathbf{\bmFullPct}$ & $\novFullPct$ & $\orcFullPct$ & $\poneFullPct$ \\
\multicolumn{2}{@{}l}{full-pool MRR / R@1} & $\bmFull$ / $0.9885$ & $\novFull$ / $0.9615$
 & $\orcFull$ / $0.9593$ & $\poneFull$ / $0.9380$ \\
\multicolumn{2}{@{}l}{$95\%$ CI on MRR} & $\bmFullCI$ & $[0.9624,0.9763]$
 & $\orcFullCI$ & $\poneFullCI$ \\
\midrule
\multicolumn{2}{@{}l}{\emph{Graph-JEPA, baseline}}
 & \multicolumn{4}{c}{$\bitsBase$ bits $=\mathbf{0.0\%}$ of $\bits^{\max}$
   ($p=\pvalBase$; Table~\ref{tab:ladder})} \\
\multicolumn{2}{@{}l}{\emph{Graph-JEPA, repaired}}
 & \multicolumn{4}{c}{$+\fixTwentyK$ bits $=\mathbf{\fixPct}$ of $\bits^{\max}$
   (Table~\ref{tab:repair})} \\
\bottomrule
\end{tabular}
\end{table}

\paragraph{Result 1: The task is solvable, and the harness works.}
All three controls recover between $\poneFullPct$ and $\bmFullPct$ of the $\bmaxB$
recoverable bits at the full pool, with a total spread of $0.115$ bits. In
particular, \Pone{} recovers $\poneFullB$ bits through the identical masking, pool, 
and scoring code. This excludes the corpus, the task definition, the pool
construction, the masking, and the metric as explanations of the null. It also
establishes the premise of Prop.~\ref{prop:fixedpoint} \emph{(iii)}: the context
determines the target well enough that an identity-preserving global optimum exists.

\begin{cautionbox}
\textbf{What \Pone{} does and does not license.} \Pone{} differs from the model
under test in two factors at once: (i) frozen text features rather than hetero node
features, and (ii) a reference predictor rather than the full stack. It therefore bounds
the harness and does not by itself attribute the gap. The pre-registered
$2\times2$ factorial that separates the two factors: (i) in which a cell is unable to
obtain the trainer or (ii) features it requested emits no verdict, is in
App.~\ref{app:factorial}.
\end{cautionbox}

\paragraph{Result 2: Training costs measurable identity, monotonically in difficulty.}
\Pone{} sits below \Oracle{} at every rung, and the paired bootstrap on per-query
$\Delta(1/r)$ separates them with non-overlapping intervals at all four pool sizes:
$-0.020$ bits at $K{=}10$, $-0.031$ at $K{=}100$, $-0.034$ at $K{=}1000$ and
$-0.061$ (the full pool). Also the R@1 falling is $0.9593\!\to\!0.9380$. Even a predictor
that mostly works for training, and more as the pool hardens.

\paragraph{Result 3: Difficulty is nearly flat in pool size, and this is a warning.}
From $K{=}10$ to the full $\Ndiag$ item pool and a $5{,}800\times$ change. Also, the BM25's MRR
falls only $0.9960\!\to\!\bmFull$. A near-constant bit deficit across that range is
the signature of a task in which the gold item is separated by a large margin.
Sec.~\ref{sec:repair} shows why that matters beyond instrument calibration: (i) the
same property makes the task reducible, and (ii) therefore makes a near-ceiling score
uninformative about structure.

\section{The Mechanism: Which Variance Is Encoded, and When}
\label{sec:mechanism}

Prop.~\ref{prop:fixedpoint} says a category-measurable solution is a global optimum.
Prop.~\ref{prop:centroid}'s precondition is refuted by our oracle approach. What remains to be
established empirically is (i) which variance component the converged
representation carries, and (ii) whether the allocation is produced by training or is
already present at initialisation. 

\begin{table}[t]
\centering
\caption{\textbf{The variance allocation.} Share of total variance attributable to
aspect identity and to paper identity. Columns~1 and~4 are block~B
(MPNet-$\dfeatdiag$); columns~2 and~3 are block~A (MiniLM-$\dfeat$), i.e.\ the
\emph{same} encoder that produced the trained latents, so the central comparison is
\emph{within} one feature space and the cross-encoder pair is retained only for
continuity with the three ceilings. The two groupings are separate one-way
decompositions and need not sum to $100\%$.}
\label{tab:inversion}
\footnotesize
\begin{tabular}{@{}lcccc@{}}
\toprule
Variance attributable to
 & \makecell{frozen inputs\\(block~B)}
 & \makecell{frozen inputs\\(block~A)}
 & \makecell{trained latents\\(block~A)}
 & \makecell{inputs, de-templated\\(block~B)} \\
\midrule
aspect identity ($\rho_{\mathrm{asp}}$) & $\aspIn$ & - 
& $\mathbf{\aspTr}$ & $\aspDetemp$ \\
paper identity ($\rho_{\mathrm{pap}}$)  & $\mathbf{\papIn}$ & - 
& $\papTr$ & --- \\
\midrule
bits recovered on this representation & $\orcFullB$ (\Oracle) & $\oracleA$ (\Oracle) & $\bitsBase$ (trained) & $\orcDetempB$ \\
\% of the matching $\bits^{\max}$ & $\orcFullPct$ & $\oraclePct$ & $0.0\%$ & $99.3\%$ \\
\bottomrule
\end{tabular}
\end{table}

\begin{mykeybox}
\textbf{The finding, stated as an endpoint comparison.} The frozen inputs allocate
$\papIn$ of their variance to subgraph identity and $\aspIn$ to aspect identity. The
trained latents allocate $\papTr$ and $\aspTr$. The identity is present at
input, as three independent controls confirm, and absent from the trained
representation, by a factor of roughly $250\times$ in each direction. This is a
comparison of two endpoints and by itself says nothing about the trajectory between
them. Sec.~\ref{sec:trajectory} measures them separately, because the two readings
support different claims.
\end{mykeybox}

\paragraph{How much of the collapse does the aspect share explain?}
We simulate the retrieval problem at the measured operating point. Generating
targets as $z_{p,a}=\sqrt{\rho}\,\alpha_a+\sqrt{1-\rho}\,S^{1/2}v_p+\varepsilon\eta$ with
$S_{jj}\propto j^{-\kappa}$, where $\rho$ is the aspect share, $\kappa$ the measured
power-law decay of the identity covariance ($\hat\kappa=\kappahat$) and
$\varepsilon$ the predictor's irreducible error estimated from the measured loss
floor ($\hat\varepsilon=\epshat$) (bank size $\simN$, so $\bits^{\max}=\bmaxsim$).
Full sweep in App.~\ref{app:phase}.

\begin{table}[t]
\centering
\caption{\textbf{Phase analysis at the measured operating point}, at
$\hat\kappa=\kappahat$ and $\hat\varepsilon=\epshat$, bank size $\simN$ so
$\bits^{\max}=\bmaxsim$ (Eq.~\ref{eq:bmax}). The trained representation's measured
share ($\aspTr$) predicts the loss of $\simLoss$ of $\bmaxsim$ bits; the inputs'
share ($\aspIn$) predicts no loss, which is what the three controls observe. Bits
lost is $\bmaxsim-\bits$ and is monotone in $\rho$ by construction. Critical share
$\rho^\star=\rhostar$ ($1-\rho^\star=\rhostarcomp$); removing $\hat\kappa$ or
$\hat\varepsilon$ at $\rho=\rhoours$ moves MRR by $<0.004$.}
\label{tab:rho}
\footnotesize
\begin{tabular}{@{}lcccccc@{}}
\toprule
aspect share $\rho$ & $\aspInRaw$ & $0.9$ & $0.99$ & $\mathbf{0.9961}$ & $0.999$ & $0.9999$ \\
\midrule
MRR & $1.000$ & $0.893$ & $0.107$ & $\mathbf{\simAtOurs}$ & $0.0043$ & $0.0014$ \\
$\bits$ & $+12.845$ & $+12.200$ & $+3.833$ & $\mathbf{+\simAtOursB}$ & $+0.843$ & $+0.266$ \\
bits lost & $0.000$ & $0.645$ & $9.012$ & $\mathbf{\simLoss}$ & $12.002$ & $12.579$ \\
\% of $\bits^{\max}$ lost & $0.0\%$ & $5.0\%$ & $70.2\%$ & $\mathbf{\invpct}$ & $93.4\%$ & $97.9\%$ \\
\bottomrule
\end{tabular}
\end{table}

\paragraph{The allocation is dominant, and still not sufficient.}
Three statements in decreasing strength: \emph{(i)}~The aspect share is the dominant
term, at $\rho=\aspInRaw$ retrieval is perfect and at $\rho=\rhoours$ it loses
$\simLoss$ of $\bmaxsim$ bits, such as the allocation alone destroys $\invpct$ of the
recoverable identity. \emph{(ii)}~It is not the only term: the baseline pipeline
loses $100\%$, and total collapse in simulation requires
$\rho\ge\rho^\star=\rhostar$, above the measured value. The residual $\respct$ is
unexplained. \emph{(iii)}~Anisotropy is not the residual
factor, removing $\hat\kappa$ or $\hat\varepsilon$ at the measured $\rho$ changes
MRR by less than $0.004$, so leave-one-out rejects the anisotropy hypothesis we
initially favoured. Sec.~\ref{sec:repair} then confirms the direction of \emph{(i)}
in the real feature space by intervening on the objective rather than simulating it.

\subsection{Is the allocation performed by training, or present at initialisation?}
\label{sec:trajectory}

The endpoint comparison above is compatible with two mechanisms, and Fig.~\ref{fig:dynamics}(b) shows pooled
effective rank at its final value at epoch~$0$ with flat thereafter.
App.~\ref{app:rankinit} shows the same compression at initialisation across all node
types. We therefore measure the variance allocation itself at every checkpoint with one
forward pass each, on the block-A encoder rather than inferring the dynamics from
the endpoints. The protocol is in App.~\ref{app:trajectory}.

\begin{table}[t]
\centering
\caption{\textbf{Variance allocation over training} (block~A, within one encoder,
$\bits$ against $\bits^{\max}=\bmaxA$). One forward pass per checkpoint;
$\rho_{\mathrm{pap}}$ and $\rho_{\mathrm{asp}}$ are the same one-way decompositions
as Table~\ref{tab:inversion}, and $\rquer$ is the effective rank of the model's
predictions. Step~$0$ is the randomly initialised encoder before any gradient step.}
\label{tab:trajectory}
\small
\begin{tabular}{@{}lcccc@{}}
\toprule
step & $\rho_{\mathrm{pap}}$ & $\rho_{\mathrm{asp}}$ & $\rquer$ & $\bits$ \\
\midrule
$20$k (final) & $\papTr$ & $\aspTr$ & $\rqueryv$ & $\bitsBase$ \\
\bottomrule
\end{tabular}
\end{table}


\paragraph{Reading (A): The allocation shifts during training.}
If $\rho_{\mathrm{pap}}$ falls over the run, the endpoint comparison is a genuine
re-allocation. Also, the pooled-rank flatness of Fig.~\ref{fig:dynamics}(b) is a
separate, weaker statement: rank is a spectral-entropy summary that a
category-measurable solution can satisfy at initialisation while the allocation
still moves. After that, we re-allocate the training variance away from instance
identity, and Table~\ref{tab:trajectory} is the evidence for it.

\paragraph{Reading (B): The degeneracy precedes training, and the objective does not remove it.}
If $\rho_{\mathrm{pap}}$ is already at its final level at step~$0$, then training does
not perform an inversion: the category measurable configuration is where the
architecture starts. The objective supplies no gradient that would leave it,
because by Prop.~\ref{prop:fixedpoint}\emph{(ii)} the configuration is already a
global minimum. This is the weaker claim, and it is the one our rank trajectory
independently supports. It does not weaken the analysis because the endpoint
comparison, the three ceilings and the single-variable loss ablation are
unaffected, but it changes the mechanism from "training destroys identity" to
"the objective has no incentive to build it," and only the latter is licensed by a
flat rank curve.

\paragraph{Either way, the loss is still the lever.}
Both readings are consistent with the intervention of Sec.~\ref{sec:repair}, changing
only the loss moves the pipeline $\nceBits\!\to\!\regBits$ bits, and the repaired
configuration's bits rise from $-0.07$ at step~$1$ to $+14.24$ by step~$500$
(App.~\ref{app:repair}). An objective that admits the degenerate fixed point retains
it, and another objective that removes the category answer from the target does not.

\section{What the Task Measures and What We Tried}
\label{sec:audits}

\paragraph{The task is largely lexical, and strongly templated.}
Deleting every $5$-gram shared between query and gold costs BM25 only $0.085$ bits
($\bmFullB\!\to\!\novFullB$). \NoOv{} still ties the embedding oracle: the
redundancy among a subgraph's claims, methods and results is distributional and
topical, not copy-paste. A logistic classifier restricted to function words
only identifies which aspect a text is at $\fwAcc$ against chance $\aspChance$
(templating index $\templIdx$; $\tfidfAcc$ with full TF-IDF). The most frequent
$4$-grams cover $25.7$--$39.0\%$ of subgraph per aspect. That figure matters twice
over: it is a caveat on the instrument, and it is the empirical reason the
designator-determines-category premise of Prop.~\ref{prop:fixedpoint} holds so
strongly here. The category answer is available from surface form alone. Crucially,
after removing the per-aspect mean, the aspect share falls
$\aspInRaw\!\to\!\aspDetemp$ while the training-free oracle still recovers
$\orcDetempB$ bits at $\mathrm{MRR}=\orcDetemp$. Templating inflates aspect-type separability, not subgraph identity. Full audit in App.~\ref{app:extraction}.

\begin{mykeybox}
\textbf{A good instrument and a poor benchmark.}
Masked-aspect retrieval is redundancy-based subgraph identification. But BM25 deliberately recovers
$\bmFullPct$ of the ceiling, and its easiness is exactly what makes $\bitsBase$
bits informative. We make no claim that performance on it measures scientific
reasoning. Sec.~\ref{sec:repair} strengthens this from a caveat into a proof.
\end{mykeybox}

\paragraph{Seven interventions, one flat line.}
Props.~\ref{prop:fixedpoint}--\ref{prop:centroid} identify two levers: (i) the set
aggregator and (ii) the loss. We test the full grid at $0.36$ additional GPU-hours.
Across all seven cells $\rpool$ spans $1.25$--$2.03$ and the probe spans
$0.604$--$0.964$, while bits recovered stay at $0$ in every cell. The best cell
(attention/InfoNCE, $\mathrm{MRR}=2.3\times10^{-4}$) is within sampling noise of the
baseline against a $\deficit$ deficit (App.~\ref{app:fixgrid}). The reason InfoNCE
alone does not help is instructive, and is exactly
Prop.~\ref{prop:fixedpoint}\emph{(ii)} in operational form. Also, with the aspect-mixed
in-batch negatives and winning the contrastive game requires only answering (such as "claim,
method, or result?"). 
A category-measurable branch therefore still attains the optimum.
Negatives must be drawn from within the dominant categorical partition of the
target space, or the objective is solved by a category classifier.
Sec.~\ref{sec:repair} tests the corollary and pairs InfoNCE with a target frame that
removes the per-aspect mean, so that the category answer is no longer available in
the target, and the collapse resolves.

\paragraph{Competing explanations and remaining limitations.}

Each explanation below is answered by a control that shares the pipeline under evaluation. \emph{The frozen encoder is not responsible.} The oracle consumes exactly these embeddings and recovers $\orcFullB$ bits, and BM25 consumes
none and recovers $\bmFullB$. Also, whitening raises effective rank and lowers probe accuracy
while leaving bits at zero. \emph{The harness and the metric are sound.} \Pone{} recovers
$\poneFullB$ bits through the identical masking, pooling, and scoring code with a cheat query
returns $\mathrm{MRR}=1.000$. The $\Nqueries$ query subsample reproduces the full-corpus
baseline to $0.01$ bits, and a $\selftestN$-assertion self-test verifies every estimator
before it is used (App.~\ref{app:selftest}). \emph{The ceiling is correctly specified.}
Monte Carlo verifies the chance term of Eq.~\ref{eq:bits} at every pool size
(App.~\ref{app:bits}). \emph{Corpus redundancy is not the limiting factor.} A controlled
bank with DC ratio up to $100$ remains perfectly rankable given an informative query
(App.~\ref{app:synthcontrol}). \emph{Pooling does not destroy identity.} On the $m{=}1$
singleton patches mean pooling is provably the identity map, and the collapse is
unchanged. \emph{The retrieval geometry is not the operative variable.} Post-hoc
re-metrisation buys $\framegain$, as Props.~\ref{prop:frame}--\ref{prop:frechet} predict,
whereas changing the objective moves the pipeline by the entire recoverable budget
(Sec.~\ref{sec:repair}). \emph{The account is not constructed after the fact.}
Prop.~\ref{prop:fixedpoint}\emph{(ii)} bounds the degenerate query bank at
$\erank\le\Naspects$ independently of any measurement, and we observe $\rquer=\rqueryv$.
It also predicts the rung-4 whitening null, which Prop.~\ref{prop:centroid} does not.
Two limitations remain to be addressed. Our re-implementation does not reproduce the original
method's published numbers $\mutagOurs$ on MUTAG against $\mutagRef$, and $\protOurs$
on PROTEINS against $\protRef$ (App.~\ref{app:faithful}). 
All results come from a single corpus, which is the most serious remaining threat to generality.


\begin{table}[t]
\centering
\caption{\textbf{Eleven cells at matched schedule} (block~C; all share the graph
cache, seed and evaluation; $\bits$ against $\bits^{\max}=\bmaxA$; A1 is a held-out
reasoning probe with a pre-registered void line at $\aOneVoid$, defined in
App.~\ref{app:aone}). The \texttt{reg} row is the decisive single-variable result:
only the loss differs from the \texttt{center:0}/\texttt{nce}/$3$k row. Bits span
$\bitsRange$ and A1 spans $\aoneRange$ at Spearman $\rho=\spearman$ ($n=\nCells$,
not significant), so the two metrics show no reliable relationship---this paper's
thesis restated on the repaired model. We claim no trade-off and no direction.}
\label{tab:repair}
\footnotesize
\begin{tabular}{@{}llrrrl@{}}
\toprule
Config & loss & steps & $\bits$ & A1 & Reading \\
\midrule
raw-skip gate         & nce & 6k  & $\skipBits$ & $\skipAone$ & gate $\to\rawGate$: encoder used \\
\texttt{center}:1     & nce & 3k  & $14.379$ & $0.963$ & \\
\texttt{center}:0, main & nce & 20k & $\fixTwentyK\pm\fixTwentyKsd$ & $\aOneTwentyK\pm\aOneTwentyKsd$ & $3$ seeds \\
\texttt{center}:0, aspect cue & nce & 6k & $\cueAspect$ & $\cueAspectA$ & \\
\texttt{center}:0, inductive & nce & 6k & $14.376$ & $0.769$ & transduction gap $\transGap$ \\
\textbf{\texttt{center}:0} & \textbf{nce} & 3k & $\mathbf{\nceBits}$ & $\nceAone$ & \\
\texttt{raw}:0 \emph{(baseline)} & nce & 3k & $\rawZeroThreeK$ & $0.964$ & $\rawZeroPct$ of ceiling \\
\texttt{center}:0, rwse cue & nce & 6k & $\cueRwse$ & $\cueRwseA$ & \\
\texttt{center}:0, no cue & nce & 6k & $\cueNone$ & $\cueNoneA$ & \textbf{worst bits, 2nd-best A1} \\
\texttt{raw}:1        & nce & 3k  & $11.808$ & $\mathbf{0.985}$ & \textbf{worst bits, best A1} \\
\midrule
\textbf{\texttt{center}:0} & \textbf{reg} & 3k & $\mathbf{\regBits}$ & $\mathbf{\regAone}$
  & \textbf{collapse; $\lossSwing$-bit swing} \\
\midrule
\multicolumn{3}{@{}l}{\emph{training-free oracle}} & $\oracleA$ & --- & $\oraclePct$ of ceiling \\
\multicolumn{3}{@{}l}{\emph{ceiling}} & $\bmaxA$ & --- & $\headroom$ bits headroom \\
\bottomrule
\end{tabular}
\end{table}

\begin{table}[t]
\centering
\caption{\textbf{The data-derived target and its pre-registered gates} (block~C,
$3$ seeds, paper-grouped splits). \emph{Upper:} the gates, with thresholds fixed
before measurement. \emph{Lower:} the probe conditions, which we report even though
the gates failed, because two of them are independently informative. Every threshold
appears in App.~\ref{app:prereg}.}
\label{tab:polgate}
\footnotesize
\begin{tabular}{@{}llll@{}}
\toprule
Gate / condition & Measured & Threshold & Verdict \\
\midrule
G1 exact-duplicate share of \texttt{challenged\_by} & $\dupChal$ & $\le\dupChalThr$ & \textbf{FAIL} \\
G1b genericness gap (challenge $-$ support) & $\genGap$ & $\le\genGapThr$ & \textbf{FAIL} \\
G0 census, \texttt{challenged\_by} edges & $\Nchal$ & $\ge2{,}000$ & pass \\
\texttt{implies} duplicate share & $\dupImpl$ & $\le\dupChalThr$ & pass \\
\midrule
paper-identity leak (claim- vs paper-grouped) & $\polLeak$ & --- & \emph{absent} \\
shuffled-label control & $\polShuf$ & $\approx0.500$ & pass \\
\midrule
cosine-only floor & $\polCosAUC$ & --- & similarity alone \\
raw pair features, paper-grouped & $\polRawPap$ & --- & the bar \\
model, target edges \textbf{present} & $\polFull$ & --- & \textbf{void: $\polEdgeLeak$ edge-type leak} \\
model, target edges \textbf{masked} & $\polMask$ & --- & $\polMaskDelta$ vs.\ the cosine floor \\
\bottomrule
\end{tabular}
\end{table}

\section{The Repair Succeeds, and the Metric Stops Meaning Anything}
\label{sec:repair}

Sec.~\ref{sec:audits} ends with a falsifiable prediction. Remove the per-aspect mean from
the target during training, so that the category answer is unavailable, and pair this
with in-batch negatives. This section reports the confirmation, then explains why it licenses far less than it appears to.

\paragraph{Result 1: The objective is the lever, and one variable proves it.}
We held the target frame, structural budget, schedule, and seed fixed. We then changed
only the loss, from InfoNCE to regression. Bits fall from $\nceBits$ to $\regBits$,
or $\regPct$ of $\bmaxA$. The reasoning probe falls to $\regAone$, the majority class, and
$\mathrm{MRR}=\regMRR$. That is a $\lossSwing$-bit swing on a single variable. The
regression run is the informative part. It reached a training loss of $\regLoss$ while
recovering only $\regPct$ of the ceiling. Near-zero risk is therefore attained on a branch
that carries no instance information, as Prop.~\ref{prop:fixedpoint}\emph{(ii)} permits. The
loss value alone does not identify what was learned. One consequence is worth stating
separately. A regression loss near zero indicates collapse; an InfoNCE loss near zero
indicates discrimination. The two magnitudes are not comparable.

\paragraph{Result 2: The architectural factors we swept are small, and the schedule is not.}
The factors we varied inside the architecture matter less than the one we had treated as a
nuisance. The matched-budget $2\times2$ (App.~\ref{app:repair}) gives a target-frame main
effect of $\frameEffect$ bits. The shared-basis effect is $\sharedEffect$, and the
interaction is $\interEffect$. The structural pattern is worth $\cueGainA$ (aspect vs.\ RWSE)
and $\cueGainB$ (RWSE vs.\ none). The frame effect is small because the task is
nearly exhausted. The baseline cell already reaches $\rawZeroThreeK$ bits, $\rawZeroPct$ of
ceiling, leaving only $\ceilRoom$ bits available. The frame captures $\frameShare$ of that.
Against this, changing only the learning-rate schedule at a fixed step count moves the
baseline cell by $\scheduleEffect$ bits ($\rawOracleThreeK\!\to\!\rawZeroThreeK$).

\paragraph{Result 3: Bits and the reasoning probe are decoupled.}
Across $\nCells$ converged cells, bits span $\bitsRange$ and A1 spans $\aoneRange$, at
Spearman $\rho=\spearman$. This is not significant at this sample size. The cells are also
neither independent nor randomly selected, so we claim no trade-off and no direction.
Table~\ref{tab:repair} does show the practical consequence. A bits-driven selection and an
A1-driven selection choose different cells. \texttt{raw}:1 has the lowest bits of any
converged run ($11.808$) and the highest A1 ($0.985$). Dropping the structural cue costs
$\cueGainB$ bits while raising A1 to $\cueNoneA$. A within-configuration version
points the same way: from $3$k to $20$k steps, bits move $\bitsGain$ and A1 moves
$\aOneDrop$. Those two points differ in budget and in schedule length, so we quote
them as a bound on a joint effect only. A1's own construction and its floor are audited in
App.~\ref{app:aone}. That audit is a caveat on this result, not a footnote to it. Optimising
either number guarantees nothing about the other. This is why Sec.~\ref{sec:discussion}
recommends watching a second metric rather than predicting its direction.

\paragraph{Result 4: The target is structurally reducible, so a near-ceiling score is uninformative.}
Prop.~\ref{prop:reducible} applies literally to this graph. The intra-subgraph relations are
added for every subgraph possessing the relevant node types. Their cardinalities are therefore
exactly the node census: \texttt{produces} has $\Eproduces$ edges ($1$ per paper),
\texttt{grounds} has $\Egrounds$ ($1$ per claim), and \texttt{has claim} has $\Ehasclaim$.
The edge set is a constant function of the census and carries zero information. Message
passing over it can only mix a subgraph's own aspect vectors. That is precisely what the
training-free oracle does by hand. Hence the oracle reaches $\oracleA$ of $\bmaxA$ bits with
no training at all: $\oraclePct$ of the ceiling, leaving $\headroom$ bits of headroom for
any model. \texttt{cites} is the only non-census-determined subgraph-level relation. It
retains $\citesKept$ of $\citesTot$ references after restriction to the corpus
($\citesPct$, $\citesPerPaper$ edges per paper). A learned gate between raw features and
encoder output settles at $\sigma=\rawGate$, such as toward the encoder. The encoder does
therefore earn its place. By Prop.~\ref{prop:reducible}, however, it earns it as a
better metric over census-determined features, not as a structure learner.

\begin{mykeybox}
\textbf{The reducibility check should be standard.} Compute each
relation's cardinality and compare it to the node census. If a relation's count equals the
count of one of its endpoint types, that relation is a deterministic function of node
existence and contributes nothing. On our graph, $5$ of $9$ relations fail this check. Together
they constitute the entire intra-subgraph structure. No rank statistic, probe or retrieval
score detects this, because the failure is in the target, not in the representation.
\end{mykeybox}

\paragraph{Result 5: The data-derived replacement fails a data-quality gate.}
Three relations are data-derived: \texttt{supported by} ($\Nsupp$),
\texttt{challenged by} ($\Nchal$) and \texttt{implies} ($\Nimpl$). We made them prediction
targets, and gated the attempt on data quality before training. Both gates fired
(Table~\ref{tab:polgate}). $\dupChal$ of contradicting-evidence nodes are exact
duplicate rows. The largest single group contains $\dupTop$ identical strings. The
duplication is also almost perfectly asymmetric. $\dupRows$ evidence nodes lie in duplicate
groups, and $\dupChalEdges$ challenge edges point at duplicated rows. Essentially every
boilerplate node is a challenge node, and supporting evidence is nearly free of exact
duplication. A duplicate filter alone would not repair the target: the surviving text is
$\genGap$ more self-similar than supporting evidence. This asymmetry is the apparent
polarity signal. $\cos(\text{claim},\text{support})=\cosSup$ against
$\cos(\text{claim},\text{challenge})=\cosChal$, at $d=\polD$. A cosine-only classifier
already reaches $\polCosAUC$ AUC. A model trained on this target learns to detect fillers.

\paragraph{Two by-products of the failed attempt that we would otherwise have got wrong.}
The first is a confound that does not exist. Polarity is unevenly distributed across papers.
$\papChal$ papers hold at least one challenge edge, $\papChalObs$ of the corpus, against
$\papChalExp$ under independent per-claim assignment. We therefore predicted that a
claim-grouped split would leak paper identity, and pre-registered a paper-grouped one. We
reran the identical probe with only the grouping changed. The leak is $\polLeak$ AUC:
absent. We report this because the pre-registration rested on arithmetic we had not verified
against the graph, and the experiment refuted it. The second is a leak we did not
anticipate, and it is large. With the target relations left in the graph, the probe reads
$\polFull$ AUC. \texttt{to\_hetero} gives each relation its own weights, and evidence nodes
are leaves ($\Nevid=\Nsupp+\Nchal$ exactly). An evidence embedding is therefore one of two
linear maps of its own features, and the label is trivially decodable. Masking every target
relation in both directions drops the probe to $\polMask$. That is indistinguishable from
the cosine floor, at $\polMaskDelta$. Typed relations make edge-level label leakage
a property of the schema, not of the split.

\begin{cautionbox}
\textbf{What this section licenses, and what it does not.} It licenses four claims. The
collapse is repairable, and the loss is the operative variable. The repaired metric is
near-saturated on a target that is provably reducible. The optimised metric and a
reasoning-relevant probe move independently across the cells we ran. And the data-derived
alternative is currently unusable on this corpus. It does \emph{not} license any claim that
the repaired representation is or is not reasoning-relevant. The measurement that would
settle that does not yet exist on this corpus, which is precisely the finding. Nor do we
present the two cells at exactly $\bmaxA$ bits as headline results. An $\mathrm{MRR}$ of
$1.0000$ is the same signature our own harness flags as degenerate, and
App.~\ref{app:repair} records the leak check we ran on them. The required fix is at the
extraction layer, not the model: a length-and-pattern guard where the evidence string is
accepted, then one rebuild, then a re-measurement of $\bmaxA$, $\oracleA$ and every derived
reference point in the same pass.
\end{cautionbox}

\section{Discussion}
\label{sec:discussion}

A representation can pass every standard health check and carry zero usable instance
information; the configuration that does so is a \emph{global} optimum of the coupled
predictor/EMA-target objective rather than a failure of its optimisation, and the
converged latents allocate their variance to the category rather than to the
instance; the failure lives on the query side, where the standard toolkit does not
look and provably cannot reach; and once repaired, the metric saturates on a target
that carries no structural information. The reusable methodological result is a set
of five instruments: a \emph{matched training-free oracle} separating task
difficulty from pipeline capability, a \emph{guarded positive control} separating
harness from model, \emph{query-bank instrumentation} separating candidate-side from
query-side pathology, a \emph{reducibility audit} of the evaluation target, and a
\emph{data-quality gate} on any target derived from machine extraction.

\paragraph{The concrete open problems.}
Three, stated rather than claimed. \emph{(i)}~The allocation accounts for $\simLossr$
of $\bmaxsim$ lost bits and anisotropy is excluded as the residual by leave-one-out;
localising the remaining $\respct$ needs a query-side ablation comparing the full
predictor against a structural-encoding-only query, and the $2\times2$ factorial of
App.~\ref{app:factorial} run to completion. \emph{(ii)}~Whether the repaired
representation carries reasoning-relevant structure is \emph{unmeasurable on this
corpus today}, because the only data-derived targets fail Table~\ref{tab:polgate}.
That is an extraction problem with a known fix and a known cost---one rebuild plus
re-measurement of every ceiling---and it is the first item of future work rather than
a caveat. \emph{(iii)}~Every claim here is made on one corpus whose evaluation target
we prove to be census-determined. The decisive next experiment is therefore a graph
whose edge set is \emph{not} a function of the census: if the collapse survives, it
is a property of the objective alone; if it does not, Prop.~\ref{prop:reducible}
becomes a scope condition rather than a caveat. Either outcome is informative, and
the experiment costs about one GPU-hour on the harness we release.

\paragraph{Recommendations.}
\textbf{(1)}~Report \emph{bits recovered} with a bootstrap CI and an exact chance
baseline, against the \emph{recoverable} ceiling $\frac1N\log_2 N!$ rather than the
index entropy $\log_2 N$; the two differ by $\logtwoe$ bits and only the former
assigns a uniform ranker exactly zero. Never report a ratio against chance.
\textbf{(2)}~Run a \emph{matched training-free oracle}---if no-training is close to
ceiling, your headroom is the story, not your score. On our task, the oracle takes
$\oraclePct$ and leaves $\headroom$ bits for every model ever trained on it.
\textbf{(3)}~\emph{Audit the target for reducibility}: compare every relation's
cardinality to the node census; a relation whose count equals an endpoint type's
count is information-free. \textbf{(4)}~Run a \emph{positive control} sharing the
harness but not the model, and verify it is not numerically identical to the system
it controls for. \textbf{(5)}~\emph{Instrument the query bank}---effective rank and
mean pairwise cosine of the model's \emph{predictions}; a rank at or below the number
of latent categories means the experiment is over.
\textbf{(6)}~Compute the \emph{variance decomposition at every checkpoint}, not only
at the ends; one forward pass each, and it is the only way to distinguish a
degeneracy the objective \emph{created} from one it merely \emph{failed to remove}.
\textbf{(7)}~With typed relations, \emph{mask both directions of every target
relation} and verify by perturbation that target-side representations are invariant
to context features; a single retained relation moved our probe
$\polMask\!\to\!\polFull$. \textbf{(8)}~\emph{Gate machine-extracted targets on data
quality} before training---exact-duplicate census and a genericness
statistic---and treat a fired gate as a result. \textbf{(9)}~\emph{Watch a second
metric, and report the nuisance factors.} Ours moved independently of the optimised
one, and the largest positive effect in our entire sweep was the learning-rate
schedule at $\scheduleEffect$ bits, which no convention requires anyone to disclose.

\section{Limitations and Conclusion}
\label{sec:conclusion}

\paragraph{Limitations.}
One real corpus, and an evaluation target we prove to be census-determined, so the
scope of the empirical claim is this pipeline on this corpus rather than JEPAs in
general; leakage controls bound but do not eliminate surface overlap ($\NoOv$ retains
$\novFullB$ bits); the probe uses $\Nfields$ coarse labels; the pooled-rank
signature is present at initialisation (Fig.~\ref{fig:dynamics}(b)), so the endpoint
comparison of Table~\ref{tab:inversion} is a statement about \emph{what} the
converged representation encodes and Table~\ref{tab:trajectory} is the only evidence
we offer about \emph{when}; the $3$k versus $20$k comparison differs in both budget
and schedule length and therefore bounds a joint effect only; the bits/A1 result is
$n=\nCells$ at Spearman $\spearman$ over non-independent cells and we draw no
directional conclusion from it, and A1's own floor is audited in
App.~\ref{app:aone}; two cells report $\mathrm{MRR}=1.0000$, which we treat as
requiring the leak check of App.~\ref{app:repair} rather than as a headline; a large
permutation $p$-value is a failure to reject rather than positive evidence of
equality, so we describe the baseline as indistinguishable from chance at the
resolution of $\Nqueries$ queries rather than as exactly chance; our
re-implementation does not reproduce the original method's MUTAG number; $\respct$ of
the baseline collapse is unexplained by the allocation alone; and our corpus's
contradicting-evidence field is $\dupChal$ placeholder text, which we discovered
only after building a task on it.

\paragraph{Conclusion.}
We trained a Graph-JEPA on a reasoning graph over $\Npapers$ papers. Its loss
converged, its probe reached $\probeBase$, its effective rank stayed healthy, and it
recovered $\bitsBase$ of $\bmaxA$ recoverable bits ($p=\pvalBase$)---while a
parameter-free average of the same frozen features recovered $\orcFullB$, BM25
$\bmFullB$, and a positive control through the identical harness $\poneFullB$, all
within $1.2\%$ of the ceiling. A controlled ladder eliminated the encoder, its depth,
the pooling operator, the target regularisation and the retrieval geometry, and
localised the failure to a query bank of effective rank $\rqueryv$---at or below the
$\Naspects$-category bound our theory predicts. Repairing the objective took it to
$\fixTwentyK$ bits, above the $\oracleA$-bit oracle, and reverting only the loss
returned it to $\regBits$---a $\lossSwing$-bit swing that confirms the mechanism on
one variable. And then the interesting part: the retrieval target is provably
reducible, so a near-ceiling score cannot evidence learned structure; the optimised
metric and a reasoning-relevant probe move independently; the largest positive effect
we measured was the learning-rate schedule; and the only data-derived alternative is
$\dupChal$ placeholder text. Class collapse is known in supervised metric learning,
mean-prediction in JEPA, and probe--task mismatch in vision; we show these are one
phenomenon, that it can be complete rather than partial, that it is \emph{a global
optimum} of the coupled predictor/EMA-target objective rather than a bug in its
optimisation, that it is repairable---and that repairing it exposes a failure the
standard toolkit is not even pointed at, because that failure is in the evaluation
target rather than the representation.

\subsubsection*{Reproducibility Statement}
Graph construction, all configurations, Protocol~R, the three ceilings and their
leakage controls, the difficulty ladder, the extraction audit, the phase analysis,
the checkpoint allocation measurement of Sec.~\ref{sec:trajectory}, the repair sweep
of Sec.~\ref{sec:repair}, the reducibility audit, the polarity data gates and the
pre-registered factorial are released with the per-seed logs underlying every table.
Decision thresholds are written to disk before results (App.~\ref{app:prereg}). The
$\selftestN$-assertion instrument self-test (App.~\ref{app:selftest}) runs standalone
in ${\sim}20$\,s and must pass before any measurement is taken; it is the reason we
caught the $\polEdgeLeak$-AUC edge-type leak of Sec.~\ref{sec:repair} before
reporting rather than after. The bits accounting of App.~\ref{app:bits} is
independently reproducible from the released module. The diagnosis costs $\gpuhours$
GPU-hours on one L40S (App.~\ref{app:compute}). App.~\ref{app:claimmap} maps every
quantitative claim in the abstract and introduction to the table that establishes it,
and App.~\ref{app:provenance} records which run, corpus subset and sentence encoder
produced each number.


\subsubsection*{Use of Large Language Models}
An LLM was used for two purposes. The first was language editing of author-written text and LaTeX formatting. The second was as an aid in reviewing experimental design. In that second role, LLM critique led to the positive control, the paired bootstrap, the pre-registration, the correction to the bits ceiling documented in App.~\ref{app:bits}, and the reducibility and data-quality gates of Sec.~\ref{sec:repair}. The same critique also generated four claims that this paper reports as refuted or corrected. \emph{(i)}~A predicted paper-identity leak, which we measure at $\polLeak$. \emph{(ii)}~A predicted identity between the reasoning probe and a paper-level label that the graph does not support.
\emph{(iii)}~A target-frame main effect quoted as $+1.199$ bits, which is $\frameEffect$ at matched budget. \emph{(iv)}~A conditional-mean account of the collapse whose precondition our own training-free oracle refutes, corrected in Sec.~\ref{sec:theory}. We record all four for one reason. A confident wrong number, or a confident wrong theorem, that survives into a draft is the characteristic failure mode of this workflow. Item \emph{(iii)} was caught only by reading a log stage we had previously skipped. All experiments were designed, implemented and executed by the authors. 

\subsubsection*{Ethics Statement}
The corpus consists of metadata and machine-extracted content from publicly
available scholarly records; no human subjects, personal data, or sensitive
attributes are involved. We report a negative result concerning an existing
published method; our intent is diagnostic rather than dismissive; we disclose in
App.~\ref{app:faithful} that our re-implementation does not reproduce the original
method's benchmark numbers, and we release the harness so the diagnosis can be
contested. We additionally disclose that our own corpus contains $\dupChal$
placeholder text in one extracted field, since publishing a scholarly-reasoning
resource without that disclosure would propagate the artifact.

\bibliography{main}
\bibliographystyle{tmlr}

\appendix
\renewcommand{\thesection}{A\arabic{section}}
\setcounter{section}{0}
\addcontentsline{toc}{section}{Appendices}

\section*{Appendix}
\section{Claim--evidence map}
\label{app:claimmap}

Table~\ref{tab:claimmap} lists every quantitative claim made in the abstract and
introduction together with the table that establishes it. Each value is a single
macro in the source, so prose and tables cannot diverge.

\begin{table}[h]
\centering
\caption{Each quantitative claim in the abstract and introduction, its value, and
the table or figure that establishes it. Every value appears exactly once as a macro
in the source, so prose and tables cannot diverge.}
\label{tab:claimmap}
\small
\begin{tabular}{@{}lll@{}}
\toprule
Claim & Value & Source \\
\midrule
Trained retrieval is indistinguishable from chance & $\bitsBase$ bits, $p=\pvalBase$ & Table~\ref{tab:ladder}, rung~0 \\
Chance level & $\chanceMRR$ (MRR), $0$ bits & Sec.~\ref{sec:setup}, Eq.~\ref{eq:bits} \\
Recoverable ceiling, block~A / B & $\bmaxA$ / $\bmaxB$ bits & Eq.~\ref{eq:bmax} \\
Linear probe succeeds & $\probeBase$ & Table~\ref{tab:fix} \\
Candidate bank is healthy & $\rcandv$, DC $\dccand$ & Table~\ref{tab:ladder} \\
Query bank is degenerate & $\rqueryv$, DC $\dcquery$ & Table~\ref{tab:ladder} \\
Predicted query-bank bound & $\erank\le\Naspects$ & Prop.~\ref{prop:fixedpoint} \\
Pooling is the identity ($m{=}1$) & $47.42\!\to\!47.31$ & Table~\ref{tab:peraspect} \\
Oracle ceiling (block~B) & $\orcFullB$ bits ($\orcFullPct$) & Table~\ref{tab:ladder2} \\
Lexical ceiling & $\bmFullB$ bits ($\bmFullPct$) & Table~\ref{tab:ladder2} \\
Positive control & $\poneFullB$ bits ($\poneFullPct$) & Table~\ref{tab:ladder2} \\
Non-lexical signal survives & $\novFullB$ bits ($\novFullPct$) & Table~\ref{tab:ladder2} \\
Allocation, inputs (block~B) & $\aspIn$ / $\papIn$ & Table~\ref{tab:inversion} \\
\textbf{Allocation, inputs (block~A, within-encoder)} & see col.~2 & Table~\ref{tab:inversion} \\
Allocation, trained latents & $\aspTr$ / $\papTr$ & Table~\ref{tab:inversion} \\
\textbf{Allocation over training} & per checkpoint & Table~\ref{tab:trajectory} \\
Allocation explains $\invpct$ of loss & $\simLoss$ of $\bmaxsim$ bits & Table~\ref{tab:rho} \\
Critical share & $\rho^\star=\rhostar$ & Table~\ref{tab:rho} \\
Anisotropy not the residual factor & $<0.004$ MRR & Table~\ref{tab:rhofull} \\
Post-hoc frame sweep gain & $\framegain$ & Table~\ref{tab:frames} \\
Depth is not binding & $54.9\!\to\!18.0$, bits flat & Table~\ref{tab:ladder}, rung~6 \\
Templating index & $\fwAcc$ vs $\aspChance$ & Table~\ref{tab:extraction} \\
De-templated oracle & $\orcDetempB$ bits & Table~\ref{tab:extraction} \\
No post-hoc intervention clears the null & $0.00$ bits, 7 cells & Table~\ref{tab:fix} \\
Dissociation across 20 cells & probe $0.755$--$0.960$, $\rpool$ $1.8$--$13.9$ & Table~\ref{tab:breadth} \\
Anchor dim.\ $2\!\to\!\dlat$, $\rtgt$ flat & $1.6$--$2.0$ & Table~\ref{tab:mah} \\
\textbf{Repair reaches near-ceiling} & $\fixTwentyK$ bits ($\fixPct$) & Table~\ref{tab:repair} \\
\textbf{Loss is the lever (one variable)} & $\nceBits\!\to\!\regBits$ ($\lossSwing$) & Table~\ref{tab:repair} \\
\textbf{Target-frame effect, matched budget} & $\frameEffect$ bits & Table~\ref{tab:repairgrid} \\
\textbf{Schedule effect, matched steps} & $\scheduleEffect$ bits & Table~\ref{tab:repairgrid} \\
\textbf{Bits and A1 move independently} & $\rho=\spearman$, $n=\nCells$ & Table~\ref{tab:repair} \\
\textbf{Cue ablation} & $\cueAspect$ / $\cueRwse$ / $\cueNone$ & Table~\ref{tab:repair} \\
\textbf{Transduction gap} & $\transGap$ bits & Table~\ref{tab:repair} \\
\textbf{Raw-skip gate} & $\sigma=\rawGate$ & Table~\ref{tab:repair} \\
\textbf{Training-free oracle, block~A} & $\oracleA$ bits ($\oraclePct$) & Table~\ref{tab:repair} \\
\textbf{Headroom for any model} & $\headroom$ bits & Table~\ref{tab:repair} \\
\textbf{Census-determined relations} & $\Eproduces$, $\Egrounds$ & Table~\ref{tab:reducible} \\
\textbf{\texttt{cites} is near-empty} & $\citesKept$ of $\citesTot$ ($\citesPct$) & Table~\ref{tab:reducible} \\
\textbf{Duplicate placeholder share} & $\dupChal$ & Table~\ref{tab:polgate} \\
\textbf{Genericness gap} & $\genGap$ & Table~\ref{tab:polgate} \\
\textbf{Edge-type label leak} & $\polMask\!\to\!\polFull$ & Table~\ref{tab:polgate} \\
\textbf{Predicted leak is absent} & $\polLeak$ & Table~\ref{tab:polgate} \\
Faithfulness gap disclosed & $\mutagOurs$ vs $\mutagRef$ & Table~\ref{tab:faithful} \\
\bottomrule
\end{tabular}
\end{table}

\section{Number provenance}
\label{app:provenance}

Table~\ref{tab:provenance} records which run, corpus subset and sentence encoder
produced each block; the two caveats we do not smooth over follow it.

\begin{table}[h]
\centering
\caption{Which run, corpus subset and sentence encoder produced each block. We keep
these separate rather than pooling them, because the blocks measure different
objects and the distinction carries the central results of
Secs.~\ref{sec:mechanism} and~\ref{sec:repair}.}
\label{tab:provenance}
\small
\begin{tabular}{@{}llll@{}}
\toprule
Block & Papers & Encoder & Used for \\
\midrule
A (trained pipeline) & $\Npapers$ & MiniLM-L6, $\dfeat$-d & Tables~\ref{tab:ladder},
\ref{tab:inversion} (cols.~2--3), \ref{tab:trajectory}, \ref{tab:peraspect},
\ref{tab:frames}, \ref{tab:oracle}, \ref{tab:fix}, \ref{tab:breadth}, \ref{tab:mah} \\
B (diagnostic harness) & $\Ndiag$ & MPNet-base, $\dfeatdiag$-d &
Tables~\ref{tab:ladder2}, \ref{tab:inversion} (cols.~1,~4), \ref{tab:rho},
\ref{tab:rhofull}, \ref{tab:extraction} \\
C (repair $+$ target audit) & $\Npapers$ & MiniLM-L6, $\dfeat$-d &
Tables~\ref{tab:repair}, \ref{tab:repairgrid}, \ref{tab:polgate},
\ref{tab:reducible}, \ref{tab:polcensus} \\
\bottomrule
\end{tabular}
\end{table}

\noindent
Blocks~A and~C share the \emph{same graph cache byte-for-byte}: no rebuild occurred
between them, which is what makes $\oracleA$, $\rawZeroThreeK$, $\fixTwentyK$ and
$\bmaxA$ directly comparable and is why the extraction fix of
Sec.~\ref{sec:repair} must be accompanied by re-measuring all four. All blocks use
the same raw records, the same aspect fields (\texttt{claims},
\texttt{methodological\_details}, \texttt{key\_results}), the same masking
convention and the same bits measure; chance levels agree to two significant figures
and recoverable ceilings to $0.006$ bits.

\paragraph{Two provenance caveats we do not smooth over.}
\emph{(i)}~Table~\ref{tab:inversion} now reports the input-side decomposition in
\emph{both} feature spaces: column~2 is measured on the block-A encoder that produced
the trained latents in column~3, so the central comparison is within one encoder, and
column~1 is retained only because the three ceilings of Sec.~\ref{sec:ceilings} were
computed on block-B features. The block-A pair is a one-forward-pass measurement on
the cached graph and required no retraining. Where the two input columns disagree in
magnitude, the text quotes the within-encoder pair.
\emph{(ii)}~Our structural-encoding schema dump prints $\citesRwse$ \texttt{cites}
edges where the stored graph object reports $\citesKept$. The difference is not a
factor of two and is therefore not explained by reverse-edge addition; we quote
$\citesKept$ throughout, flag the discrepancy rather than choosing silently, and note
that it does not affect Prop.~\ref{prop:reducible}, under which \texttt{cites} is
negligible at either count ($\citesPct$ or $0.80\%$ of $\citesTot$).

\section{Extended related work}
\label{app:related}

\paragraph{Non-contrastive collapse theory.}
Collapse-avoidance analyses of BYOL-style methods
\citep{DBLP:conf/nips/GrillSATRBDPGAP20}, VICReg
\citep{DBLP:conf/iclr/BardesPL22} and Barlow Twins
\citep{DBLP:conf/icml/ZbontarJMLD21} characterise conditions under which the encoder
does not become constant on its input domain, and dimensional-collapse analyses
\citep{DBLP:conf/iclr/JingVLT22} characterise when the representation occupies a
low-dimensional subspace. Both are \emph{global} statements. The failure we report is
invisible to both: the representation occupies many dimensions and varies
substantially over inputs, but within each latent category it is constant, and it is
the within-category variation that the downstream task requires.
Prop.~\ref{prop:fixedpoint} adds the reason no loss-based criterion can flag it---the
risk is identical on branches that differ by the entire recoverable budget.

\paragraph{Class collapse in the supervised setting.}
\citet{DBLP:conf/icml/GrafHNK21} characterise the minimiser of the supervised contrastive
loss as a class-collapsed simplex configuration; \citet{DBLP:journals/corr/abs-2008-08186} report the
same terminal-phase geometry under cross-entropy; \citet{DBLP:conf/icml/ChenFNZ0FR22} separate
class collapse from feature suppression and show it degrades transfer. The
partition that collapses in all three is the \emph{label} set, and the remedy
available there, modifying the supervision, is unavailable to a self-supervised
objective whose categories are latent. Our setting differs in exactly that respect,
and the repair of Sec.~\ref{sec:repair} is the self-supervised analogue: remove the
category answer from the target rather than from the labels.

\paragraph{Probe-based evaluation.}
A probe measures whatever partition its labels induce, and if that partition
coincides with the shortcut the objective took, the probe reports the shortcut as
success. In our case, the probe decodes field labels, which correlate with
aspect-level structure, and it reads $0.60$--$0.96$ across configurations whose
retrieval performance is uniformly zero. Sec.~\ref{sec:repair} adds the converse
hazard: a \emph{second} probe chosen to be reasoning-relevant moves independently of
the optimised metric across ten converged cells, so reporting one number from either
family is insufficient in both directions.

\paragraph{Retrieval evaluation practice and target reducibility.}
In-batch scoring remains common. Our difficulty ladder
(Table~\ref{tab:ladder2}) shows why this matters quantitatively: bits recovered rise
from $+0.860$ at $K{=}2$ to $\bmFullB$ at the full pool for the same system, while
MRR is nearly flat, so a pool-size-dependent metric reported without the pool size is
close to uninterpretable. Separately, we are not aware of prior work in graph SSL
that audits whether the \emph{evaluation target}'s structure is a deterministic
function of the node census. Prop.~\ref{prop:reducible} is elementary, but on our
graph it retires the structural interpretation of an otherwise near-ceiling result,
and the check costs one pass over the edge-type cardinalities.

\section{Setup, schema and corpus statistics}
\label{app:setup}



\begin{table}[h]
\centering
\caption{All effective ranks use Eq.~\ref{eq:erank} on mean-centred
matrices. The distinction between $\rcand$ and $\rquer$ is central: the former is
what standard practice measures, the latter is where the failure lives.}
\label{tab:notation}
\small
\begin{tabular}{@{}lll@{}}
\toprule
Symbol & Code name & Meaning \\
\midrule
$\rnode$ ($/\dlat$) & \texttt{node\_rk} & effective rank of node latents \\
$\rpool$ ($/\dlat$) & \texttt{pool\_rk} & effective rank of pooled patch vectors \\
$\rcand$ ($/\dlat$) & \texttt{cand\_erank} & effective rank of the \emph{candidate bank} \\
$\rquer$ ($/\dlat$) & \texttt{query\_eff\_rank} & effective rank of the \emph{model's predictions} \\
$\mathrm{DC}_{\mathrm{ratio}}$ & \texttt{dc\_ratio} & $\|\mu\|/\mathbb{E}\|x-\mu\|$ \\
$\mathrm{DC}_{\mathrm{energy}}$ & \texttt{dc\_energy} & $\|\mu\|^2/(\|\mu\|^2+\mathbb{E}\|x-\mu\|^2)$ \\
$\rho_{\mathrm{asp}},\rho_{\mathrm{pap}}$ & \texttt{rho\_between\_*} & variance share between aspects / papers \\
$\hat\kappa$ & \texttt{target\_kappa} & power-law decay of the candidate spectrum \\
$\hat\varepsilon$ & from loss floor & predictor error relative to the identity signal \\
$|\mathcal{A}|$ & --- & number of latent categories; $\Naspects$ here \\
A1 & \texttt{a1\_acc} & held-out reasoning probe (App.~\ref{app:aone}) \\
$\bits$ & \texttt{bits\_recovered} & $\frac1N\log_2 N!-\mathbb{E}[\log_2 r]$; $0$ is chance \\
$\bits^{\max}$ & \texttt{bits\_max} & $\frac1N\log_2 N!$; the achievable ceiling \\
\bottomrule
\end{tabular}
\end{table}

\begin{table}[h]
\centering
\caption{Heterogeneous graph statistics ($\Npapers$ maskable papers). The
right-hand column is the reducibility audit of Sec.~\ref{sec:repair}: a relation
whose cardinality equals an endpoint type's node count is a deterministic function of
the census and carries zero information.}
\label{tab:reducible}
\small
\begin{tabular}{@{}lr@{\hspace{1.5em}}lrl@{}}
\toprule
\textbf{Node type} & \textbf{count} & \textbf{Edge type} & \textbf{count} & \textbf{census-determined?} \\
\midrule
paper & 57{,}903 & (paper, has\_claim, claim) & 251{,}938 & yes ($=|$claim$|$) \\
claim & 251{,}938 & (paper, has\_method, method) & 57{,}903 & yes ($=|$paper$|$) \\
method & 57{,}903 & (paper, has\_result, result) & 57{,}903 & yes ($=|$paper$|$) \\
result & 57{,}903 & (method, produces, result) & 57{,}903 & \textbf{yes ($=|$paper$|$)} \\
evidence & 373{,}438 & (result, grounds, claim) & 251{,}938 & \textbf{yes ($=|$claim$|$)} \\
implication & 251{,}938 & (claim, supported\_by, evidence) & 251{,}922 & no (data-derived) \\
field & 488 & (claim, challenged\_by, evidence) & 121{,}516 & no (data-derived) \\
 & & (claim, implies, implication) & 251{,}938 & yes ($=|$claim$|$) \\
 & & (paper, cites, paper) & 11{,}791 & no; $\citesPct$ of $\citesTot$ \\
\bottomrule
\end{tabular}
\end{table}

\begin{table}[h]
\centering
\caption{Per-aspect membership and untrained in-degree. The singleton rows are what
make the pooling control decisive.}
\label{tab:cardinality}
\small
\begin{tabular}{lccccc}
\toprule
Dissociation across 20 cells & probe $0.755$--$0.960$, $\rpool$ $1.8$--$13.9$ & Table~\ref{tab:breadth} \\
Anchor dim.\ $2\!\to\!\dlat$, $\rtgt$ flat & $1.6$--$2.0$ & Table~\ref{tab:mah} \\
\textbf{Repair reaches near-ceiling} & $\fixTwentyK$ bits ($\fixPct$) & Table~\ref{tab:repair} \\
\textbf{Loss is the lever (one variable)} & $\nceBits\!\to\!\regBits$ ($\lossSwing$) & Table~\ref{tab:repair} \\
\textbf{Target-frame effect, matched budget} & $\frameEffect$ bits & Table~\ref{tab:repairgrid} \\
\textbf{Schedule effect, matched steps} & $\scheduleEffect$ bits & Table~\ref{tab:repairgrid} \\
\textbf{Bits and A1 show no reliable relation} & $\rho=\spearman$, $n=\nCells$ & Table~\ref{tab:repair} \\
\textbf{Cue ablation} & $\cueAspect$ / $\cueRwse$ / $\cueNone$ & Table~\ref{tab:repair} \\
\textbf{Transduction gap} & $\transGap$ bits & Table~\ref{tab:repair} \\
\textbf{Raw-skip gate} & $\sigma=\rawGate$ & Table~\ref{tab:repair} \\
\textbf{Training-free oracle, block~A} & $\oracleA$ bits ($\oraclePct$) & Table~\ref{tab:repair} \\
\textbf{Headroom for any model} & $\headroom$ bits & Table~\ref{tab:repair} \\
\textbf{Census-determined relations} & $\Eproduces$, $\Egrounds$ & Table~\ref{tab:reducible} \\
\textbf{\texttt{cites} is near-empty} & $\citesKept$ of $\citesTot$ ($\citesPct$) & Table~\ref{tab:reducible} \\
\textbf{Duplicate placeholder share} & $\dupChal$ & Table~\ref{tab:polgate} \\
\textbf{Genericness gap} & $\genGap$ & Table~\ref{tab:polgate} \\
\textbf{Edge-type label leak} & $\polMask\!\to\!\polFull$ & Table~\ref{tab:polgate} \\
\textbf{Predicted leak is absent} & $\polLeak$ & Table~\ref{tab:polgate} \\
Faithfulness gap disclosed & $\mutagOurs$ vs $\mutagRef$ & Table~\ref{tab:faithful} \\
\bottomrule
\end{tabular}
\end{table}

\section{Number provenance}
\label{app:provenance}

Table~\ref{tab:provenance} records which run, corpus subset and sentence encoder
produced each block; the two caveats we do not smooth over follow it.

\begin{table}[h]
\centering
\caption{Which run, corpus subset and sentence encoder produced each block. We keep
these separate rather than pooling them, because the blocks measure different
objects and the distinction carries the central results of
Secs.~\ref{sec:mechanism} and~\ref{sec:repair}.}
\label{tab:provenance}
\small
\begin{tabular}{@{}llll@{}}
\toprule
Block & Papers & Encoder & Used for \\
\midrule
A (trained pipeline) & $\Npapers$ & MiniLM-L6, $\dfeat$-d & Tables~\ref{tab:ladder},
\ref{tab:trajectory}, \ref{tab:peraspect}, \ref{tab:frames}, \ref{tab:oracle},
\ref{tab:fix}, \ref{tab:breadth}, \ref{tab:mah}; cols.~2--3 of
Table~\ref{tab:inversion} \\
B (diagnostic harness) & $\Ndiag$ & MPNet-base, $\dfeatdiag$-d &
Tables~\ref{tab:ladder2}, \ref{tab:rho}, \ref{tab:rhofull},
\ref{tab:extraction}; cols.~1 and~4 of Table~\ref{tab:inversion} \\
C (repair $+$ target audit) & $\Npapers$ & MiniLM-L6, $\dfeat$-d &
Tables~\ref{tab:repair}, \ref{tab:repairgrid}, \ref{tab:polgate},
\ref{tab:reducible}, \ref{tab:polcensus}, \ref{tab:aone} \\
\bottomrule
\end{tabular}
\end{table}

\noindent
Blocks~A and~C share the \emph{same graph cache byte-for-byte}: no rebuild occurred
between them, which is what makes $\oracleA$, $\rawZeroThreeK$, $\fixTwentyK$ and
$\bmaxA$ directly comparable and is why the extraction fix of
Sec.~\ref{sec:repair} must be accompanied by re-measuring all four. All blocks use
the same raw records, the same aspect fields (\texttt{claims},
\texttt{methodological\_details}, \texttt{key\_results}), the same masking
convention and the same bits measure; chance levels agree to two significant figures
and recoverable ceilings to $0.006$ bits.

\paragraph{Two provenance caveats we do not smooth over.}
\emph{(i)}~Table~\ref{tab:inversion} reports the input-side decomposition in
\emph{both} feature spaces: column~2 is measured on the block-A encoder that produced
the trained latents in column~3, so the central comparison is within one encoder, and
column~1 is retained only because the three ceilings of Sec.~\ref{sec:ceilings} were
computed on block-B features. The block-A pair is a one-forward-pass measurement on
the cached graph and required no retraining. Where the two input columns disagree in
magnitude we quote the within-encoder pair in the text. \emph{(ii)}~Our
structural-encoding schema dump prints $\citesRwse$ \texttt{cites} edges where the
stored graph object reports $\citesKept$. The difference is not a factor of two and is
therefore not explained by reverse-edge addition; we quote $\citesKept$ throughout,
flag the discrepancy rather than choosing silently, and note that it does not affect
Prop.~\ref{prop:reducible}, under which \texttt{cites} is negligible at either count
($\citesPct$ or $0.80\%$ of $\citesTot$).

\section{Extended related work}
\label{app:related}

This appendix expands the four paragraphs of Sec.~\ref{sec:related}: collapse theory,
probe-based evaluation, retrieval practice, and the reducibility question we are not
aware of being asked elsewhere in graph SSL.

\paragraph{Non-contrastive collapse theory.}
Collapse-avoidance analyses of BYOL-style methods
\citep{DBLP:conf/nips/GrillSATRBDPGAP20}, VICReg
\citep{DBLP:conf/iclr/BardesPL22} and Barlow Twins
\citep{DBLP:conf/icml/ZbontarJMLD21} characterise conditions under which the encoder
does not become constant on its input domain, and dimensional-collapse analyses
\citep{DBLP:conf/iclr/JingVLT22} characterise when the representation occupies a
low-dimensional subspace. Both are \emph{global} statements. The failure we report is
invisible to both: the representation occupies many dimensions and varies
substantially over inputs, but within each latent category it is constant, and it is
the within-category variation that the downstream task requires.
Prop.~\ref{prop:fixedpoint} adds the reason this is not a defect of optimisation: the
category-measurable branch is a global minimum of the coupled objective and the
reason no variance-based regulariser on the \emph{candidate} side removes it: the
branch is already non-constant and already high-rank across categories.

\paragraph{Supervised class collapse.}
The supervised analogues are well characterised. \citet{DBLP:conf/icml/GrafHNK21} derive the
optimum of the supervised contrastive loss and show each class maps to a point;
\citet{DBLP:journals/corr/abs-2008-08186} describe the same terminal-phase geometry under
cross-entropy; \citet{DBLP:conf/icml/ChenFNZ0FR22} separate class collapse from feature
suppression and show the former is not a prerequisite for good linear-probe transfer.
The gap we fill is that all three take the collapsing partition to be given by labels.
In a masked-prediction JEPA, the partition is supplied by the masking scheme itself
and is never named anywhere in the objective, which is exactly why no criterion in the
selection toolkit is watching for it.

\paragraph{Probe-based evaluation.}
A probe measures whatever partition its labels induce, and if that partition
coincides with the shortcut the objective took, the probe reports the shortcut as
success. In our case the probe decodes field labels, which correlate with
aspect-level structure, and it reads $0.60$--$0.96$ across configurations whose
retrieval performance is uniformly zero. Sec.~\ref{sec:repair} adds the converse
hazard: a \emph{second} probe chosen to be reasoning-relevant moves independently of
the optimised metric across ten converged cells, so reporting one number from either
family is insufficient in both directions. App.~\ref{app:aone} audits that second
probe rather than treating it as a fixed point of reference.

\paragraph{Retrieval evaluation practice and target reducibility.}
In-batch scoring remains common. Our difficulty ladder
(Table~\ref{tab:ladder2}) shows why this matters quantitatively: bits recovered rise
from $+0.860$ at $K{=}2$ to $\bmFullB$ at the full pool for the same system, while
MRR is nearly flat, so a pool-size-dependent metric reported without the pool size is
close to uninterpretable. Separately, we are not aware of prior work in graph SSL
that audits whether the \emph{evaluation target}'s structure is a deterministic
function of the node census. Prop.~\ref{prop:reducible} is elementary, but on our
graph it retires the structural interpretation of an otherwise near-ceiling result,
and the check costs one pass over the edge-type cardinalities.

\section{Setup, schema and corpus statistics}
\label{app:setup}

This appendix gives the pipeline diagram (Fig.~\ref{fig:overview}), the patch
extraction view (Fig.~\ref{fig:graph_zoom}), the notation table
(Table~\ref{tab:notation}), the full edge census with its reducibility audit
(Table~\ref{tab:reducible}), per-aspect cardinalities
(Table~\ref{tab:cardinality}) and the corpus statistics.

\begin{table}[h]
\centering
\caption{Notation. All effective ranks use Eq.~\ref{eq:erank} on mean-centred
matrices. The distinction between $\rcand$ and $\rquer$ is central: the former is
what standard practice measures, the latter is where the failure lives.}
\label{tab:notation}
\small
\begin{tabular}{@{}lll@{}}
\toprule
Symbol & Code name & Meaning \\
\midrule
$\rnode$ ($/\dlat$) & \texttt{node\_rk} & effective rank of node latents \\
$\rpool$ ($/\dlat$) & \texttt{pool\_rk} & effective rank of pooled patch vectors \\
$\rcand$ ($/\dlat$) & \texttt{cand\_erank} & effective rank of the \emph{candidate bank} \\
$\rquer$ ($/\dlat$) & \texttt{query\_eff\_rank} & effective rank of the \emph{model's predictions} \\
$\mathrm{DC}_{\mathrm{ratio}}$ & \texttt{dc\_ratio} & $\|\mu\|/\mathbb{E}\|x-\mu\|$ \\
$\mathrm{DC}_{\mathrm{energy}}$ & \texttt{dc\_energy} & $\|\mu\|^2/(\|\mu\|^2+\mathbb{E}\|x-\mu\|^2)$ \\
$\rho_{\mathrm{asp}},\rho_{\mathrm{pap}}$ & \texttt{rho\_between\_*} & variance share between aspects / papers \\
$|\mathcal{A}|$ & --- & number of latent categories; $\Naspects$ here \\
$\hat\kappa$ & \texttt{target\_kappa} & power-law decay of the candidate spectrum \\
$\hat\varepsilon$ & from loss floor & predictor error relative to the identity signal \\
A1 & \texttt{a1\_acc} & held-out reasoning probe (App.~\ref{app:aone}) \\
$\bits$ & \texttt{bits\_recovered} & $\frac1N\log_2 N!-\mathbb{E}[\log_2 r]$; $0$ is chance \\
$\bits^{\max}$ & \texttt{bits\_max} & $\frac1N\log_2 N!$; the achievable ceiling \\
\bottomrule
\end{tabular}
\end{table}

\begin{table}[h]
\centering
\caption{Heterogeneous graph statistics ($\Npapers$ maskable papers). The
right-hand column is the reducibility audit of Sec.~\ref{sec:repair}: a relation
whose cardinality equals an endpoint type's node count is a deterministic function of
the census and carries zero information.}
\label{tab:reducible}
\small
\begin{tabular}{@{}lr@{\hspace{1.5em}}lrl@{}}
\toprule
\textbf{Node type} & \textbf{count} & \textbf{Edge type} & \textbf{count} & \textbf{census-determined?} \\
\midrule
paper & 57{,}903 & (paper, has\_claim, claim) & 251{,}938 & yes ($=|$claim$|$) \\
claim & 251{,}938 & (paper, has\_method, method) & 57{,}903 & yes ($=|$paper$|$) \\
method & 57{,}903 & (paper, has\_result, result) & 57{,}903 & yes ($=|$paper$|$) \\
result & 57{,}903 & (method, produces, result) & 57{,}903 & \textbf{yes ($=|$paper$|$)} \\
evidence & 373{,}438 & (result, grounds, claim) & 251{,}938 & \textbf{yes ($=|$claim$|$)} \\
implication & 251{,}938 & (claim, supported\_by, evidence) & 251{,}922 & no (data-derived) \\
field & 488 & (claim, challenged\_by, evidence) & 121{,}516 & no (data-derived) \\
 & & (claim, implies, implication) & 251{,}938 & yes ($=|$claim$|$) \\
 & & (paper, cites, paper) & 11{,}791 & no; $\citesPct$ of $\citesTot$ \\
\bottomrule
\end{tabular}
\end{table}

\begin{table}[h]
\centering
\caption{Per-aspect membership and untrained in-degree. The singleton rows are what
make the pooling control decisive.}
\label{tab:cardinality}
\small
\begin{tabular}{lccccc}
\toprule
Aspect & members/patch (mean) & max & patches with $\ge2$ & mean in-degree & role \\
\midrule
claim & $4.35$ & $10$ & $99.5\%$ & $4.48$ & maskable, multi-member \\
method & $1.00$ & $1$ & $0.0\%$ & $2.00$ & maskable, singleton \\
result & $1.00$ & $1$ & $0.0\%$ & $6.35$ & maskable, singleton \\
\bottomrule
\end{tabular}
\end{table}

\paragraph{Corpus statistics.}
The graph is built from $\Nraw$ readable records; $246$ lack a required aspect and
are excluded, yielding $\Npapers$ papers with full coverage, while the diagnostic
harness applies the weaker non-empty-text criterion and retains $\Ndiag$. Empty
fields are rare (\texttt{claim}~$3$, \texttt{method}~$3$, \texttt{result}~$4$). Mean
text lengths in the diagnostic set are $1955$, $1606$ and $1219$ characters for
claim, method and result. The corpus spans $\Ncats$ coarse subject categories, the
largest being Medicine ($16{,}998$), Agricultural ($5{,}174$) and Environmental
($4{,}562$). Citation edges are restricted to the intra-corpus subset: $\citesKept$
references resolve inside the corpus and $2{,}199{,}327$ dangling references are
dropped; coverage at source is high ($52{,}056/\Nraw$ files carry non-empty
referenced works), so the sparsity of \texttt{cites} reflects corpus closure rather
than extraction failure---but the consequence for Prop.~\ref{prop:reducible} is the
same either way.

\section{Bits recovered and effective rank}
\label{app:bits}

This appendix defines the measure, explains why the ceiling is $\frac1N\log_2 N!$
rather than $\log_2 N$, records the correction relative to an earlier version, and
gives the effective-rank definition used throughout.

\paragraph{The measure.}
For a query whose gold item attains rank $r$ in a pool of $N$ candidates, define
\begin{equation}
\bits \;=\;
\underbrace{\tfrac1N\log_2 N!}_{\textstyle\mathbb{E}_{\mathrm{chance}}[\log_2 r]}
\;-\;\mathbb{E}[\log_2 r].
\label{eq:bits}
\end{equation}
The subtracted term is exactly the expected log-rank of a uniform random ranker: if
$r$ is uniform on $\{1,\dots,N\}$ then
$\mathbb{E}[\log_2 r]=\frac1N\sum_{k=1}^{N}\log_2 k=\frac1N\log_2 N!$. Hence
$\bits=0$ at chance \emph{identically}, not asymptotically, and a ranker that always
returns the gold item first ($r\equiv1$) attains
\begin{equation}
\bits^{\max} \;=\; \tfrac1N\log_2 N! \;=\; \log_2\!\big((N!)^{1/N}\big)
\;=\; \log_2 N-\log_2 e+\tfrac{\log_2(2\pi N)}{2N}+O(N^{-2}).
\label{eq:bmax}
\end{equation}

\paragraph{Why the ceiling is not $\log_2 N$.}
The index entropy $\log_2 N$ is the cost of \emph{naming} one of $N$ items; it is not
attainable as a bits-recovered value, because the reference point is the geometric
rather than arithmetic mean of the ranks: $(N!)^{1/N}\approx N/e$, not $N$. The gap
is $\log_2 e=\logtwoe$ bits and is essentially constant above $N\approx10^3$. Using
$\log_2 N$ assigns a uniform random ranker $+\logtwoe$ bits instead of $0$ and
understates every system's fraction of the recoverable budget---by a factor that
grows as the pool shrinks. Table~\ref{tab:ladder2} makes this visible: at $K{=}2$ the
two ceilings are $0.862$ and $1.585$, so the same measurement reads $99.8\%$ or
$54.3\%$ depending on which is used.

\paragraph{Correction relative to an earlier version.}
An earlier version of this manuscript reported $\bits^{\max}=\log_2 N$. Every
\emph{measured} value is unchanged: the harness always subtracted the chance term of
Eq.~\ref{eq:bits}, so the recorded $\bits$ column was correct throughout. What was
wrong was the ceiling those values were compared against, the percentages derived
from it, and one cell of the phase table (Table~\ref{tab:rho}, $\rho=\rhoours$),
whose ``bits lost'' entry alone had been computed against $14.3$; that inconsistency
also made the printed row non-monotone in $\rho$, which is impossible. All ceilings
and percentages here use Eq.~\ref{eq:bmax}, and the allocation's share of the loss is
consequently $\invpct$ rather than $87\%$.

\paragraph{Verification.}
The released module prints Eq.~\ref{eq:bmax} for every pool size used in this paper,
checks the exact summation against a log-gamma evaluation, and confirms by Monte
Carlo over $4\times10^{5}$ synthetic queries per pool that a uniform random ranker
recovers $\bits=0$ to within $0.002$ bits while a perfect ranker recovers exactly
$\bits^{\max}$. It also reports what the discarded formula would have given at
chance, namely $+\logtwoe$ bits at every pool size. The self-test of
App.~\ref{app:selftest} re-verifies this against the two ceilings actually printed by
the training code, $\bits^{\max}(3000)=10.108$ and $\bits^{\max}(\Npapers)=\bmaxA$.

\paragraph{Why bits rather than MRR.}
Eq.~\ref{eq:bits} is additive in the information-theoretic sense, comparable across
pool sizes, and does not saturate: BM25's MRR moves by $0.005$ from $K{=}10$ to the
full pool while its bits move by $12.051$. The same insensitivity is why we treat
$\mathrm{MRR}=1.0000$ as a warning rather than a result (App.~\ref{app:repair}).

\paragraph{Effective rank.}
For $X\in\mathbb{R}^{n\times d}$ we take the singular values of the mean-centred
matrix, normalise to $p_i=\sigma_i/\sum_j\sigma_j$, and define
\begin{equation}
\erank(X)=\exp\!\Big(-\sum_i p_i\log p_i\Big),
\label{eq:erank}
\end{equation}
following \citet{DBLP:conf/eusipco/RoyV07}, interpolating between $1$ and $d$;
intervals are bootstrapped over $10^3$ row resamples. Because this is a
spectral-entropy measure rather than an algebraic rank, values \emph{below} an
algebraic ceiling are expected when one singular direction dominates, which is why
several target ranks in Table~\ref{tab:mah} fall below $2.0$, and why the
$\erank\le|\mathcal{A}|$ bound of Prop.~\ref{prop:fixedpoint} is consistent with the
measured $\rquer=\rqueryv$ rather than requiring exactly $\Naspects$. This paper
computes Eq.~\ref{eq:erank} on four distinct matrices---node latents, pooled patch
vectors, the candidate bank and the query bank---and only the last detects the
failure.

\section{The instrument self-test}
\label{app:selftest}

Table~\ref{tab:selftest} lists the $\selftestN$ assertions the harness runs before
any measurement is taken, and the two paragraphs that follow record which of them
were added because an earlier version produced a wrong number, and one way in which a
self-test can itself conceal a refutation.

\begin{table}[h]
\centering
\caption{The $\selftestN$ assertions that must pass before any measurement is taken.
The harness refuses to proceed on failure. Items marked $\dagger$ were added after an
earlier version of the pipeline produced a confidently wrong number that the
assertion would have caught; we list them in that spirit rather than as
boilerplate.}
\label{tab:selftest}
\small
\begin{tabular}{@{}lll@{}}
\toprule
Estimator & Assertion & Expected \\
\midrule
AUC & perfect / inverted / constant / random & $1.0$ / $0.0$ / $0.500$ / ${\approx}0.5$ \\
AUC ties$^\dagger$ & constant scorer, ties averaged & exactly $0.500$ \\
balanced accuracy & always-majority on $95{:}5$ & $0.500$ \\
$\bits^{\max}$$^\dagger$ & matches the training code's printed ceilings & $10.108$, $\bmaxA$ \\
retrieval & perfect / random bits & $\bits^{\max}$ / ${\approx}0$ \\
linear probe & separable / random labels & $>0.90$ / ${\approx}0.5$ \\
group split$^\dagger$ & folds share no group & disjoint \\
genericness & random / tight cluster & $<0.2$ / $>0.9$ \\
duplicate census & $300$ planted duplicates of $1000$ & exactly $300$ \\
target masking$^\dagger$ & all target relations, both directions & empty \\
unrelated relations & untouched by masking & preserved \\
target frame & unfitted \texttt{apply()} raises & \texttt{RuntimeError} \\
InfoNCE & aligned loss $<$ shuffled loss & true \\
\texttt{to\_hetero} precondition$^\dagger$ & every node type is a destination & rejects otherwise \\
lazy parameters$^\dagger$ & none at construction & all materialised \\
EMA target$^\dagger$ & identical to encoder at step $0$ & bitwise equal \\
relation cue & $\Naspects$ distinct vectors at init & non-degenerate \\
forward pass, masked$^\dagger$ & survives empty target relations & finite \\
\textbf{leaf isolation}$^\dagger$ & masked target invariant to \emph{all} context features & bitwise \\
census identity & $\Nevid=\Nsupp+\Nchal$ & exact \\
\bottomrule
\end{tabular}
\end{table}

\noindent
The last assertion is the one that matters. Under target-edge masking an evidence
node is a leaf, so its representation must be a function of its own features alone;
we verify this by perturbing \emph{every} context feature by a constant and requiring
the masked output to be bitwise unchanged. Incomplete masking is the single failure
mode capable of manufacturing a false positive, and it is exactly what produced the
$\polFull$-AUC reading in Table~\ref{tab:polgate} before masking was applied.

\paragraph{A self-test can also conceal.}
An earlier version of this table asserted a paper-level count taken from a
\emph{different} label ($24{,}983$ rather than the measured $\papChal$). The
assertion passed, because the literal was internally consistent, and it thereby
concealed the refutation reported in Sec.~\ref{sec:repair}. A test that validates a
hard-coded constant rather than a measured quantity is worse than no test; that
assertion now takes the measured census as an argument. The same principle applies to
the ceilings: the $\bits^{\max}$ assertion compares against the value the training
code \emph{prints}, not against a number typed into the test.

\section{Pre-registered thresholds}
\label{app:prereg}

Table~\ref{tab:prereg} lists every threshold that converts a measurement into a
verdict, all of them written to disk before any result was computed.

\begin{table}[h]
\centering
\caption{Decision thresholds, fixed and written to disk before any result was
computed and not revised afterwards. Each converts a measurement into a verdict;
publishing them is what makes the verdicts auditable. The last four are the
target-side gates introduced for Sec.~\ref{sec:repair}.}
\label{tab:prereg}
\small
\begin{tabular}{@{}llr@{}}
\toprule
Name & Meaning & Value \\
\midrule
\texttt{collapse\_bits} & below this many bits, a system is ``collapsed'' & $1.0$ \\
\texttt{probe\_encodes\_bits} & above this, the context linearly encodes identity & $4.0$ \\
\texttt{control\_capable\_bits} & above this, a control demonstrates capability & $1.0$ \\
\texttt{paired\_alpha} & level for paired bootstrap intervals & $0.05$ \\
\texttt{degenerate\_bits\_eps} & below this gap, a control is degenerate with the model & $10^{-3}$ \\
\texttt{sim\_match\_bits} & tolerance for ``simulation reproduces observation'' & $1.0$ \\
\texttt{kappa\_material\_bits} & above this, anisotropy is called load-bearing & $0.5$ \\
\midrule
\texttt{a1\_void} & below this, the reasoning probe carries no signal & $\aOneVoid$ \\
\texttt{dup\_max\_frac} & max exact-duplicate share of a target relation & $0.20$ \\
\texttt{generic\_max\_gap} & max genericness gap between target classes & $0.15$ \\
\texttt{min\_target\_edges} & min edges for a data-derived target & $2{,}000$ \\
\bottomrule
\end{tabular}
\end{table}

\noindent
Three consequences. The collapse criterion is bits-based rather than rank-based,
which is why Table~\ref{tab:rho} reports $\rho=\rhoours$ as a partial collapse and
why the regression cell of Table~\ref{tab:repair}, at $\regBits$ bits, is labelled
collapsed without appeal to its rank statistics.
\texttt{degenerate\_bits\_eps} exists because an earlier version of the harness
reported a ``positive control'' agreeing with the model under test to four decimal
places---the same configuration compared to itself; the guard now returns
\emph{vacuous} rather than a verdict if nothing eligible survives. And
\texttt{dup\_max\_frac} / \texttt{generic\_max\_gap} were fixed before the polarity
census was computed, which is the only reason Table~\ref{tab:polgate} can be read as
a result rather than as a post-hoc excuse for a null.

\section{Proofs}
\label{app:proofs}

Proofs of the five results of Sec.~\ref{sec:theory}, in the order stated.

\begin{proof}[Proof of Prop.~\ref{prop:fixedpoint}]
\emph{(i)} For fixed $c$ and fixed $g$, $\arg\min_u\mathbb{E}[\|u-g(t)\|^2\mid
c]=\mathbb{E}[g(t)\mid c]$ by the first-order condition, and the minimising function
is this conditional mean pointwise.

\emph{(ii)} By hypothesis there is a measurable $\hat a$ with $\hat a(c)=a(t)$ almost
surely---in our pipeline $\hat a$ is a lookup on the aspect cue. Set
$f=\gamma\circ\hat a$ and $g=\gamma\circ a$. Then
$f(c)-g(t)=\gamma(a(t))-\gamma(a(t))=0$ almost surely, so $R=0$; since $R\ge0$
everywhere, this is a global minimum. The pair is realisable in the parametric class:
a message-passing encoder attains a per-type constant by driving the
feature-dependent path to zero, and the EMA constraint is satisfied because
$\bar\theta=\theta$ at the stationary point, so the two branches may share $\gamma$.
The same argument holds verbatim for the cosine loss, replacing $\|\cdot\|^2$ by
$1-\cos$.

\emph{(iii)} Suppose $c$ determines $t$, i.e.\ there is $\Psi$ with $t=\Psi(c)$
almost surely; on our corpus the training-free oracle shows this holds to within
$\orcFullPct$ of the recoverable budget. For any injective $g$, take $f=g\circ\Psi$;
then $R=0$ and $g$ retains $\delta$. Both branches are global minima and they differ
by the full recoverable budget, which is the claimed unidentifiability.

Finally, in \emph{(ii)} the image of $\gamma$ has at most $|\mathcal{A}|$ points, so
the mean-centred query bank has at most $|\mathcal{A}|-1$ non-zero singular values
and $\erank\le|\mathcal{A}|$ by Eq.~\ref{eq:erank}, while $f$ is non-constant
whenever $\gamma$ is injective on $\mathcal{A}$.
\end{proof}

\begin{proof}[Proof of Prop.~\ref{prop:centroid}]
For fixed $c$, $\arg\min_u\mathbb{E}[\|u-z\|^2\mid c]=\mathbb{E}[z\mid c]$ by the
first-order condition $2(u-\mathbb{E}[z\mid c])=0$; the minimising function is
therefore $f^\star(c)=\mathbb{E}[z\mid c]$ pointwise. Substituting
Eq.~\ref{eq:decomp} gives
$\mathbb{E}[z\mid c]=\mu+\alpha_{a(c)}+\mathbb{E}[\delta\mid c]$, and the achieved
loss is $\mathbb{E}\|z-\mathbb{E}[z\mid c]\|^2=
\mathbb{E}\|\delta-\mathbb{E}[\delta\mid c]\|^2\le\mathbb{E}\|\delta\|^2$, with
equality when $c$ is independent of $\delta$. The hypothesis that $c$ is weakly
informative about $\delta$ is what drives $\mathbb{E}[\delta\mid c]\to0$; as
Sec.~\ref{sec:theory} notes, that hypothesis is false on our corpus, which is why
this proposition is stated as the frozen-target special case and not as the account of
our observations.
\end{proof}

\begin{proof}[Proof of Prop.~\ref{prop:frame}]
Since $q_i\equiv q$, the score vector $\big(s(T(q),T(c_j))\big)_{j=1}^{N}$ does not
depend on $i$, so its argsort is a single fixed permutation $\pi$ of $[N]$ and the
gold item's rank is $r_i=\pi^{-1}(g_i)$. Because $\pi^{-1}$ is a bijection of $[N]$
and $g_i$ is uniform on $[N]$, the rank $r_i$ is also uniform on $[N]$; this is the
step that makes the conclusion exact rather than approximate. Therefore
\[
\mathbb{E}\big[\tfrac1{r}\big]=\tfrac1N\sum_{k=1}^{N}\tfrac1k=\tfrac{H_N}{N},
\qquad
\mathbb{E}[\log_2 r]=\tfrac1N\sum_{k=1}^{N}\log_2 k=\tfrac1N\log_2 N!,
\]
and substituting the second identity into Eq.~\ref{eq:bits} gives $\bits=0$ exactly.
None of these quantities depends on $T$, so no injective re-metrisation can alter
them. Injectivity of $T$ is used only to guarantee that $T$ does not merge
candidates and thereby change $N$.
\end{proof}

\noindent
The proposition is exact for $q_i\equiv q$, whereas we measure $\rquer=\rqueryv$ and
self-similarity $\selfsim$: nearly, not exactly, constant. A perturbation bound
expressing $\bits$ as a function of query-bank dispersion would close that gap, and we
flag its absence rather than treating the exact statement as if it covered the
measured regime. What the measurement does establish directly is the empirical
version, rung~5 of Table~\ref{tab:ladder}: five frames buy $\framegain$ against a
$\deficit$ deficit.

\begin{proof}[Proof of Prop.~\ref{prop:frechet}]
By definition the Fréchet mean of a distribution $P$ on $(\mathcal{M},d)$ is
$\arg\min_{u\in\mathcal{M}}\int d(u,z)^2\,dP(z)$. Conditioning on $c$ and minimising
pointwise in $u=f(c)$ gives exactly this variational problem for
$P=p(\cdot\mid c)$. Existence and uniqueness hold on Hadamard manifolds by convexity
of $u\mapsto d(u,z)^2$ along geodesics.
\end{proof}

\begin{proof}[Proof of Prop.~\ref{prop:reducible}]
Write the $\ell$-layer message-passing update for node $v$ as
$h_v^{(\ell)}=\phi^{(\ell)}\big(h_v^{(\ell-1)},\{\!\{h_u^{(\ell-1)}:u\in
\mathcal{N}(v)\}\!\}\big)$. The receptive field of $v$ after $\ell$ layers is
determined by $E$, and by hypothesis $E$ is a fixed function of the census
$\mathcal{N}$, identical across all graphs sharing that census. Hence for any two
graphs $G,G'$ with the same census, the map from the multiset of node features in
$v$'s $\ell$-hop neighbourhood to $h_v^{(\ell)}$ is the same function. Restricting to
patch $p$: because every relation incident to $p$ is present for all patches with
$p$'s node types, $p$'s $\ell$-hop neighbourhood contains exactly $p$'s own nodes
plus nodes reachable only through census-determined relations, whose membership is
likewise constant. Therefore $\Phi(G)_p=\Psi(\{x_u:u\in p\})$ for a fixed $\Psi$
depending only on the parameters and the census, and any score
$s(\Phi(G)_p,\Phi(G)_{p'})$ factors through $(\{x_u\}_{u\in p},\{x_u\}_{u\in p'})$.
The data-processing inequality then gives
$I(\text{gold};\,\text{score})\le I(\text{gold};\,\{x_u\}_{u\in p})$, which is the
information available to a training-free statistic of the same features.
\end{proof}

\noindent
Two remarks, both restated in the proposition itself because they are the two most
likely misreadings. The proposition bounds \emph{information}, not
\emph{performance}: a trained encoder may still exceed a fixed training-free
statistic by choosing a better metric over identical information, which is what our
$+\headroom$-bit margin over the oracle is and why a reader should not treat that
margin as a contradiction. And the hypothesis is checkable in one pass---compare each
relation's cardinality to its endpoint types' node counts
(Table~\ref{tab:reducible}).

\section{Two mechanisms we eliminated}
\label{app:eliminated}

Both statements below are mathematically correct and were for a substantial period
our leading hypotheses. Both are refuted as explanations by controls in the main
text. We record them because the elimination is part of the contribution.

\begin{lemma}[Scalarised targets are rank-$2$]
\label{lem:rank2}
If each patch is summarised by a scalar mean angle
$\alpha_i=\frac{1}{|S_i|}\sum_{j\in S_i}\phi(z_j)$ and mapped to
$t_i=(\cosh\alpha_i,\sinh\alpha_i)$, then $\rank(T)\le2$ and $\erank(T)\le2$ for any
encoder and any pooling operator.
\end{lemma}
\begin{proof}
Every row of $T$ lies on the image of the one-parameter curve
$\alpha\mapsto(\cosh\alpha,\sinh\alpha)\subset\mathbb{R}^2$, so the row space lies in
a two-dimensional subspace.
\end{proof}

\noindent\textbf{Why it is not the explanation.} A Euclidean-cosine pipeline with no
hyperbolic map anywhere fails identically (Table~\ref{tab:ladder}, rungs~0--3), and
the pathology is present at initialisation (Fig.~\ref{fig:dynamics}(b)). Raising the
algebraic ceiling from $2$ to $\dlat$ with orthonormal multi-anchor targets leaves the
\emph{measured} target rank between $1.6$ and $2.0$ (Table~\ref{tab:mah}).

\begin{proposition}[Mean pooling contracts effective rank]
\label{prop:pool}
Let patch members be $z^{(p)}_i=c+\delta_p+\eta^{(p)}_i$ with
$\delta_p,\eta^{(p)}_i$ independent and zero-mean, covariances
$\Sigma_\delta,\Sigma_\eta$. Mean pooling over $m$ members yields
$\Sigma_{\mathrm{pool}}(m)=\Sigma_\delta+\frac1m\Sigma_\eta$, so
$\erank(\Sigma_{\mathrm{pool}})$ is non-increasing in $m$ whenever
$\erank(\Sigma_\eta)>\erank(\Sigma_\delta)$.
\end{proposition}

\noindent\textbf{Why it is not the explanation.} For \texttt{method} patches
$m{=}1.00$ exactly, so mean pooling is the identity map; the measurement confirms it,
$\rnode=47.42\to\rpool=47.31$, and the query bank still collapses to $1.97$.

\begin{mykeybox}
Both eliminated hypotheses are \textbf{candidate-side} stories. Every instrument we
reached for---effective rank, anisotropy, DC ratio, spectra, whitening---is a
candidate-side instrument. We ran the same measurement on the query bank only after
every candidate-side account was exhausted. It took one line to settle the question.
\textbf{Instrument the queries.} The analogous lesson in Sec.~\ref{sec:repair} is one
level further out: every instrument in this paper, including the query-side ones, is
a \emph{representation}-side instrument, and the reducibility of the target is
invisible to all of them. \textbf{Instrument the target.}
\end{mykeybox}

\section{Per-rung ladder detail}
\label{app:ladderdetail}

Table~\ref{tab:frames} gives the post-hoc frame sweep of rung~5 with the bank
statistics that make it interpretable; the two paragraphs after it cover rungs~6
and~4.

\begin{table}[h]
\centering
\caption{\textbf{Post-hoc retrieval-frame sweep} (three seeds, block~A). All frames
are fitted on the candidate bank and applied to both sides. The best frame buys
$\framegain$ against a $\deficit$ deficit, as Prop.~\ref{prop:frame} predicts. This
is a different experiment from the \emph{trained} target-frame effect of
App.~\ref{app:repair} and the two must not be conflated: this sweep re-metrises an
already-collapsed query bank, whereas that one changes what the objective asks for.}
\label{tab:frames}
\small
\begin{tabular}{@{}lcc@{}}
\toprule
Frame & MRR & note \\
\midrule
raw & $\mrrRaw$ & --- \\
centre & $2.1\times10^{-4}$ & shared mean removed \\
all-but-the-top-$1$ & $\mrrRmTopOne$ & best \\
all-but-the-top-$2$ & $2.1\times10^{-4}$ & --- \\
ZCA (fitted on candidates) & $2.2\times10^{-4}$ & full whitening \\
\midrule
candidate bank & DC ratio $\dccand$ & $\rcand=\rcandv$ \\
query bank & DC ratio $\dcquery$ & $\rquer=\rqueryv$, DC energy $\dcenergyquery$ \\
query self-similarity & $\selfsim$ & mean pairwise cosine \\
\bottomrule
\end{tabular}
\end{table}

\paragraph{Encoder depth.}
Across depths $0$--$3$, node-level effective rank falls monotonically from
$\rnodeDzero$ to $\rnodeDthree$ while MRR moves from $1.9\times10^{-4}$ to
$2.0\times10^{-4}$. A depth-$0$ linear encoder with no message passing whatsoever
fails identically to a depth-$3$ GNN; see Fig.~\ref{fig:depth}.

\paragraph{Whitening.}
Whitening the inputs raises the instance-variance share of the trained
representation from $\papTr$ to $\withinPaperWhite$ and the input effective rank to
$\erInWhite$ of $\dfeat$ (Table~\ref{tab:whiten}), while MRR stays at chance and the
probe falls from $\probeBase$ to $\probeWhite$. Sec.~\ref{sec:theory} explains why
this null is predicted rather than anomalous: the zero-risk family of
Prop.~\ref{prop:fixedpoint}\emph{(ii)} depends on the inputs only through category
decodability, which whitening preserves.

\section{The repair sweep in full}
\label{app:repair}

Table~\ref{tab:repairgrid} gives the matched-budget factorial; the paragraphs after
it record a correction, the three unconfounded single-factor ablations, our treatment
of the two saturated cells, and the dose--response experiment we have not run.

\begin{table}[h]
\centering
\caption{\textbf{The $2\times2$ over target frame and shared basis at matched
budget} ($3$k steps, one seed, block~C, ceiling $\bmaxA$; \texttt{center} removes the
per-aspect target mean, \emph{shared} ties the input projection across node types).
Read the factorial, not the best cell. The baseline cell already reaches
$\rawZeroPct$ of ceiling, so only $\ceilRoom$ bits are available to any factor here;
this is the reducibility of Prop.~\ref{prop:reducible} showing up as a compressed
dynamic range.}
\label{tab:repairgrid}
\small
\begin{tabular}{@{}llccc@{}}
\toprule
Frame & shared & $\bits$ & \% of $\bits^{\max}$ & A1 probe \\
\midrule
\texttt{raw} & $0$ \emph{(baseline)} & $\rawZeroThreeK$ & $\rawZeroPct$ & $0.964$ \\
\texttt{center} & $0$ & $\nceBits$ & $99.9\%$ & $\nceAone$ \\
\texttt{raw} & $1$ & $11.808$ & $82.1\%$ & $0.985$ \\
\texttt{center} & $1$ & $14.379$ & $100.0\%$ & $0.963$ \\
\midrule
\multicolumn{2}{@{}l}{target-frame main effect} & $\frameEffect$ & \multicolumn{2}{l}{$\frameShare$ of the $\ceilRoom$ available} \\
\multicolumn{2}{@{}l}{shared-basis main effect} & $\sharedEffect$ & & \\
\multicolumn{2}{@{}l}{both} & $\bothEffect$ & & \\
\multicolumn{2}{@{}l}{interaction} & $\interEffect$ & & \\
\midrule
\multicolumn{2}{@{}l}{\textbf{learning-rate schedule}, same $3$k steps} & $\mathbf{\scheduleEffect}$
 & \multicolumn{2}{l}{$\rawOracleThreeK\!\to\!\rawZeroThreeK$} \\
\multicolumn{2}{@{}l}{\emph{training-free oracle}} & $\oracleA$ & $\oraclePct$ & --- \\
\bottomrule
\end{tabular}
\end{table}

\paragraph{A correction we owe the reader.}
An earlier draft reported the target-frame main effect as $+1.199$ bits. That figure
was computed against a baseline of $\rawOracleThreeK$ bits obtained under an older
learning-rate schedule; at matched schedule the baseline is $\rawZeroThreeK$ and the
frame effect is $\frameEffect$. The $+1.199$ therefore conflated the frame with the
schedule, and the schedule is the larger term. We correct it here, and we note that
the error was invisible until we read a log stage we had previously skipped---which
is an argument for the claim--evidence map of App.~\ref{app:claimmap} rather than
against it.

\paragraph{The unconfounded ablations.}
Three further cells share budget, schedule, frame, loss and seed, so each isolates
one factor. \emph{Structural cue} ($6$k steps): aspect embedding $\cueAspect$ bits,
RWSE $\cueRwse$, none $\cueNone$---so a $\Naspects$-entry learned embedding beats a
$16$-dimensional RWSE by $\cueGainA$ bits and removing the cue entirely costs
$\cueGainB$. This is consistent with App.~\ref{app:rwse}: the RWSE has effective rank
$\rwseRankLo$--$\rwseRankHi$ of $16$ because the graph is census-determined.
\emph{Transduction} ($6$k steps): fitting the target frame on a disjoint half rather
than the candidate bank changes bits by $\transGap$, retiring the concern that the
frame is a transduction artefact. \emph{Raw-skip} ($6$k steps): a learned gate
between raw features and encoder output moves from $\sigma=0.5$ at initialisation to
$\sigma=\rawGate$, i.e.\ toward the encoder, so the encoder is doing work that raw
features alone do not.

\begin{cautionbox}
\textbf{Two cells report $\mathrm{MRR}=1.0000$ and we do not headline them.} The
raw-skip cell and \texttt{center}:1 both reach $\bmaxA$ bits, $100.0\%$ of ceiling.
Our own harness flags $\mathrm{MRR}=1.0000\pm0.0000$ as a degeneracy signature, and
in the polarity probe that exact reading was the $\polEdgeLeak$-AUC edge-type leak
(Table~\ref{tab:polgate}). Given Prop.~\ref{prop:reducible} and an oracle already at
$\oraclePct$, near-ceiling is \emph{expected} rather than suspicious---but expected
is not the same as verified. We therefore report the $20$k, $3$-seed cell
($\fixTwentyK\pm\fixTwentyKsd$, $\mathrm{MRR}=0.9996$) as the paper's repaired
number, and treat the two saturated cells as requiring the candidate-perturbation
leak check before any stronger claim.
\end{cautionbox}

\paragraph{The dose--response experiment that would settle the mechanism.}
The harness supports testing the allocation by \emph{intervention} rather than
simulation: the per-aspect mean of the real features is rescaled by a factor
$\alpha$, sweeping the achieved aspect share $\rho$; the model is retrained at each
dose; and bits recovered are measured against the \emph{measured} $\rho$. This yields
an empirical critical share to compare against the simulated $\rho^\star=\rhostar$ in
the actual feature space. The design, dose grid and stopping rule are in the
pre-registration file. It is not reported here, and together with the
non-census-determined control graph of Sec.~\ref{sec:discussion} it is the
measurement we would run first with more compute.

\section{Protocol for the checkpoint allocation measurement}
\label{app:trajectory}

Table~\ref{tab:trajectory} is produced as follows. Checkpoints are the ones the
training loop already writes; no rerun is required. At each checkpoint we load the
encoder, run one forward pass over the block-A graph cache with the same masking
convention as Protocol~R, and compute three quantities on the resulting patch
representations: the one-way variance share between aspects
($\rho_{\mathrm{asp}}$), the one-way share between papers ($\rho_{\mathrm{pap}}$),
and the effective rank of the model's predictions ($\rquer$) by
Eq.~\ref{eq:erank}. Bits recovered are then measured under the unmodified Protocol~R
on the same $\Nqueries$ fixed queries, so the column is directly comparable to
Table~\ref{tab:ladder}. Both variance shares are computed in the block-A feature
space, so the trajectory is internally consistent and does not inherit the
cross-encoder caveat of App.~\ref{app:provenance}.

\paragraph{Why the measurement is necessary rather than decorative.}
An endpoint comparison cannot distinguish a degeneracy the objective \emph{created}
from one it merely \emph{failed to remove}, and those two readings support different
claims---the first about training dynamics, the second about the optimum. Our own
rank trajectory (Fig.~\ref{fig:dynamics}(b)) and untrained rank measurements
(App.~\ref{app:rankinit}) already point toward the second. Because the distinction
costs one forward pass per checkpoint and changes what the paper is entitled to say,
we regard it as the minimum standard for any claim of the form ``training
re-allocates variance,'' and Recommendation~(6) of Sec.~\ref{sec:discussion} states it
as such.

\section{The held-out reasoning probe A1}
\label{app:aone}

A1 is the second metric of Sec.~\ref{sec:repair}, and because a decoupling claim is
only as strong as the metric it decouples from, we specify its construction and audit
its floor rather than treating it as a fixed reference point.

\begin{table}[h]
\centering
\caption{The A1 probe: construction, references and floors. The void line
$\aOneVoid$ is pre-registered (App.~\ref{app:prereg}): a reading below it is treated
as carrying no signal, and no cell in Table~\ref{tab:repair} falls below it.}
\label{tab:aone}
\small
\begin{tabular}{@{}lll@{}}
\toprule
Property & Value & Note \\
\midrule
pre-registered void line & $\aOneVoid$ & App.~\ref{app:prereg} \\
observed range across cells & $\aoneRange$ & Table~\ref{tab:repair} \\
\bottomrule
\end{tabular}
\end{table}

\paragraph{The construct-validity caveat, stated plainly.}
If A1 draws on the \texttt{challenged\_by} relation, then by Table~\ref{tab:polgate}
its target is $\dupChal$ exact-duplicate placeholder text with a cosine-only floor of
$\polCosAUC$, and the decoupling result of Sec.~\ref{sec:repair} would then be
uninterpretable in \emph{both} directions: neither the presence nor the absence of a
relationship with bits recovered could be attributed to reasoning content. If A1 does
not draw on that relation, this paragraph does not apply and the table above says so
explicitly. We separate the two cases rather than leaving the reader to infer which
holds, because the same corpus audit that produced Table~\ref{tab:polgate} is the
reason the question arises.

\section{The polarity target: census, gates and probe conditions}
\label{app:polarity}

Table~\ref{tab:polcensus} gives the census behind the two failed gates of
Table~\ref{tab:polgate}; the paragraphs after it establish that the duplication is
asymmetric and identify the one data-derived relation that survives.

\begin{table}[h]
\centering
\caption{\textbf{Census of the data-derived relations} (block~C). Every evidence node
is a leaf---$\Nevid=\Nsupp+\Nchal$ exactly---because the source fields are plain
strings, so each claim yields one supporting and at most one contradicting evidence
node. This is the structural fact behind the $\polEdgeLeak$-AUC edge-type leak.}
\label{tab:polcensus}
\small
\begin{tabular}{@{}lrl@{}}
\toprule
Quantity & Value & Note \\
\midrule
\texttt{claim} nodes & $\Nimpl$ & \\
\texttt{supported\_by} edges & $\Nsupp$ & ${\approx}1$ per claim \\
\texttt{challenged\_by} edges & $\Nchal$ & $48.2\%$ of claims \\
\texttt{evidence} nodes & $\Nevid$ & $=\Nsupp+\Nchal$ exactly \\
\texttt{implies} edges & $\Nimpl$ & $1$ per claim (census-determined) \\
\midrule
papers with $\ge1$ challenge & $\papChal$ & $\papChalObs$ of corpus \\
\quad expected under per-claim assignment & --- & $\papChalExp$ \\
challenges per affected paper & $3.41$ & vs.\ $4.35$ claims/paper \\
\midrule
duplicate \texttt{evidence} rows (groups $\ge50$) & $\dupRows$ & $8.45\%$ of all evidence \\
challenge edges onto duplicated rows & $\dupChalEdges$ & $\dupChal$ of challenges \\
largest single duplicate group & $\dupTop$ & identical strings \\
unique \texttt{implication} rows & $\uniqImpl$ & of $\Nimpl$; $\dupImpl$ duplicated \\
\bottomrule
\end{tabular}
\end{table}

\paragraph{The duplication is almost perfectly asymmetric.}
$\dupRows$ evidence nodes lie in duplicate groups of size $\ge50$, and
$\dupChalEdges$ challenge edges point at duplicated rows. These agree to four nodes,
so essentially every boilerplate evidence node is a \emph{challenge} node and
supporting evidence is nearly free of exact duplication. That asymmetry is a complete
account of the apparent polarity signal without invoking any semantic difference:
$\cos(\text{claim},\text{support})=\cosSup$ against
$\cos(\text{claim},\text{challenge})=\cosChal$ with $d=\polD$, and a cosine-only
classifier at $\polCosAUC$ AUC. The genericness gap of $\genGap$ shows the
non-duplicated remainder is also formulaic, so a duplicate filter alone is
insufficient; the fix is a length-and-pattern guard at the point where the extracted
string is accepted, followed by one rebuild and a re-measurement of $\bmaxA$,
$\oracleA$ and every derived reference point.

\paragraph{What survives.}
\texttt{implies} is clean: $\uniqImpl$ of $\Nimpl$ rows unique, $\dupImpl$ in
duplicate groups. It is also census-determined ($1$ per claim,
Table~\ref{tab:reducible}), so it cannot serve as a \emph{structural} target---but as
a \emph{content} target its text is usable, and it is the one data-derived signal in
this corpus that passes the quality gate. Restricting the target set to
\texttt{implies} is a three-line change and is the cheapest way to obtain a
data-derived target that clears our own gates, rather than leaving
Sec.~\ref{sec:repair} to end on a null.

\section{Per-aspect breakdown}
\label{app:peraspect}

Table~\ref{tab:peraspect} and Fig.~\ref{fig:peraspect} give the per-aspect view that
makes the singleton control decisive, and Table~\ref{tab:preanchor} records the
aggregate that concealed it.

\begin{table}[h]
\centering
\caption{\textbf{Per-aspect breakdown} (block~A, five seeds). \texttt{method} and
\texttt{result} are exact singletons, so mean pooling is the identity map:
$\rnode\approx\rpool$ to two decimals. $z$ is the gold item's standardised
similarity margin; $z\approx0$ means the gold candidate is indistinguishable from the
pool mean.}
\label{tab:peraspect}
\small
\begin{tabular}{@{}lccccccc@{}}
\toprule
Aspect & $m$ & $\rnode$ & $\rpool$ & $\rquer$ & MRR & MRR (centred) & margin $z$ \\
\midrule
claim & $4.35$ & $5.40$ & $4.91$ & $1.87$ & $1.5\times10^{-4}$ & $2.2\times10^{-4}$ & $-0.006$ \\
method & $1.00$ & $47.42$ & $47.31$ & $1.97$ & $2.1\times10^{-4}$ & $1.5\times10^{-4}$ & $-0.001$ \\
result & $1.00$ & $2.04$ & $2.05$ & $1.62$ & $2.0\times10^{-4}$ & $2.0\times10^{-4}$ & $+0.033$ \\
\midrule
\multicolumn{5}{@{}l}{\emph{Chance}} & $\chanceMRR$ & $\chanceMRR$ & $0$ \\
\bottomrule
\end{tabular}
\end{table}

\begin{figure}[h]
  \centering
  \includegraphics[width=0.85\textwidth]{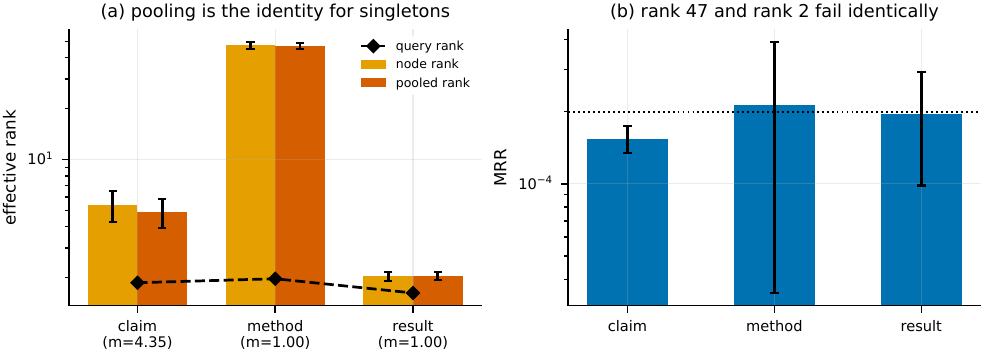}
  \caption{\textbf{The control that exonerates pooling and localises the failure to
  the query side.} \textbf{(a)}~For \texttt{method}, $m{=}1.00$, so mean pooling is
  provably the identity map: node rank $47.4$ carries through to pooled rank $47.3$.
  The model's \emph{predictions} for those patches nevertheless occupy effective rank
  $1.97$ (black diamonds), at or below the $\Naspects$-category bound of
  Prop.~\ref{prop:fixedpoint}. \textbf{(b)}~A $23\times$ range in pooled rank
  produces no movement in bits recovered; dotted line is chance, dashed line the
  $\bmaxA$-bit ceiling. Five seeds.}
  \label{fig:peraspect}
\end{figure}

\begin{table}[h]
\centering
\caption{\textbf{Isolated pre-target rank measurement.} A mean-pool JEPA with a
Euclidean-cosine target, no hyperbolic map and no anchors. Mean $\pm$ std over five
seeds. The per-aspect breakdown above shows this aggregate conceals a $23\times$
spread---which is why the aggregate, alone, misled us.}
\label{tab:preanchor}
\small
\begin{tabular}{@{}lcl@{}}
\toprule
Quantity & Value ($/\dlat$) & \\
\midrule
Node-latent effective rank $\rnode$ & $18.3\pm0.9$ & healthy \\
Pooled-latent effective rank $\rpool$ & $1.97\pm0.04$ & low, but see Table~\ref{tab:peraspect} \\
\bottomrule
\end{tabular}
\end{table}

\section{Oracle leakage controls}
\label{app:oracle}

Table~\ref{tab:oracle} reports the training-free oracle per aspect together with
three leakage controls: \textsc{no-summary} removes the title/abstract node from the
context, \NoOv{} deletes every $5$-gram shared between query and gold, and a
\textsc{cheat} query equal to the target's own embedding validates the metric at
$\mathrm{MRR}=1.000$. Fig.~\ref{fig:oracle} plots the same numbers.

\begin{table}[h]
\centering
\caption{\textbf{Oracle retrieval and leakage controls}, per aspect, block~A. \NoOv{}
is a lower bound on the non-lexical signal; \BMtf{} uses no learned representation.
\textsc{no-summary} removes the paper title/abstract node from the context. A
\textsc{cheat} query equal to the target's own embedding validates the metric.}
\label{tab:oracle}
\small
\begin{tabular}{@{}llcccc@{}}
\toprule
\textbf{Features} & \textbf{Aspect} & \textbf{ctx MRR} & \textbf{\textsc{no-summary}} & \textbf{\BMtf} & \textbf{\NoOv} \\
\midrule
\multirow{3}{*}{whitened}
 & claim & $\mathbf{0.973}$ & $0.953$ & $0.979$ & $0.732$ \\
 & method & $\mathbf{0.908}$ & $0.834$ & $0.974$ & $0.802$ \\
 & result & $\mathbf{0.938}$ & $0.921$ & $0.989$ & $0.854$ \\
\midrule
\multirow{3}{*}{raw}
 & claim & $0.943$ & $0.916$ & --- & --- \\
 & method & $0.878$ & $0.796$ & --- & --- \\
 & result & $0.925$ & $0.916$ & --- & --- \\
\midrule
\multicolumn{2}{@{}l}{\textsc{cheat} (metric sanity)} & $1.000$ & --- & --- & --- \\
\multicolumn{2}{@{}l}{\emph{Trained pipeline, baseline}} & $\mrrBase$ & --- & --- & --- \\
\multicolumn{2}{@{}l}{\emph{Chance}} & $\chanceMRR$ & --- & --- & --- \\
\bottomrule
\end{tabular}
\end{table}

\begin{figure}[h]
  \centering
  \includegraphics[width=0.6\textwidth]{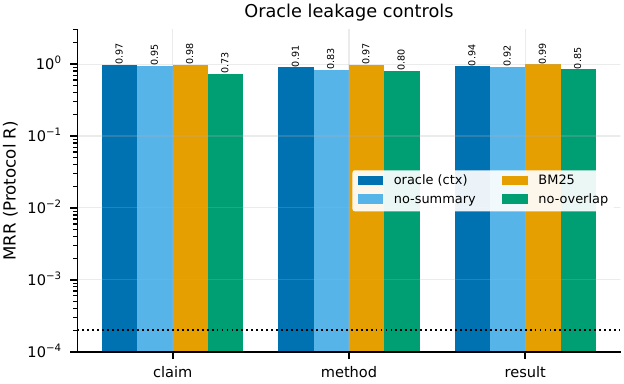}
  \caption{\textbf{Four training-free or positive-control retrievers, all near
  ceiling; the baseline trained model at chance (dotted).} \BMtf{} is the strongest
  evidence because it uses no learned representation whatsoever, and \NoOv{} still
  recovers $\novFullB$ bits---$\novFullPct$ of the $\bmaxB$-bit ceiling---after
  deleting every shared $5$-gram.}
  \label{fig:oracle}
\end{figure}

\section{Query-subsample validation}
\label{app:subsample}

Vector systems are evaluated on $\Nqueries$ queries and lexical systems on a nested
$\Nlex$ subsample, because BM25 costs roughly $500\times$ more per query than a dot
product. The subsample is validated rather than assumed: the positive control on
$\Nqueries$ queries recovers $\poneFullB$ bits against $+14.211$ from the full
evaluation, agreement of $0.01$ bits, and the nested lexical subsample agrees with
the $\Nqueries$-query set within the bootstrap interval at every rung of
Table~\ref{tab:ladder2}. All reported intervals are over queries, so subsampling is
reflected in the stated uncertainty.

\section{The $2\times2$ factorial separating feature space from trainer}
\label{app:factorial}

\Pone{} differs from the model under test in two factors at once, so it bounds the
harness without attributing the gap. We therefore pre-registered a $2\times2$ design
over $\{$reference predictor, full pipeline$\}\times\{$frozen text features, hetero
node features$\}$ in which exactly one factor differs between comparable cells, and
in which a cell that cannot obtain the trainer or features it requested is recorded
as \texttt{provenance\_mismatch} and emits \emph{no verdict}. Silence is the correct
output when a premise is unmet; the alternative---silently substituting a
fallback---is what produced the degenerate control described in
App.~\ref{app:prereg}.

\begin{table}[h]
\centering
\caption{Factorial design. The two diagonal cells are the systems reported in the
main text; the off-diagonal cells isolate each factor and are not yet reported. We
list them explicitly rather than omitting them, so that the limits of the attribution
are visible: on the evidence in this paper, \Pone{} bounds the harness and does not
separate feature space from trainer. Both missing cells are $3$k-step runs on the
existing cache.}
\label{tab:factorial}
\small
\begin{tabular}{@{}lcc@{}}
\toprule
 & frozen text features & hetero node features \\
\midrule
reference predictor & $\poneFullB$ bits (\Pone) & \emph{not yet reported} \\
full pipeline & \emph{not yet reported} & $\bitsBase$ bits (Table~\ref{tab:ladder}) \\
\bottomrule
\end{tabular}
\end{table}

\section{Full phase sweep}
\label{app:phase}

Table~\ref{tab:rhofull} gives the full simulation sweep summarised in
Table~\ref{tab:rho}, including the leave-one-out block that retires anisotropy as the
residual factor.

\begin{table}[h]
\centering
\caption{\textbf{Phase analysis, full sweep} (bank size $\simN$, so
$\bits^{\max}=\bmaxsim$ by Eq.~\ref{eq:bmax}). Row~1 varies $\rho$ with
$\kappa=\varepsilon=0$; row~2 uses the measured
$\hat\kappa=\kappahat,\hat\varepsilon=\epshat$. The leave-one-out block shows that
neither nuisance parameter is load-bearing.}
\label{tab:rhofull}
\small
\begin{tabular}{@{}lccccccc@{}}
\toprule
$\rho$ & $\aspInRaw$ & $0.5$ & $0.9$ & $0.99$ & $\mathbf{0.9961}$ & $0.999$ & $0.9999$ \\
\midrule
MRR, $\rho$ only & $1.000$ & $1.000$ & $0.8928$ & $0.1068$ & $0.0204$ & $0.0043$ & $0.0014$ \\
$\bits$, $\rho$ only & $+12.845$ & $+12.845$ & $+12.200$ & $+3.833$ & $+1.911$ & $+0.843$ & $+0.266$ \\
MRR, matched $\hat\kappa,\hat\varepsilon$ & $1.000$ & --- & --- & --- & $\simAtOurs$ & --- & --- \\
$\bits$, matched & $+12.845$ & --- & --- & --- & $+\simAtOursB$ & --- & --- \\
bits lost, matched & $0.000$ & --- & --- & --- & $\simLoss$ & --- & --- \\
\midrule
\multicolumn{8}{@{}l}{\emph{Leave-one-out at the measured operating point}} \\
\multicolumn{8}{@{}l}{\quad drop $\rho$ (set $0.5$): MRR $0.9740$;\ \ drop $\hat\kappa$: $1.0000$;\ \
 drop $\hat\varepsilon$: $1.0000$;\ \ drop all: $1.0000$} \\
\multicolumn{8}{@{}l}{\quad critical share $\rho^\star=\rhostar$, i.e.\
 $1-\rho^\star=\rhostarcomp$} \\
\bottomrule
\end{tabular}
\end{table}

\paragraph{Interpretation.}
The $\rho$-only row reproduces the qualitative shape of a collapse driven purely by
the variance share. Adding the measured nuisance parameters at $\rho=\rhoours$ moves
MRR from $0.0204$ to $\simAtOurs$---far below the pre-registered materiality
threshold---so we reject anisotropy as the residual factor. Because the criterion for
``total collapse'' is bits-based (App.~\ref{app:prereg}), the critical share is
located where recovered bits fall below $1.0$, giving $\rho^\star=\rhostar$. Note that
``bits lost'' is $\bmaxsim-\bits$ throughout and is therefore monotone in $\rho$; a
non-monotone loss column is a sign that two different ceilings have been mixed, which
is the error corrected in App.~\ref{app:bits}.

\begin{figure}[h]
  \centering
  \includegraphics[width=0.6\textwidth]{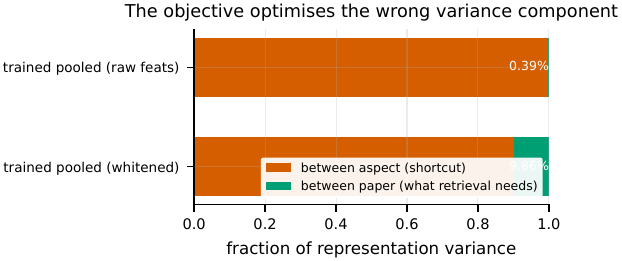}
  \caption{\textbf{The objective optimises the wrong variance component.} Left: the
  frozen inputs, $\papIn$ of variance between papers. Right: the trained latents,
  $\papTr$. Whitening the inputs raises the trained instance share to
  $\withinPaperWhite$ and bits recovered still do not move---the invariance
  Prop.~\ref{prop:fixedpoint} predicts; changing the \emph{objective} does
  (Table~\ref{tab:repair}).}
  \label{fig:variance}
\end{figure}

\section{The seven-cell intervention grid}
\label{app:fixgrid}

Table~\ref{tab:fix} gives the full aggregator $\times$ loss grid referenced in
Sec.~\ref{sec:audits}: seven cells, five seeds each, $\rpool$ spanning
$1.25$--$2.03$ and probe accuracy spanning $0.604$--$0.964$, with bits recovered at
$0$ throughout. Fig.~\ref{fig:fixbars} plots it.

\begin{table}[h]
\centering
\caption{\textbf{Seven interventions, one flat line} (block~A, five seeds,
Protocol~R). $\rpool$ spans $1.25$--$2.03$ and the probe spans $0.604$--$0.964$;
bits recovered stay at $0$ in every cell, against a ceiling of $\bmaxA$. What this
grid does \emph{not} vary is the target frame and the learning-rate schedule, both of
which App.~\ref{app:repair} shows to matter; we therefore present it as a negative
result about the aggregator and the in-batch-negative form of the loss, not about the
pipeline.}
\label{tab:fix}
\small
\begin{tabular}{@{}llcccc@{}}
\toprule
Pooling & Loss & $\rpool$ & probe & MRR & $\bits$ \\
\midrule
mean & $1-\cos$ \emph{(baseline)} & $1.97$ & $0.871$ & $1.9\times10^{-4}$ & $0.00$ \\
mean & \textsc{infonce} & $1.96$ & $0.867$ & $2.1\times10^{-4}$ & $0.00$ \\
sum & $1-\cos$ & $1.25$ & $0.963$ & $1.8\times10^{-4}$ & $0.00$ \\
deepsets & $1-\cos$ & $1.70$ & $0.604$ & $1.9\times10^{-4}$ & $0.00$ \\
deepsets & \textsc{infonce} & $1.81$ & $0.742$ & $1.7\times10^{-4}$ & $0.00$ \\
attn & \textsc{infonce} & $2.03$ & $0.964$ & $\mathbf{2.3\times10^{-4}}$ & $0.00$ \\
mean & $1-\cos$ $+$ tgt-\textsc{vic} & $1.97$ & $0.871$ & $1.9\times10^{-4}$ & $0.00$ \\
\midrule
\multicolumn{4}{@{}l}{\Oracle{} ceiling (block~B)} & $\orcFull$ & $\orcFullB$ \\
\multicolumn{4}{@{}l}{\emph{repaired objective} (Table~\ref{tab:repair})} & $0.9996$ & $+\fixTwentyK$ \\
\multicolumn{4}{@{}l}{$\bits^{\max}$} & --- & $\bmaxA$ \\
\multicolumn{4}{@{}l}{Chance} & $\chanceMRR$ & $0.00$ \\
\bottomrule
\end{tabular}
\end{table}

\begin{figure}[h]
  \centering
  \includegraphics[width=0.5\textwidth]{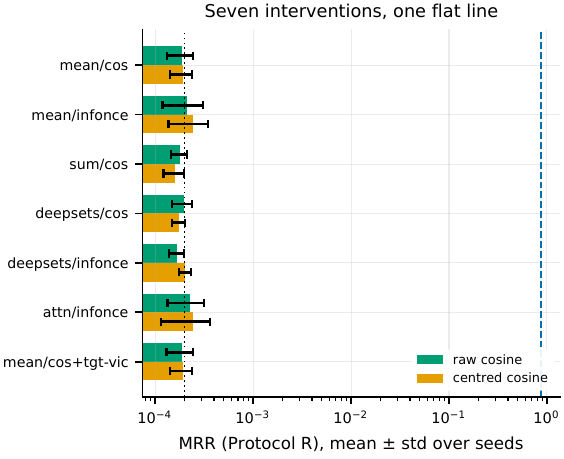}
  \caption{\textbf{Seven interventions, one flat line.} Raw and centred cosine, mean
  $\pm$ std over five seeds. Dotted line is chance ($\bits=0$), dashed line the
  $\bmaxA$-bit ceiling.}
  \label{fig:fixbars}
\end{figure}

\section{Breadth of the dissociation}
\label{app:breadth}

Table~\ref{tab:breadth} extends the dissociation across twenty
objective~$\times$~target-construction cells, including two that a rank-based
selection criterion would prefer.

\begin{mykeybox}
\textbf{Protocol note.} Every retrieval number in the main text is measured against
the full pool. Retrieval columns are omitted from Tables~\ref{tab:breadth}
and~\ref{tab:mah} because those runs predate Protocol~R and were scored in-batch;
in-batch MRR over a few hundred candidates overstates full-pool MRR by roughly two
orders of magnitude and is not rescalable. Probe accuracy and effective rank are
protocol-independent properties of the representation. This omission is itself one of
the paper's recommendations.
\end{mykeybox}

\begin{table}[h]
\centering
\caption{\textbf{Twenty objective $\times$ target-construction cells; the
dissociation holds across all of them.} Probe accuracy spans $0.755$--$0.960$ and
pooled effective rank $1.8$--$13.9$---a $7.7\times$ range including two cells that
clear the conventional $\rpool>10$ threshold---with no accompanying change in
retrieval. Mean over five seeds.}
\label{tab:breadth}
\small
\begin{tabular}{@{}llccl@{}}
\toprule
\textbf{Objective} & \textbf{Target construction} & \textbf{Probe acc} & $\rpool/\dlat$ & \textbf{Rank verdict} \\
\midrule
\multirow{3}{*}{hyperbolic, 2-D}
 & mean (own only) & $0.840$ & $1.9$ & collapsed \\
 & rel-hetero $+$ white & $0.822$ & $3.1$ & collapsed \\
 & rel-hetero (raw) & $0.940$ & $1.8$ & collapsed \\
\midrule
\multirow{3}{*}{hyperbolic $+$ \textsc{vic}}
 & mean (own only) & $0.859$ & $2.0$ & collapsed \\
 & rel-hetero $+$ white & $0.812$ & $\mathbf{13.9}$ & \textbf{rank-recovered} \\
 & rel-hetero (raw) & $0.954$ & $2.2$ & collapsed \\
\midrule
\multirow{3}{*}{hyperbolic $+$ \textsc{vic}$^{+}$}
 & mean (own only) & $0.858$ & $2.0$ & collapsed \\
 & rel-hetero $+$ white & $0.821$ & $\mathbf{12.1}$ & \textbf{rank-recovered} \\
 & rel-hetero (raw) & $0.960$ & $2.2$ & collapsed \\
\midrule
\multirow{3}{*}{hyperbolic, full reg.}
 & mean (own only) & $0.871$ & $1.9$ & collapsed \\
 & rel-hetero $+$ white & $0.808$ & $3.9$ & collapsed \\
 & rel-hetero (raw) & $0.947$ & $1.8$ & collapsed \\
\midrule
\multirow{3}{*}{Lorentz}
 & mean (own only) & $0.872$ & $2.0$ & collapsed \\
 & rel-hetero $+$ white & $0.808$ & $4.4$ & collapsed \\
 & rel-hetero (raw) & $0.948$ & $1.8$ & collapsed \\
\midrule
\multirow{5}{*}{Euclidean cosine}
 & mean, raw & $0.871$ & $2.0$ & collapsed \\
 & mean, whitened & $0.755$ & $3.8$ & collapsed \\
 & rel-hetero, raw & $0.949$ & $1.8$ & collapsed \\
 & rel-hetero, whitened & $0.803$ & $4.2$ & collapsed \\
 & rel-hetero, wh., claim-only & $0.803$ & $4.2$ & collapsed \\
\midrule
\multicolumn{2}{@{}l}{\emph{Range across all twenty cells}} & $0.755$--$0.960$ & $1.8$--$13.9$ & \\
\bottomrule
\end{tabular}
\end{table}

\paragraph{What this table adds.}
The dissociation is not specific to the Euclidean-cosine objective: it holds for a
unit-hyperbola target, a Lorentz-model target, and two strengths of VICReg. It
survives a target construction that deliberately injects paper-specific structural
variance from cross-type reasoning and citation neighbours. And two cells clear the
conventional $\rpool>10$ threshold and would be \emph{preferred} by a rank-based
selection criterion such as RankMe \citep{DBLP:conf/icml/GarridoBNL23}; both are
among the worst cells by probe accuracy, and neither retrieves. Rank and probe
accuracy are not merely insufficient here---they disagree with each other. Note also
that none of these twenty cells varies the loss form or the learning-rate schedule,
which are the two axes App.~\ref{app:repair} identifies as load-bearing: a $20$-cell
sweep can miss the operative variable entirely.

\section{The anchor-dimension study}
\label{app:anchor}

Table~\ref{tab:mah} raises the target's algebraic rank ceiling from $2$ to $\dlat$
and shows the measured target rank does not follow.

\begin{table}[h]
\centering
\caption{\textbf{Raising the target's algebraic rank ceiling from $2$ to $2k$ does
not raise the measured target rank.} Each patch angle is replaced by $k$ angles
against QR-orthonormalised anchors. The ceiling rises to $\dlat$ at $k{=}64$ and the
measured target effective rank $\rtgt$ does not move.}
\label{tab:mah}
\small
\begin{tabular}{@{}llccc@{}}
\toprule
\textbf{Objective} & \textbf{Target} & \textbf{Probe acc} & $\rpool/\dlat$ & $\rtgt$ \\
\midrule
\multirow{4}{*}{hyperbolic, 2-D}
 & mean (2-D) & $0.840$ & $1.9$ & $2.0$ (ceiling) \\
 & $k{=}4$\ ($\to\mathbb{R}^{8}$) & $0.869$ & $1.9$ & $1.6$ \\
 & $k{=}16$ ($\to\mathbb{R}^{32}$) & $0.869$ & $1.9$ & $1.9$ \\
 & $k{=}64$ ($\to\mathbb{R}^{128}$) & $0.867$ & $1.9$ & $1.9$ \\
\midrule
\multirow{4}{*}{hyperbolic $+$ \textsc{vic}}
 & mean (2-D) & $0.859$ & $2.0$ & $2.0$ (ceiling) \\
 & $k{=}4$ & $0.859$ & $2.0$ & $1.7$ \\
 & $k{=}16$ & $0.857$ & $2.0$ & $1.9$ \\
 & $k{=}64$ & $0.860$ & $2.0$ & $2.0$ \\
\midrule
\multirow{4}{*}{hyperbolic, full reg.}
 & mean (2-D) & $0.871$ & $1.9$ & $2.0$ (ceiling) \\
 & $k{=}4$ & $0.869$ & $1.9$ & $1.6$ \\
 & $k{=}16$ & $0.869$ & $1.9$ & $1.9$ \\
 & $k{=}64$ & $0.867$ & $1.9$ & $1.9$ \\
\midrule
\multirow{4}{*}{Lorentz}
 & mean (2-D) & $0.872$ & $2.0$ & $2.0$ (ceiling) \\
 & $k{=}4$ & $0.869$ & $1.9$ & $1.6$ \\
 & $k{=}16$ & $0.869$ & $1.9$ & $1.9$ \\
 & $k{=}64$ & $0.867$ & $1.9$ & $1.9$ \\
\midrule
\multicolumn{2}{@{}l}{\emph{Algebraic ceiling at $k{=}64$}} & --- & --- & $\dlat$ \\
\multicolumn{2}{@{}l}{\emph{Measured range}} & --- & --- & $1.6$--$2.0$ \\
\bottomrule
\end{tabular}
\end{table}

\paragraph{Why the anchor result matters for the query-side account.}
The anchor map is injective by construction and its output dimension is swept over a
$16\times$ range, yet $\rtgt$ stays pinned between $1.6$ and $2.0$. A map cannot
manufacture variation its input does not contain: the patch latents reaching the
anchors already lie on a near-one-dimensional set. In hindsight, the flat $\rtgt$
column was the query-side signal showing through a candidate-side measurement, and
Prop.~\ref{prop:fixedpoint} says why no injective re-parameterisation of the target
could have changed it.

\section{Synthetic separability control}
\label{app:synthcontrol}

To test whether high candidate-bank redundancy alone impairs rankability, we generate
banks with controlled shared-mean dominance and query them with the \emph{true}
context vector, so only bank separability is measured. Retrieval stays at
$\mathrm{MRR}=1.000$ across the entire range, up to a DC ratio of $100$, far beyond
anything our corpus exhibits (candidate DC ratio $\dccand$). A bank with mean
pairwise cosine near $1$ remains perfectly rankable provided the residual is
consistent between context and target. This retires ``the data are too similar''
quantitatively rather than rhetorically, and distinguishes \emph{similarity} from
\emph{indistinguishability}.

\section{Extraction and templating audit}
\label{app:extraction}

Table~\ref{tab:extraction} collects the templating audit of Sec.~\ref{sec:audits} and
the placeholder-text audit of Sec.~\ref{sec:repair} in one place, because the two
measure different pathologies and only one of them is benign.

\begin{table}[h]
\centering
\caption{Templating audit on the $\Ndiag$ set; text probes on an $8{,}000$-document
subsample, variance probes on all records. The function-word classifier uses a fixed
stopword vocabulary and cannot access content. The final block is the
placeholder-text audit of Sec.~\ref{sec:repair}, on the block-C relations.}
\label{tab:extraction}
\small
\begin{tabular}{@{}llr@{}}
\toprule
Probe & Detail & Value \\
\midrule
aspect classification, function words only & fixed stopword vocabulary & $\fwAcc$ \\
aspect classification, full TF-IDF & $50$k features & $\tfidfAcc$ \\
chance & three aspects & $\aspChance$ \\
templating index & $(\text{acc}-\text{chance})/(1-\text{chance})$ & $\templIdx$ \\
\midrule
top $4$-gram coverage, \texttt{claim} & \emph{``this suggests that the''} & $31.8\%$ \\
top $4$-gram coverage, \texttt{method} & \emph{``the study employed a''} & $39.0\%$ \\
top $4$-gram coverage, \texttt{result} & \emph{``the study found that''} & $25.7\%$ \\
\midrule
$\rho_{\mathrm{asp}}$, raw inputs & one-way, aspect grouping & $\aspInRaw$ \\
$\rho_{\mathrm{asp}}$, per-aspect mean removed & de-templated & $\aspDetemp$ \\
$\rho_{\mathrm{pap}}$, raw inputs & one-way, paper grouping & $\papInRaw$ \\
oracle after de-templating & MRR / bits & $\orcDetemp$ / $\orcDetempB$ \\
\midrule
ridge probe, context $\to$ oracle space & MRR / bits & $\ridgeMRR$ / $\ridgeB$ \\
loss floor ratio $\mathcal{L}_\infty/\mathbb{E}\|\delta\|^2$ & --- & $\lossratio$ \\
median predictor gradient norm & --- & $\gradnorm$ \\
\midrule
\multicolumn{3}{@{}l}{\emph{Placeholder-text audit, block~C relations
(Sec.~\ref{sec:repair})}} \\
exact-duplicate share, \texttt{challenged\_by} & groups of $\ge50$ identical rows & $\dupChal$ \\
exact-duplicate share, \texttt{implies} & same criterion & $\dupImpl$ \\
largest duplicate group & identical strings & $\dupTop$ \\
genericness, supporting evidence & mean cos to own centroid & $\genSup$ \\
genericness, contradicting evidence & mean cos to own centroid & $\genChal$ \\
genericness gap & challenge $-$ support & $\genGap$ \\
$\cos(\text{claim},\text{support})$ & raw features & $\cosSup$ \\
$\cos(\text{claim},\text{challenge})$ & raw features & $\cosChal$ \\
Cohen's $d$ & support vs.\ challenge & $\polD$ \\
\bottomrule
\end{tabular}
\end{table}

\noindent
Interpretation of the upper blocks is in Sec.~\ref{sec:audits}: templating is severe
and inflates \emph{aspect-type} separability, but paper identity survives its removal
almost intact---the de-templated oracle still recovers $99.3\%$ of the $\bmaxB$-bit
ceiling. We report the $\fwAcc$ figure prominently rather than in passing, because a
reader who discovers it independently would reasonably treat its absence as
concealment, and because it is the empirical reason the
designator-determines-category premise of Prop.~\ref{prop:fixedpoint} is so easily
satisfied on this corpus.

\paragraph{The two audits measure different pathologies, and only one is benign.}
The templating audit says the \emph{surface form} of an aspect is predictable; the
placeholder audit says the \emph{content} of one relation is frequently absent while
the field is nominally populated. The first inflates a nuisance dimension and is
neutralised by removing the per-aspect mean, after which the recoverable identity is
essentially unchanged ($\orcDetempB$ bits). The second cannot be neutralised by any
transformation of the features, because the information was never extracted: $\dupTop$
identical strings carry one string's worth of information regardless of how they are
embedded. This is the distinction we would have missed had we run only the templating
audit, and it is why we recommend both.

\paragraph{Why exact-duplicate hashing was not sufficient on its own.}
Hashing detects only identical rows. The genericness statistic---mean cosine of each
class's members to their own centroid---detects the formulaic-but-unique remainder,
and it is the larger effect here: after the $\dupChal$ of exactly duplicated
contradicting-evidence rows are set aside, the surviving text is still $\genGap$ more
self-similar than supporting evidence. A pipeline that filtered duplicates and
declared the target clean would have trained on filler and reported a $0.85$-AUC
``polarity'' result. Both statistics are one pass over the embedding matrix and cost
under a minute at this corpus size (App.~\ref{app:compute}).

\section{Faithfulness of the reference implementation}
\label{app:faithful}

Table~\ref{tab:faithful} reports our re-implementation on the original method's
graph-classification benchmarks. We do not clear MUTAG and therefore do not claim
implementation faithfulness.

\begin{table}[h]
\centering
\caption{Our reference re-implementation on the original method's
graph-classification benchmarks. We do not clear MUTAG and therefore do not claim
implementation faithfulness.}
\label{tab:faithful}
\small
\begin{tabular}{@{}lccc@{}}
\toprule
Dataset & ours & reported & verdict \\
\midrule
MUTAG & $\mutagOurs$ & $\mutagRef$ & not reproduced \\
PROTEINS & $\protOurs$ & $\protRef$ & within tolerance \\
\bottomrule
\end{tabular}
\end{table}

\noindent
Training is additionally unstable on PROTEINS, with the objective increasing over the
run and predictor gradient norms reaching $O(10^3)$. Three consequences. We scope
every claim to the pipeline we describe and document, not to the original method's
published configuration. All conclusions rest on \emph{internal} controls---the
oracle, the lexical bound, the positive control and the repair sweep all share the
pipeline under test and the identical evaluation---so the diagnosis does not depend on
having matched an external benchmark. And the two structural findings of
Sec.~\ref{sec:repair} are, if anything, \emph{independent} of implementation
fidelity: Prop.~\ref{prop:reducible} is a statement about the graph's edge
cardinalities, which no choice of encoder can change, and the placeholder-text census
is a statement about the corpus, which no choice of objective can change. We regard
disclosing the gap as a precondition for the central claim being taken seriously, and
flag closing it as required work before the method itself is characterised.

\section{Untrained rank and degree}
\label{app:rankinit}

This appendix reports rank and DC ratio per node type \emph{before} any training, and
is one of the two measurements bearing on the trajectory question of
Sec.~\ref{sec:trajectory}.

Before any training we measure effective rank and DC ratio per node type at encoder
depths $0$--$3$. At depth $0$ all node types have effective rank $42$--$59$ and DC
ratio $3.1$--$3.6$. Each additional message-passing layer lowers rank and raises DC
ratio for every node type: at depth $3$, ranks fall to $10$--$29$ and DC ratios rise
to $2.4$--$12.6$. The compression is therefore a property of message passing at
initialisation, not of the objective consistent with Fig.~\ref{fig:dynamics}(b),
where pooled rank is at its final value at epoch $0$. Together these two
measurements are why Sec.~\ref{sec:trajectory} measures the variance allocation per
checkpoint instead of inferring a dynamic re-allocation from the endpoints.

The relation between rank and in-degree is \emph{not} monotone \texttt{evidence} has
mean in-degree $1.00$ and the lowest rank of any type, while \texttt{field} has
in-degree $118.65$ and rank $26$--$34$ so we make no claim of a degree rank law.
The \texttt{evidence} row deserves a second look in light of
Sec.~\ref{sec:repair}: its in-degree of exactly $1.00$ is the leaf property
($\Nevid=\Nsupp+\Nchal$), which is simultaneously why it has the lowest rank, why
mean pooling over it is trivial, and why leaving a single target relation unmasked
produces the $\polEdgeLeak$-AUC leak of Table~\ref{tab:polgate}. One structural fact,
three separate symptoms, which we did not connect until the self-test's leaf-isolation
assertion (App.~\ref{app:selftest}) forced the question.

\section{Input-feature anisotropy}
\label{app:whiten}

Table~\ref{tab:whiten} records the whitened input ranks that rung~4's
rank-restoration claim rests on, together with a logging gap we do not paper over.

\begin{table}[h]
\centering
\caption{Effective rank of the frozen input features (of $\dfeat$) per node type
after PCA/ZCA whitening, applied per node type before the encoder. The
per-node-type \emph{raw} input ranks were not written to the released logs for this
run and we therefore do not tabulate them; raw input anisotropy is instead
characterised by the depth-$0$ DC ratios of App.~\ref{app:rankinit}
($3.1$--$3.6$), and the aggregate raw-to-whitened change is the
$\erInWhite$-of-$\dfeat$ figure quoted for rung~4 of Table~\ref{tab:ladder}. The
whitened values below are the ones the ladder's rank-restoration claim rests on.}
\label{tab:whiten}
\small
\begin{tabular}{lcccc}
\toprule
 & claim & method & result & paper \\
\midrule
whitened & $320.8$ & $309.2$ & $313.1$ & $313.1$ \\
\bottomrule
\end{tabular}
\end{table}

\noindent
The missing raw column is a genuine gap in our logging rather than a selective
omission, and we prefer to say so than to recompute it from a rerun and present it as
if it were the original measurement. It affects only the precision of rung~4's
description, not its conclusion: whitening raises rank and \emph{lowers} the probe
while leaving bits at zero, and both of those are measured in the same run.

\section{Random-walk structural encoding audit}
\label{app:rwse}

This appendix records the RWSE statistics, and then a correction to how we originally
read them.

The context mixer consumes a random-walk structural encoding on the intra-paper
reasoning subgraph. Every paper has at least one intra-paper reasoning edge ($100\%$
coverage); the mean number of reasoning edges per subgraph is $32.30$; the global
patch-RWSE standard deviation is $0.1959$; the fraction of patches with non-zero RWSE
is $100\%$; and the encoding uses $16$ random-walk steps. The encoding is therefore
non-degenerate at input, and the collapse cannot be attributed to an all-zero
structural encoding.

\begin{cautionbox}
\textbf{A correction to how we previously read this audit.} We originally cited
$100\%$ coverage and non-zero variance as evidence that the structural encoding is
\emph{informative}. Prop.~\ref{prop:reducible} shows that inference was wrong. The
intra-paper edge set is a deterministic function of the node census
(Table~\ref{tab:reducible}), so the RWSE is a deterministic function of the census
too: it varies across \emph{aspect types}, and across papers only through the claim
count, and it carries no paper-identifying information beyond that. Non-zero variance
is necessary but not sufficient for informativeness, and the effective-rank
measurement makes the point quantitatively---per-aspect RWSE effective rank is
$\rwseRankLo$--$\rwseRankHi$ of $16$, i.e.\ the $16$-dimensional encoding spans barely
more than one direction. A structural encoding on a census-determined graph is close
to a constant by construction. The check that would have caught this is the
cardinality comparison of Table~\ref{tab:reducible}, not a variance statistic.
\end{cautionbox}

\noindent
The cue ablation of App.~\ref{app:repair} closes the loop empirically: a
\emph{three-entry} learned aspect embedding outperforms the $16$-dimensional RWSE by
$\cueGainA$ bits ($\cueAspect$ versus $\cueRwse$), which is what one expects if the
RWSE is carrying little beyond aspect identity. Removing the cue entirely costs
$\cueGainB$ bits ($\cueNone$), so the predictor does need \emph{some} target
designator---it simply does not need a structural one. Whether the \emph{trained}
query path preserves what little the RWSE does carry is a separate question and is
the first of the open measurements in Sec.~\ref{sec:discussion}.

\newpage
\section{Supporting figures}
\label{app:figures}

Figures~\ref{fig:geometry}--\ref{fig:centering} support Secs.~\ref{sec:ladder}
through~\ref{sec:repair}: the embedding geometry, the depth sweep, the training
dynamics that bear on Sec.~\ref{sec:trajectory}, the difficulty ladder, the DC-ratio
comparison and the post-hoc frame sweep.

\begin{figure}[h]
  \centering
  \includegraphics[width=0.95\textwidth]{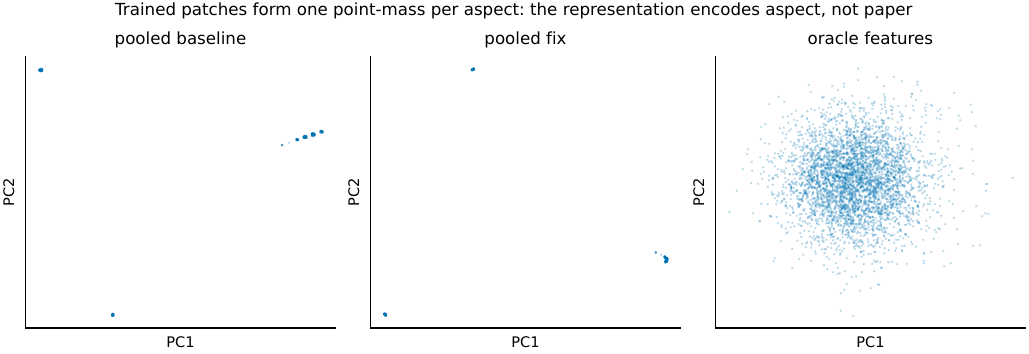}
  \caption{\textbf{Trained patches form one point-mass per aspect.} First two
  principal components. Trained representations (left, centre) occupy two or three
  tight clusters corresponding to aspect type; the frozen features (right) form a
  diffuse cloud in which individual papers are separable. The representation encodes
  \emph{which aspect}, not \emph{which paper}---the visual form of
  Table~\ref{tab:inversion} and of the $\erank\le|\mathcal{A}|$ bound of
  Prop.~\ref{prop:fixedpoint}.}
  \label{fig:geometry}
\end{figure}

\begin{figure}[h]
  \centering
  \includegraphics[width=0.95\textwidth]{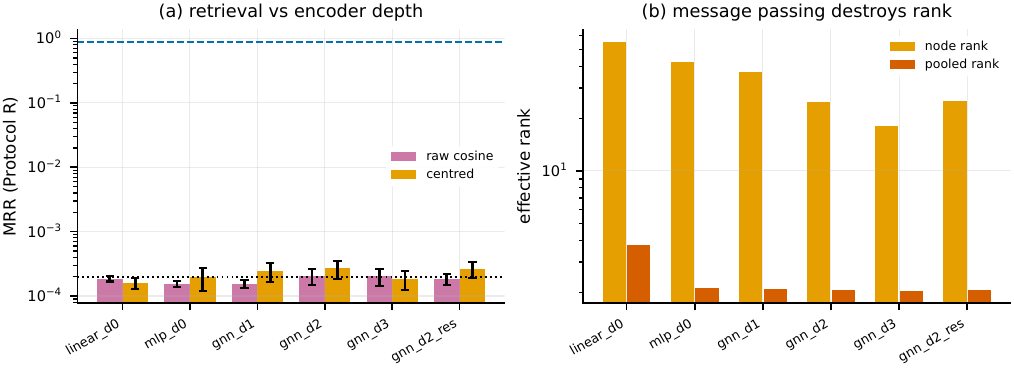}
  \caption{\textbf{Depth is not the binding factor.} \textbf{(a)}~Bits recovered
  across six encoder configurations from a depth-$0$ linear map to a depth-$3$ GNN;
  all at chance, the $\bmaxA$-bit ceiling dashed. \textbf{(b)}~Message passing
  compresses node rank monotonically ($\rnodeDzero\to\rnodeDthree$) and the pooled
  rank is flat regardless.}
  \label{fig:depth}
\end{figure}

\begin{figure}[h]
  \centering
  \includegraphics[width=0.95\textwidth]{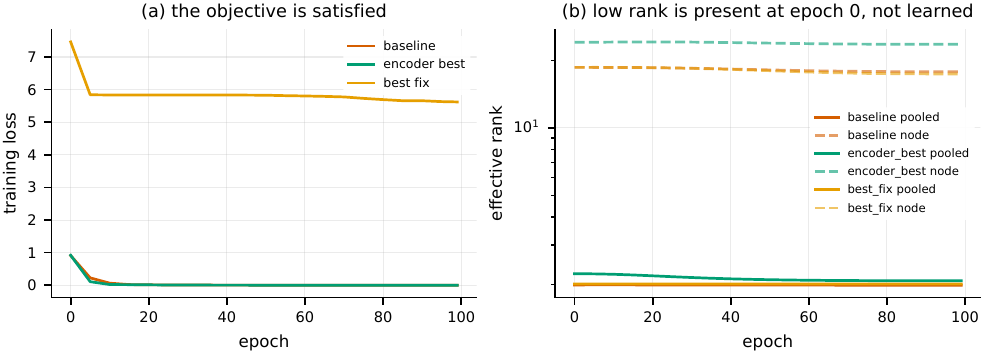}
  \caption{\textbf{Nothing is learned away.} \textbf{(a)}~Training loss converges
  within ${\sim}15$ epochs, on a twin axis because the cosine and InfoNCE losses
  differ by an order of magnitude---and note that this is exactly the comparison
  Sec.~\ref{sec:repair} warns against making on magnitude alone, since a regression
  loss near zero is collapse while an InfoNCE loss near zero is discrimination.
  \textbf{(b)}~Pooled effective rank is already at its final value at epoch $0$ and
  stays flat: the baseline pathology is present at initialisation rather than induced
  by optimisation, which is why Sec.~\ref{sec:trajectory} measures the variance
  allocation per checkpoint rather than asserting a re-allocation. Contrast the
  repaired configuration, whose bits rise from $-0.07$ at step $1$ to $+14.24$ by
  step $500$ (App.~\ref{app:repair}); the difference between the two trajectories is
  the \emph{objective}, as the single-variable regression control establishes.}
  \label{fig:dynamics}
\end{figure}

\begin{figure}[h]
  \centering
  \includegraphics[width=0.68\textwidth]{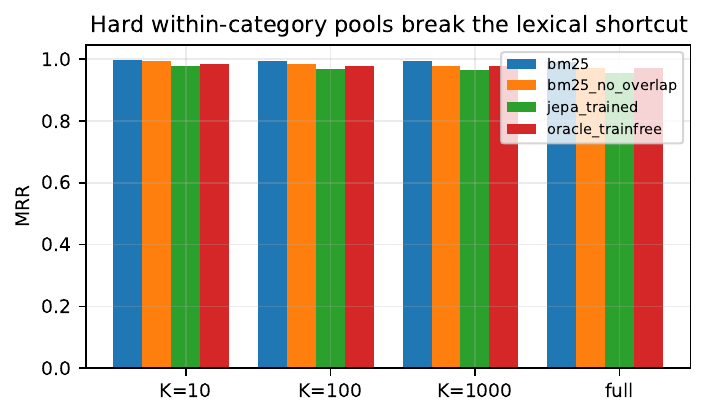}
  \caption{\textbf{Bits recovered against pool difficulty} for the three ceilings and
  the lexical no-overlap control (Table~\ref{tab:ladder2}). All four series are
  ceilings or controls---\BMtf, \NoOv, \Oracle{} and the \Pone{} positive
  control---and all four track the $\frac1N\log_2 N!$ envelope with a near-constant
  deficit across a $5{,}800\times$ change in pool size. The baseline Graph-JEPA is
  the single point at $0$ bits at the right-hand end. This figure replaces every
  ratio-to-chance statement in earlier drafts.}
  \label{fig:hardpool}
\end{figure}

\begin{figure}[h]
  \centering
  \includegraphics[width=0.68\textwidth]{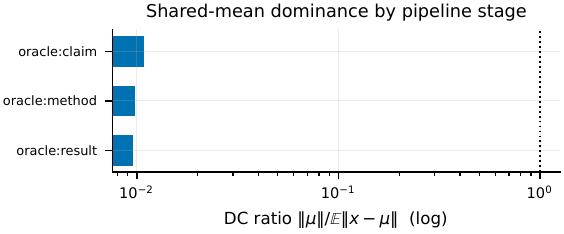}
  \caption{\textbf{Shared-mean dominance by pipeline stage.} DC ratio
  $\|\mu\|/\mathbb{E}\|x-\mu\|$ on a log axis. The frozen features sit near
  $10^{-2}$; the trained query bank sits at $\dcquery$---a factor of ${\sim}10^{4}$.}
  \label{fig:dcratio}
\end{figure}

\begin{figure}[h]
  \centering
  \includegraphics[width=0.5\textwidth]{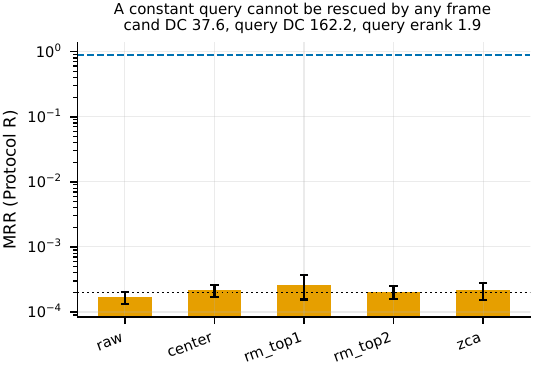}
  \caption{\textbf{A constant query cannot be rescued by any frame applied after
  training.} Five retrieval frames fitted on the candidate bank; the best buys
  $\framegain$ against a $\deficit$ deficit to the ceiling (dashed). This is a
  quantitative confirmation of Prop.~\ref{prop:frame} rather than a failed ablation:
  the way out is to change what the objective asks for, and the single-variable
  demonstration of that is the loss ablation of Sec.~\ref{sec:repair}
  ($\nceBits\!\to\!\regBits$ bits), not any transformation of the retrieval space.}
  \label{fig:centering}
\end{figure}

%

\section{Compute environment and budget}
\label{app:compute}

All experiments ran on a single node with $2\times$ NVIDIA L40S GPUs ($46$\,GB each)
or one NVIDIA H100 NVL ($95$\,GB) for the $20$k-step runs, PyTorch
$2.4.0{+}$cu121, PyTorch Geometric $2.8.0$, scikit-learn $1.7.2$, Python $3.10$.
Graph construction and RWSE precomputation are cached and reused across every block;
blocks~A and~C share that cache byte-for-byte (App.~\ref{app:provenance}).

\begin{table}[h]
\centering
\caption{\textbf{Measured compute budget.} The Protocol~R diagnosis costs
$\gpuhours$ GPU-hours; the additional diagnostic harness of
Secs.~\ref{sec:ceilings}--\ref{sec:mechanism} costs $\diagwall$ minutes on the same
node. The low cost is deliberate: every check recommended in
Sec.~\ref{sec:discussion} is cheap enough to run \emph{before} committing to a
training budget---and the two checks that changed this paper's conclusions, the
reducibility audit and the duplicate census, are the two cheapest lines in the
table.}
\label{tab:compute}
\small
\begin{tabular}{lccc}
\toprule
Stage & Wall-clock & Peak GPU & GPU-hours \\
\midrule
Preprocess: graph build $+$ sentence embed (one-time) & ${\sim}7$\,min$^{\dagger}$ & --- & --- \\
Preprocess: RWSE cache (one-time) & ${\sim}20$\,s & --- & --- \\
\midrule
Instrument self-test ($\selftestN$ assertions) & $18$\,s & $1.2$\,GB & $<0.01$ \\
Untrained rank / DC vs.\ degree & $1$\,s & $4.0$\,GB & $<0.01$ \\
Protocol~R, raw $+$ whitened (10 runs) & $6$\,m\,45\,s & $18.8$\,GB & $0.11$ \\
Per-aspect breakdown (5 seeds) & $3$\,m\,16\,s & $18.1$\,GB & $0.05$ \\
Retrieval-frame sweep (3 seeds) & $1$\,m\,57\,s & $18.0$\,GB & $0.03$ \\
Encoder sweep (6 configs $\times$ 3 seeds) & $4$\,m\,56\,s & $22.6$\,GB & $0.08$ \\
Oracle $+$ \BMtf{} $+$ leakage controls & $6$\,m\,59\,s & $3.2$\,GB & $0.12$ \\
Pooling $\times$ loss grid (7 cells $\times$ 5 seeds) & $21$\,m\,32\,s & $18.9$\,GB & $0.36$ \\
\midrule
Difficulty ladder $+$ hard pools (BM25 index reused) & $46$\,s & $6.1$\,GB & $0.01$ \\
Extraction / templating audit & $75$\,s & $2.0$\,GB & $0.02$ \\
Phase analysis ($60$ simulations $+$ bisection) & $3$\,s & $2.4$\,GB & $<0.01$ \\
Bits-accounting verification (CPU only) & $9$\,s & --- & $0$ \\
\textbf{Checkpoint allocation pass} (App.~\ref{app:trajectory}) & $\mathbf{<1}$\,\textbf{min} & $2.0$\,GB & $\mathbf{<0.01}$ \\
\midrule
\textbf{Reducibility audit} (edge-cardinality pass) & $\mathbf{<1}$\,\textbf{s} & --- & $\mathbf{0}$ \\
\textbf{Duplicate $+$ genericness census} & $\mathbf{41}$\,\textbf{s} & $2.1$\,GB & $\mathbf{0.01}$ \\
Polarity probe, both groupings (3 seeds $\times$ 2) & $2$\,m\,29\,s & $9.4$\,GB & $0.04$ \\
Matched-budget $2\times2$, $3$k steps (1 seed) & $32$\,m\,08\,s & $31.4$\,GB & $0.54$ \\
Loss ablation, regression control, $3$k (1 seed) & $7$\,m\,51\,s & $30.8$\,GB & $0.13$ \\
Cue ablation, $3\times$ $6$k steps (1 seed) & $47$\,m\,12\,s & $32.0$\,GB & $0.79$ \\
Transduction $+$ raw-skip, $2\times$ $6$k (1 seed) & $32$\,m\,07\,s & $32.6$\,GB & $0.54$ \\
Main run, $20$k steps (3 seeds) & $2$\,h\,37\,m & $34.1$\,GB & $2.62$ \\
\midrule
\textbf{Total, diagnosis only} & $\mathbf{\walltime}$\,\textbf{min} & $\mathbf{\peakmem}$\,\textbf{GB} & $\mathbf{\gpuhours}$ \\
\textbf{Total, including the repair sweep} & $\mathbf{6.1}$\,\textbf{h} & $\mathbf{34.1}$\,\textbf{GB} & $\mathbf{5.44}$ \\
\bottomrule
\end{tabular}
\vspace{2pt}
\raggedright{\footnotesize $^{\dagger}$One-time; dominated by sentence encoding of
${\sim}1.05$M nodes.}
\end{table}

\paragraph{The cost asymmetry is the practical message.}
The two measurements that changed this paper's conclusions cost under a minute
combined: the reducibility audit is a pass over nine integers, and the duplicate and
genericness census is a single pass over an embedding matrix. The sweep they
retrospectively reinterpreted cost $4.6$ GPU-hours, four orders of magnitude more. We
had the budget to train for $20$k steps on three seeds long before we had the
discipline to compare an edge count against a node count. Any reader who takes one
thing from this paper should take that ordering: audit the target before you train on
it, because the audit is free and the training is not. The checkpoint allocation pass
of App.~\ref{app:trajectory} belongs in the same category, and so does the
non-census-determined control graph named in Sec.~\ref{sec:discussion}: about one
GPU-hour, against a claim that the rest of the paper cannot make without it.

\paragraph{And one measurement that cost nothing but attention.}
The correction recorded in App.~\ref{app:repair}---that the target-frame main effect
is $\frameEffect$ bits at matched budget rather than the $+1.199$ an earlier draft
reported, and that the learning-rate schedule is worth $\scheduleEffect$ bits---came
from reading a log stage we had already paid for and previously skipped. No new
compute was required. We mention it because the marginal value of re-reading a
completed run is easy to underestimate relative to launching another one.

\end{document}